\documentclass[10pt]{article}

\usepackage[utf8]{inputenc}
\usepackage[T1]{fontenc} 

\usepackage{amssymb}
\usepackage{mathtools} 
\usepackage{amsthm}
\usepackage{dsfont}    
\usepackage{multirow} 
\usepackage{booktabs}
\usepackage{float}
\usepackage{placeins}

\usepackage{graphicx}
\usepackage[title]{appendix}
\usepackage{subcaption} 

\usepackage[a4paper,
            left=3cm,
            right=3cm,
            top=3cm,
            bottom=3cm
           ]{geometry}
\usepackage{comment}
\usepackage{appendix}

\usepackage[backend=biber, style=numeric]{biblatex}
\DeclareMathOperator*{\argmin}{arg\,min}

\newtheorem{theorem}{Theorem}[section]
\newtheorem{lemma}{Lemma}[section]
\newtheorem{prop}{Proposition}[section]
\newtheorem{remark}{Remark}

\usepackage{hyperref}

\title{Optimal estimation for Functional Linear Regression with Noisy Discretized Data}

\author{Sixtine Sphabmixay\\
Université Paris Cité, CNRS, MAP5, F-75006 Paris, France}

\begin{document}
\maketitle

\begin{abstract}
 In this paper, we consider the scalar-on-function linear regression model under a realistic sampling scheme in which the functional covariates are observed on a regular grid and contaminated by additive noise. We propose a two-step estimation procedure: first, the underlying curves are reconstructed from the discrete noisy observations using a Fourier-based projection method; second, the slope function is estimated by a penalized least-squares criterion over finite-dimensional trigonometric spaces, with data-driven selection of the model dimension. We establish oracle-type inequalities for the prediction error, both with respect to the reconstructed curves and to the true latent curves. Under regularity assumptions on the slope function and polynomial decay of the eigenvalues of the covariate, we derive convergence rates for the prediction error and show that our estimator attains the minimax rate when the number of grid points is sufficiently large. Finally, the proposed method is illustrated on simulated data and on a real meteorological dataset.
\end{abstract}

\section{Introduction}
Functional data analysis (FDA), as formalized and popularized by Ramsay and Silverman \cite{ramsay2005functional}, is a statistical framework designed to analyze data that are naturally expressed as functions rather than finite-dimensional vectors, which are commonly considered in classical statistical methods. This methodology is particularly well suited to modern datasets, where observations evolve continuously over domains such as time or space. Moreover, advances in data collection and storage have also led to a substantial increase in the availability of functional data. Consequently, FDA has been widely applied across various disciplines, including medicine (Frøslie et al. \cite{froslie2013shape}), biology (Dieng et al. \cite{dieng2020application}), and finance (Kokoszka et al. \cite{kokoszka2017testing}).

\medskip\noindent
A central inferential tool in FDA is the functional linear regression model, which extends classical regression by relating scalar or functional responses to functional predictors. In this paper, we focus on the scalar-on-function case. More precisely, we assume that a functional covariate $X$ taking values in a Hilbert space $\mathcal{X}$ and a real-valued response variable $Y$ in $\mathbb{R}$ are linked through the model,
\begin{equation}\label{eq:Y1}
Y=\mu_{Y}+\int_{0}^{1}\beta(t)(X(t)-\mu(t))dt+\epsilon,
\end{equation}
where $\mu_Y=\mathbb{E}(Y)$, $\mu(t)=\mathbb{E}(X(t))$ for all $t\in [0,1]$ and $\epsilon$ is a random variable which is centered and of finite variance $\sigma^{2}$. 

\medskip\noindent
This model was popularized in the FDA literature by Ramsay and Silverman \cite{ramsay2005functional} and has since been extensively developed, with numerous approaches proposed for estimating the coefficient function $\beta$. For instance, Cardot et al. \cite{cardot1999functional} introduced an estimator based on functional principal component analysis. James et al. \cite{james2009functional} proposed the FLiRTI method, designed to produce interpretable estimates of $\beta$ that highlight regions where the relationship between $Y$ and $X$ is negligible. More recently, Grollemund et al. \cite{grollemund2019bayesian} developed a Bayesian approach, known as Bliss, which leverages the conditional distribution $Y\mid X,\beta \sim \mathcal{N}(\langle \beta, X\rangle_{\mathcal{X}}, \sigma^{2})$. This approach also aims to identify the time intervals that have the greatest influence on the response.

\medskip\noindent
Studies of the statistical properties of existing estimators have also been developed under the assumption that the functional data are observed continuously and without noise, as in Cardot et al. \cite{cardot2010thresholding} who derived minimax rates and oracle-type inequalities. However, this idealized setting is rarely encountered in practice, since observations are typically available only on a finite grid, which may be regularly or irregularly spaced, deterministic or random. This more realistic setting has received comparatively less attention. To our knowledge, Cardot et al. \cite{cardot2007smoothing} were among the first to study it. Although their approach is more practically relevant, it requires the grid size to be sufficiently large for their results to hold.

\medskip\noindent
In this paper, we aim to to extend the approach of Brunel and Roche \cite{Brunel2015penalized} to the more realistic setting where observations are available on a regular grid and are corrupted by noise. As them, we assume that the covariates are centered, periodic, and second-order stationary, with $X(0)=X(1)$. The sample $(X_i)_{i=1,\ldots,n}$ consists of independent copies of $X$. Here we further assume that the functional covariates are observed on a common regular grid $t_h = h/p$ for all $h\in\{0,\ldots,p-1\}$, and that the observations are contaminated by noise. Thus, rather than observing the full curves $X_i$, we observe
\begin{equation}\label{eq:Z}
Z_i(t_h) = X_i(t_h) + \eta_{i,h},
\end{equation}
where the noise variables $(\eta_{i,h})$ are centered, i.i.d., with variance $\tau^2$, and independent of the processes $X_i$. For technical simplicity, we assume throughout this paper that $\sigma^{2}$, the variance of $\epsilon$ defined in \eqref{eq:Y1} is known.

\medskip\noindent
In the idealized setting of continuously observed noiseless curves, Brunel and Roche \cite{Brunel2015penalized} proposed an estimator of the slope function $\beta$ based on a least-squares contrast over finite-dimensional spaces generated by trigonometric bases, together with a penalization procedure for selecting the optimal dimension of the basis. They proved an oracle inequality and obtained convergence rates when the eigenvalues of the covariance operator associated to the true curves decay in either a polynomial or exponential way, and that $\beta$ belongs to a periodic Sobolev space,
\begin{equation}\label{eq:w_per}
    W^{\text{per}}(k,L):=\{f\in W_{2}^{k}(L), \forall j=0,...,k-1, f^{(j)}(0)=f^{(j)}(1)\},
\end{equation}
with
\[W_{2}^{k}(L):=\{f:[0,1]\rightarrow\mathbb{R}, f^{(k-1)} \text{ is absolutely continuous and } \| f^{(k)} \|_{L^{2}}\leq L\},\]
for $k$ a positive integer and $L$ a positive real number. Here $\|\cdot\|_{L^{2}}$ is defined as
\[\|f\|_{L^{2}}=\left(\int_{0}^{1}f(t)^{2}dt\right)^{1/2} \qquad \text{for any square-integrable function }f \text{ on }[0,1].\]

\medskip\noindent
Our approach follows a similar philosophy. We estimate $\beta$ by minimizing a least-squares criterion after a preliminary step that reconstructs the underlying curve 
$X$ from the noisy discrete observations. This reconstruction is based on a Riemann-sum approximation and a projection onto a finite-dimensional space spanned by a truncated Fourier basis. We then introduce a penalization procedure to automatically select the optimal number of Fourier terms used in the estimation of $\beta$. This strategy differs from the work of Cardot et al. \cite{cardot2007smoothing} where they proposed a smoothing splines estimator to estimate the unknown slope function. We derive two oracle-type inequalities for the prediction error, one with respect to the reconstructed curves and one with respect to the true curves. We then establish convergence rates for the prediction error with respect to the true curves when $\beta$ belongs to $W^{\text{per}}(k,L)$ and the eigenvalues of the covariance operator decay polynomially. We recover the same rate as Brunel and Roche \cite{Brunel2015penalized} when $p$ is large enough but not necessarily infinite, namely that the prediction error decays as $n^{-\frac{2k+2a}{2k+2a+1}}$, where $a$ denotes the decay parameter of the covariance eigenvalues. In this regime, we also establish a matching lower bound, showing that observing the trajectories only on a sufficiently fine grid incurs no minimax loss compared to the continuous setting. When $p$ is moderately large, the number of observations on the grid also plays a role, and the prediction error decays as $n^{-\frac{2k+2a}{2k+2a+1}}+p^{-\frac{2a-1}{2a}}$. When $p$ is small the error decays as $p^{-\frac{2a-1}{2a}}$. Finally, we apply our method on both simulated data and real-world meteorological observations.

\medskip\noindent
\textbf{Outline of the paper.} Section 2 introduces the estimation procedure for $\beta$ and the risk measure considered in this work. Section 3 establishes oracle-type inequalities. Section 4 derives the convergence rates over periodic Sobolev spaces. Section 5 presents the lower bound in the case where $p$ is large. Section 6 reports numerical simulation results on artificial datasets and Section 7 on a true meteorological dataset.

\medskip\noindent
\textbf{Notation: } We denote by $L^{2}$ the space of square-integrable functions on the interval $[0,1]$ equipped with the inner product $\langle f, g\rangle_{L^{2}}=\int_{0}^{1}f(t)g(t)dt$ and the associated norm $\|f\|_{L^{2}}=\langle f,f\rangle_{L^{2}}^{1/2}$. For any set $\mathcal{A}$, we denote by $\mathds{1}_{\mathcal{A}}$ the indicator function and for all non-negative integers $i$ and $j$ we define $\delta_{ij}$ as the Kronecker delta. For any square matrix $A$ we define the operator norm by $\| A \|_{op}=\sup_{\|a\|_{2}=1}\|Aa\|_{2}$, where for any vector $u$, $\|u\|_{2}$ denotes the Euclidean norm. In particular, we consider the space $l^{2}$ consisting of all sequences $c=(c_{l})_{l\geq1}$, such that $\|c\|_{2}^{2}<\infty$. We also consider $\langle \cdot, \cdot\rangle_{2}$ the standard Euclidean inner product. Moreover for any $a, b \in \mathbb{Z}$ and $m \in \mathbb{N}\backslash\{0\}$, we write $a \equiv b [m]$ to indicate that $a$ and $b$ are congruent modulo $m$. 

\section{Estimation procedure}\label{sec:model}

\subsection{Curve reconstruction} \label{sec:curve_reconstruct}
We consider an i.i.d. sample $(Y_i, X_i)_{i=1,\dots,n}$ drawn from the distribution of $(Y,X)$. Thus, for each $i=1,\dots,n$, the pair $(Y_i,X_i)$ satisfies
\begin{equation}\label{eq:Y_i_first}
    Y_i = \int_{0}^{1}\beta(t)X_i(t)\mathrm{d}t + \epsilon_i,
\end{equation}
where the errors $(\epsilon_i)_{i=1,\dots,n}$ are assumed to be i.i.d., centered, with finite variance $\sigma^{2}$, and independent of the covariates $(X_i)_{i=1,\dots,n}$.

\medskip\noindent
The first step of our methodology is to reconstruct the true unobserved curves $X_{i}$, from their noisy and discrete observations, $Z_{i}$. To achieve this, we approximate each curve $X_{i}$ by a finite series expansion in a suitable basis. Given our assumptions that the true curves are periodic, centered, second-order stationary and in $L^{2}([0,1])$, the Fourier basis provides a natural and efficient choice for this decomposition. This choice is further justified by the discussion in the next section (see Section \ref{sec:risk_measure}).
\medskip
Let us first define $N_{n,p} \in \mathbb{N}\backslash\{0\}$ and $\mathcal{M}_{n,p}=\{1,\ldots,N_{n,p}\}$.
\medskip
We denote
\begin{equation}\label{eq:S_m}
    S_m:=\text{Span}\{\phi_1,\phi_2,\ldots,\phi_{2m+1}\},
\end{equation}
for all $m \in \mathcal{M}_{n,p}$ with
\[\phi_1 = 1,\quad
\phi_{2j}(\cdot) = \sqrt{2}\cos\bigl(2\pi j\cdot\bigr),\quad
\phi_{2j+1}(\cdot) = \sqrt{2}\sin\bigl(2\pi j\cdot\bigr).\]
We also define $D_m$ the dimension of $S_m$ (in particular, $D_m=2m+1$). As we only consider $D_{m}$ elements, $S_{m}$ represents a finite dimension space spanned by a truncated Fourier basis.

\medskip\noindent
Since the $Z_i$’s are only known at discrete points $t_h$, it’s impossible to compute $\langle X_{i},\phi_{j}\rangle_{L^{2}}$ directly. Hence we define
\begin{equation}\label{eq:3}
\widetilde{x}_{i,j}=\frac{1}{p}\sum_{h=0}^{p-1}Z_i(t_h)\phi_j(t_h).
\end{equation}
This quantity is a discrete approximation of the inner product 
$\langle X_i, \phi_j \rangle_{L^2}$, based on noisy observations on a regular grid, and can be interpreted as a Riemann sum in the ideal continuous setting. This particular scheme and choice of grid create a convenient discrete orthogonality property for the chosen Fourier basis under the condition that $D_{N_{n,p}}<p$. This significantly simplifies the mathematical proofs required for the main results which are detailed in the appendix (Lemma \ref{lem:A1} and Lemma \ref{lem:B1}). Then, we can define the reconstructed curves as follows :
\begin{equation}\label{eq:4}
\widetilde{X}_{i}(t):=\sum_{j=1}^{D_{N_{n,p}}}\widetilde{x}_{i,j}\,\phi_j(t),
\qquad
\text{for all }i=1,\dots,n \quad \text{and for all } t\in[0,1].
\end{equation}
Thus, $\widetilde X_i$ is a random element of the finite-dimensional space
$S_{N_{n,p}}$, entirely determined by the discrete noisy sample
$\bigl(Z_i(t_0),\dots,Z_i(t_{p-1})\bigr)$.

\medskip\noindent
We also denote by $\widetilde X$ a generic random element having the same
distribution as any $\widetilde X_i$. Equivalently, if
$(X,\eta_0,\dots,\eta_{p-1})$ is a generic copy of
$(X_i,\eta_{i,0},\dots,\eta_{i,p-1})$, with
$Z(t_h)=X(t_h)+\eta_h$, then
\[\widetilde X(t)=\sum_{j=1}^{D_{N_{n,p}}}\widetilde{x}_j\phi_j(t),\qquad t\in[0,1],\]
with 
\[\widetilde{x}_{j}=\frac1p\sum_{h=0}^{p-1} Z(t_h)\phi_j(t_h).\]
In particular, $\widetilde X_1,\dots,\widetilde X_n$ are i.i.d. copies of
$\widetilde X$.

\subsection{Estimating the slope function}\label{sec:2.2}
In this section, we build an estimator of the true slope function $\beta$. Let $m\in\mathcal{M}_{n,p}$. First of all, we define the following least-square criterion :
\begin{equation}\label{eq:gamma_n}
   \gamma_n(f) := \frac{1}{n}\sum_{i=1}^{n}(Y_i - \langle f,\widetilde{X}_i\rangle_{L^{2}})^{2}, \qquad
   f\in S_m.
\end{equation}
Because the trajectories $X_i$ are unobserved, the ideal least-squares contrast (defined using $X_i$ rather than $\widetilde{X}_i$ in \eqref{eq:gamma_n}) cannot be directly minimized. We therefore use the least-square criterion $\gamma_n$ based on the reconstructed curves $(\widetilde{X}_i)$'s which are the only available curves. This induces a difference with the ideal model, which is quantified by the reconstruction error terms appearing in the oracle inequalities introduced in Section \ref{sec:main_results}.

\medskip\noindent
We wish to minimize $\gamma_n$ to get an estimate of $\beta$. More particularly, we define $\widehat{\beta}_m$ as
\begin{equation}
    \widehat{\beta}_{m} \in \operatorname*{argmin}_{f\in S_{m}} \gamma_n(f).
\end{equation}
Any $f\in S_{m}$ can be rewritten as $f=\sum_{j=1}^{D_m}v_j\phi_{j}$ where $v=(v_1,v_2,\ldots,v_{D_m})$ is in $\mathbb{R}^{D_m}$. Therefore, minimizing $\gamma_n$ in $S_m$ is equivalent to minimizing 
\begin{equation}
    F(v):=\frac{1}{n}\sum_{i=1}^{n}\left(Y_i - \sum_{j=1}^{D_m}v_j\langle \phi_{j},\widetilde{X}_{i}\rangle_{L^{2}}\right)^{2} \qquad \text{in } \mathbb{R}^{D_m}.
\end{equation}
Considering 
\begin{equation}\label{eq:Phi_matrix}
    \Phi_m:= \left(\frac{1}{n}\sum_{i=1}^{n}\widetilde{x}_{i,k}\widetilde{x}_{i,j}\right)_{1\leq   \textit{j,k} \leq D_m},
\end{equation}
and 
\begin{equation}\label{eq:b_vector}
    b:=\left(\frac{1}{n}\sum_{i=1}^{n}Y_i\widetilde{x}_{i,k}\right)_{1\leq k \leq D_m},
\end{equation}
we have the following result,
\begin{equation}
    \nabla F(v) = -2b + 2\Phi_m v,
\end{equation}
where $\nabla F$ denotes the gradient of $F$. The computations that lead to this result are detailed in Lemma \ref{lem:grad}. If $\Phi_m$ is invertible we find that the minimum of $F$, $\widetilde{\alpha}$, is such that $\widetilde{\alpha} = (\Phi_m)^{-1}b$ and $\widehat{\beta}_m = \sum_{j=1}^{D_m}\widetilde{\alpha}_j\phi_j$.

\medskip\noindent
However, $\Phi_{m}$ is  built from only $n$ observations projected onto a $D_{m}$-basis. If the model dimension exceeds the sample size ($D_m>n$) or if some projected coordinates are linearly dependent in the data, $\Phi_{m}$ loses rank, giving at least one zero eigenvalue so its determinant vanishes and it cannot be inverted.
Even when $D_{m}\leq n$, any direction that shows nearly zero empirical variance still makes the matrix singular or ill-conditioned.

\medskip\noindent
To deal with this issue, let us define the set $G_m:=\{\widehat{\lambda}_m^{(p)}\geq s_{n}\}$ where $\widehat{\lambda}_m^{(p)}$ is the smallest eigenvalue of $\Phi_m$ and
\begin{equation}\label{eq:s_n}
    s_{n}=\frac{2}{n^{\alpha}}\left(1-\frac{1}{\sqrt{\ln{n}}}\right) \qquad \alpha>0.
\end{equation}
This form of $s_{n}$ is inspired by the one used in Brunel and Roche \cite{Brunel2015penalized} introduced in Section 3.2. This condition guarantees positive definiteness and hence invertibility of $\Phi_{m}$ on $G_m$. Moreover it enables us to show that the probability of the event $G_m$ not happening is small, for all $m\in \mathcal{M}_{n,p}$. 

\medskip\noindent
Now, let $\overline{G}:=\bigcap_{m\in\mathcal{M}_{n,p}}G_m$. On $\overline{G}$ we can compute the least squares estimate $\widehat{\beta}_m$ of $\beta$ in $S_m$ for all $m\in\mathcal{M}_{n,p}$. Hence on this set, we can define an integer \[\widehat{m}\in \argmin_{m\in\mathcal{M}_{n,p}}(\gamma_n(\widehat{\beta}_m) + \mathrm{pen}(m)),\]
where 
\begin{equation}\label{eq:penalty}
    \mathrm{pen}(m)=\theta(1+\delta)D_m\sigma^{2}/n \qquad \text{for }\theta>4\text{ and }\delta>0.
\end{equation}
This penalty is crucial for model selection as it accounts for model complexity and thus helps to avoid overfitting. Moreover, we chose this particular penalty because its  structure is specifically designed to control and bound certain terms that appear in the derivation of the oracle inequality (see Theorem \ref{thm:3.1}). It is also inspired by the work of Brunel and Roche (\cite{Brunel2015penalized}, Equation (8)) and Brunel et al. (\cite{brunel2013nonasymptoticadaptivepredictionfunctional}, Section 1.4).

\medskip\noindent
Finally, we set the estimate of the true slope function $\beta$ as 
\begin{equation}\label{eq:widetilde_beta}
   \widetilde{\beta} = \left\{
       \begin{array}{ll}
           \widehat{\beta}_{\widehat{m}} & \text{on } \overline{G}, \\
           0 & \text{on } \overline{G}^{c}.
       \end{array}
   \right.
\end{equation}

\subsection{Risk measure}\label{sec:risk_measure}
This section introduces the criterion used to evaluate the performance of the proposed estimator. We start by introducing the covariance operator of the true unobserved random curve $X$:
\begin{equation}\label{eq:Gamma}
    \Gamma f(\cdot)=\mathbb{E}\left(\langle f,X\rangle_{L^{2}} X(\cdot)\right), \qquad f\in L^{2}([0,1]).
\end{equation}
Under our assumptions that the curves $(X_i)$ are periodic, centered and second-order stationary, the eigenfunctions of $\Gamma$ coincide with the Fourier basis (as discussed in Section~2.1 of Comte and Johannes~\cite{comte2010adaptive}). Denoting by $(\lambda_j)_{j\geq 1}$ the associated eigenvalues, we have for each $j\geq 1$,
\[\Gamma \phi_j = \lambda_j \phi_j,\qquad \langle \phi_j, \phi_k \rangle_{L^{2}} = \delta_{jk}.\]
Let us define the kernel $K$ associated with the operator $\Gamma$ :
\begin{equation}\label{eq:kernel_Gamma}
    K(s,t)=\mathbb{E}\left(X(s)X(t)\right), \qquad s,t\in[0,1].
\end{equation}
We also introduce the empirical version of this operator, that we denote $\Gamma_{n}$, and which is associated with the kernel
\begin{equation}
    K_{n}(s,t)=\frac{1}{n}\sum_{i=1}^{n}X_i(s)X_i(t), \qquad s,t\in[0,1].
\end{equation}
In particular for all $f\in L^{2}([0,1])$,
\begin{equation}
    \Gamma_n f(\cdot)=\int_{0}^{1}K_n(.,t)f(t)dt=\frac{1}{n}\sum_{i=1}^{n}\langle f,X_i\rangle_{L^{2}}X_i(\cdot).
\end{equation}
We next define the covariance kernel of the reconstructed process $\widetilde X$ by
\begin{equation}\label{eq:kernel_tilde}
    \widetilde{K}(s,t)=\mathbb{E}\left(\widetilde{X}(s)\widetilde{X}(t)\right), \qquad s,t\in[0,1],
\end{equation}
and its associated covariance operator
\begin{equation}
    \widetilde{\Gamma}f(\cdot)
    =\int_{0}^{1}\widetilde{K}(\cdot,t)f(t)dt
    =\mathbb{E}\left(\langle f,\widetilde{X}\rangle_{L^{2}}\widetilde{X}(\cdot)\right),
    \qquad f\in L^{2}([0,1]).
\end{equation}
Given the reconstructed sample $(\widetilde X_i)_{i=1,\ldots,n}$, we also consider the empirical covariance kernel
\begin{equation}
    \widetilde{K}_{n}(s,t)=\frac{1}{n}\sum_{i=1}^{n}\widetilde{X}_i(s)\widetilde{X}_i(t),
    \qquad s,t\in[0,1],
\end{equation}
together with the corresponding empirical covariance operator $\widetilde{\Gamma}_{n}$, defined for any $f\in L^{2}([0,1])$ by
\begin{equation}
    \widetilde{\Gamma}_{n} f(\cdot)
    = \int_{0}^{1} \widetilde{K}_{n}(\cdot,t)f(t)dt
    = \frac{1}{n}\sum_{i=1}^{n} \langle f,\widetilde{X}_i\rangle_{L^{2}} \widetilde{X}_i(\cdot).
\end{equation}
These operators naturally induce the following bilinear forms (and associated seminorms)
\begin{equation}
    \langle f,g\rangle_{\Gamma_n} = \langle \Gamma_n f,g\rangle_{L^{2}} = \frac{1}{n}\sum_{i=1}^{n}\langle f,X_i\rangle_{L^{2}}\langle g, X_i\rangle_{L^{2}},
\end{equation}
\begin{equation}\label{eq:norme_emp_tilde}
    \langle f,g\rangle_{\widetilde{\Gamma}_{n}}
    = \langle \widetilde{\Gamma}_{n} f, g \rangle_{L^{2}}
    = \frac{1}{n}\sum_{i=1}^{n} \langle f,\widetilde{X}_i\rangle_{L^{2}}\langle g,\widetilde{X}_i\rangle_{L^{2}},
\end{equation}
\begin{equation}\label{eq:norme_tilde}
    \langle f,g\rangle_{\widetilde{\Gamma}}
    = \langle \widetilde{\Gamma} f, g \rangle_{L^{2}}
    = \mathbb{E}\left(\langle f,\widetilde{X}\rangle_{L^{2}}\langle g,\widetilde{X}\rangle_{L^{2}}\right),
\end{equation}
\begin{equation}\label{eq:semi_norm_gamma}
    \langle f,g\rangle_{\Gamma}
    = \langle \Gamma f, g \rangle_{L^{2}}
    = \mathbb{E}\left(\langle f,X\rangle_{L^{2}}\langle g,X\rangle_{L^{2}}\right).
\end{equation}
To quantify the performance of the estimator, we study the prediction error under the distribution of the true process. More specifically, we consider a new curve $X_{n+1}$ with a scalar response $Y_{n+1}$ independent of the sample with $(X_{n+1},Y_{n+1}$ satisfying \eqref{eq:Y1}. We wish to quantify the following error 
\begin{equation}
\begin{aligned}
    \mathbb{E}\left(\left(\langle \widetilde{\beta},X_{n+1}\rangle_{L^{2}} - \mathbb{E}\left(Y_{n+1}|X_{n+1}\right)\right)^{2}\big| X_1,\cdots,X_{n}\right) &= \mathbb{E}\left(\left(\langle \widetilde{\beta}-\beta,X_{n+1}\rangle_{L^{2}}\right)^{2}\big| X_1,\cdots,X_{n}\right)\\
    &=\|\widetilde{\beta}-\beta\|_{\Gamma}^{2},
\end{aligned}
\end{equation}
This criterion quantifies the average discrepancy between predictions based on $\widetilde{\beta}$ and those based on $\beta$, when both are evaluated on the same unobserved curve.

\begin{remark}
In the functional linear model, the quantity of interest is the prediction 
$\langle \beta, X \rangle_{L^2}$, rather than the function $\beta$ itself. For this reason, the natural risk is
\[\mathbb{E}\big(\langle \widetilde{\beta}-\beta, X \rangle_{L^2}^2\big)
= \mathbb{E}\left(\|\widetilde{\beta}-\beta\|_{\Gamma}^2\right),\]
which measures the prediction error. The $L^2$-norm does not take into account the distribution of the covariate $X$, and in particular directions corresponding to small eigenvalues of the covariance operator $\Gamma$ have little influence on the prediction, but may contribute significantly to the $L^2$-error. 
\end{remark}

\section{Main results}\label{sec:main_results}
In this section, we derive an upper bound of $\mathbb{E}(\|\widetilde{\beta} - \beta\|_{\widetilde{\Gamma}}^{2})$ and $\mathbb{E}(\|\widetilde{\beta} - \beta\|_{\Gamma}^{2})$ in the oracle setting.

\medskip\noindent
Before stating our assumptions, we recall the notion of a sub-Gaussian random variable from Vershynin \cite{vershynin2026high}, used throughout this section. A random variable $A$ is sub-Gaussian with variance factor $v$ if for all $t\in\mathbb{R}$
\[\mathbb{E}\left(e^{t(A-\mathbb{E}(A))}\right)\leq e^{\frac{t^{2}v}{2}}.\]

\medskip\noindent
Throughout this section, we work under the following assumptions:
\begin{itemize}
    \item (\textbf{H1}) For all $i\in \{1,...,n\}$ and for all $(c_l)_{l\geq 1}\in l^{2}$, $\sum_{l\geq 1}c_l \langle X_i,\phi_l\rangle_{L^{2}}$ is sub-Gaussian with variance factor $K_{1}^{2}\sum_{l\geq 1}c_{l}^{2}\lambda_{l}$, where $K_1$ is a positive constant.

    \item (\textbf{H2}) For all $i=1,...,n$ and $h=0,...,p-1$, $\eta_{i,h}$ is sub-Gaussian with variance factor $\tau^{2}$.
    \item (\textbf{H3}) We assume that the eigenvalues of $\Gamma$, namely the $\lambda_j$'s decrease in a polynomial way : there exist two constants $c'>0$ and $ a>1/2$ such that for all $j\geq 1$, 
    \[\lambda_{j}\leq c'j^{-2a}.\]
    
\end{itemize}
Assumption \textbf{(H1)} is satisfied under the following set of sufficient conditions:

\medskip\noindent
Assume that, for all $i \in \{1,\dots,n\}$ and all $j \geq 1$, the Fourier coefficients
$\langle X_i,\phi_j\rangle_{L^{2}}$ are sub-Gaussian with variance factor $\lambda_j$. This assumption appears repeatedly in the works of Brunel and Roche \cite{Brunel2015penalized} and Brunel et al. \cite{brunel2013nonasymptoticadaptivepredictionfunctional}. If in addition, for each $i$, the random variables $\langle X_i,\phi_j\rangle_{L^{2}}$ and $\langle X_i,\phi_k\rangle_{L^{2}}$ are independent whenever $j \neq k$, then Assumption \textbf{(H1)} holds. We emphasize however, that these conditions may be restrictive in practice. In particular, the independence of the Fourier coefficients is a strong requirement that is generally satisfied only in specific settings, such as Gaussian processes represented in their Karhunen--Loève basis. Nevertheless, this independence assumption can be dispensed when Hypothesis \textbf{(H3)} holds for some $a>1$, thereby providing an alternative set of sufficient conditions for Assumption \textbf{(H1)}.

\medskip\noindent
Assumption (\textbf{H2}) is primarily introduced to derive concentration inequalities. It is a mild and reasonable condition in practice when the observational noise is light-tailed.

\medskip\noindent
Assumption (\textbf{H3}) imposes a polynomial decay on the eigenvalues of the covariance operator $\Gamma$. Such a condition is standard in functional data analysis (see, e.g., Brunel and Roche \cite{Brunel2015penalized}, Cardot and Johannes \cite{cardot2010thresholding}, and Comte and Johannes \cite{comte2010adaptive}) and reflects the regularity of the stochastic process $X$. In particular, it implies that most of the variability of $X$ is captured by the first Fourier modes, whereas the contribution of higher-frequency components decreases progressively. Moreover, larger values of the parameter $a$ correspond to smoother sample paths of $X$. For example, Brownian motion satisfies this assumption with $a=1$, while integrated Brownian motion corresponds to $a=2$, consistent with its smoother sample paths.

\begin{remark}
Since both $\Gamma$ and $\widetilde{\Gamma}$ are diagonal operators on $\mathrm{Span}\{\phi_1,\phi_2,\ldots,\phi_{2N_{n,p}+1}\}$ with strictly positive eigenvalues, the orthogonal projection of $\beta$ onto $S_m$ with respect to the scalar product $\langle\cdot,\cdot\rangle_{\Gamma}$ coincides with its orthogonal projection with respect to $\langle\cdot,\cdot\rangle_{\widetilde{\Gamma}}$. We denote this common projection by $\beta^{(m)}$. The detailed proof of this result is established in Appendix \ref{sec:remark_ortho}.
\end{remark}
First let us state the following result which will be useful to derive oracle inequalities :
\begin{prop}
    \label{prop:1}

Assume that there exists $l>4$ such that $\upsilon_{l}:=\mathbb{E}(|\epsilon|^{l})<\infty$ and that $\mathbb{E}\left(\langle \beta, X\rangle_{L^{2}}^{4}\right)<\infty$. Moreover we suppose that assumptions  $\textbf{(H1)}, \textbf{(H2)}, \textbf{(H3)}$ are verified. If $$D_{N_{n,p}} < \min{\left(\frac{n}{\ln^{2}n},p\right)},$$
\[\lambda_{D_{N_{n,p}}}\geq \frac{2}{n^{\alpha}}\]and $\alpha>2a$ in the definition of $s_{n}$ (see in Equation \eqref{eq:s_n}), then for all slope $\beta\in L^{2}([0,1])$ we have
\begin{equation}
\begin{aligned}
    \mathbb{E}(\|\widetilde{\beta}-\beta\|_{\widetilde{\Gamma}_n}^{2}) 
    &\leq C_{0} \biggl[ \min_{m \in \mathcal{M}_{n,p}} \Bigl( \mathbb{E}(\|\beta^{(m)}-\beta\|_{\widetilde{\Gamma}_n}^{2}) + \mathrm{pen}(m) \Bigr) \\
    &\quad + \|\beta\|_{L^{2}}^{2} \mathbb{E}(\|\widetilde{X}-X\|_{L^{2}}^{2}) \\
    &\quad + \frac{\|\beta\|_{L^{2}}^{2}}{n} \left( \mathbb{E} \|\widetilde{X} - X\|_{L^2}^4 \right)^{1/2} + C_{\beta}\left(\frac{1}{n}+\frac{1}{np}\right) \biggr],
\end{aligned}
\end{equation}
where $\mathrm{pen}(m)$ is defined in (\ref{eq:penalty}), $C_{0}$ is a positive constant which depends on $K_1, \theta, l, \tau_{l}, \delta, \sigma^{2}$. Moreover we have: 
\begin{equation}\label{eq:c_beta}
    C_{\beta}=\mathbb{E}(\langle\beta,X_1\rangle_{L^{2}}^{4})^{1/2}+\|\beta\|_{L^{2}}^{2}+1.
\end{equation}
\end{prop}
The following result provides an oracle inequality with respect to the norm $\|\cdot\|_{\widetilde{\Gamma}}$, thereby evaluating the prediction risk based on the reconstructed curves.
\begin{theorem}
\label{thm:3.1}
Suppose that the assumptions of Proposition \ref{prop:1} hold. Then for all slope $\beta\in L^{2}([0,1])$,
\begin{equation}
\begin{aligned}
    \mathbb{E}(\|\widetilde{\beta}-\beta\|_{\widetilde{\Gamma}}^{2}) 
    &\leq C_{1} \biggl[ \min_{m \in \mathcal{M}_{n,p}} \Bigl( \mathbb{E}(\|\beta^{(m)}-\beta\|_{\widetilde{\Gamma}}^{2}) + \mathrm{pen}(m) \Bigr) \\
    &\quad + \|\beta\|_{L^{2}}^{2} \mathbb{E}(\|\widetilde{X}-X\|_{L^{2}}^{2}) \\
    &\quad + \frac{\|\beta\|_{L^{2}}^{2}}{n} \left( \mathbb{E} \|\widetilde{X} - X\|_{L^2}^4 \right)^{1/2} + C_{\beta}\left(\frac{1}{n}+\frac{1}{np}\right) \biggr],
\end{aligned}
\end{equation}
where $\mathrm{pen}(m)$ is defined in (\ref{eq:penalty}), $C_{1}$ is a positive constant which depends on $\theta, l, \tau_{l}, \delta, \sigma^{2}$ and $C_{\beta}$ is defined in (\ref{eq:c_beta}).
\end{theorem}
The following result provides an oracle inequality with respect to the norm $\|\cdot\|_{\Gamma}$, thereby evaluating the prediction risk based on the true curves.
\begin{theorem}
\label{thm:3.2}
    Suppose that the assumptions of Theorem \ref{thm:3.1} are satisfied. Then for all slope $\beta\in L^{2}([0,1])$ we have,
    \begin{equation}
\begin{split}
\mathbb{E}(\|\widetilde{\beta}-\beta\|_{\Gamma}^{2}) 
&\leq C_{2} \biggl[ \min_{m \in \mathcal{M}_{n,p}} \Bigl( \|\beta^{(m)}-\beta\|_{\Gamma}^{2} + \|\beta^{(m)}-\beta\|_{\widetilde{\Gamma}}^{2} + \mathrm{pen}(m) \Bigr) \\
&\quad + \|\beta\|_{L^{2}}^{2} \mathbb{E}(\|\widetilde{X}-X\|_{L^{2}}^{2}) \\
&\quad + \frac{\|\beta\|_{L^{2}}^{2}}{n} \left( \mathbb{E} \|\widetilde{X} - X\|_{L^2}^4 \right)^{1/2} + C_{\beta}\left(\frac{1}{n}+\frac{1}{np}\right)\biggr],
\end{split}
\end{equation}
where $\mathrm{pen}(m)$ is defined in (\ref{eq:penalty}), $C_{2}$ is a positive constant which depends on $K_1, \theta, l, \tau_{l}, \delta, \sigma^{2}$, $c'$, $a$ and $C_{\beta}$ is defined in (\ref{eq:c_beta}).

\medskip\noindent
If we assume furthermore that $\beta\in W^{\text{per}}(k,L)$ we obtain the following oracle inequality :
\begin{equation}\label{eq:second_oracle_3.2}
\begin{aligned}
    \mathbb{E}(\|\widetilde{\beta}-\beta\|_{\Gamma}^{2}) 
    &\leq C_{3} \biggl[ \min_{m \in \mathcal{M}_{n,p}} \Bigl( \|\beta^{(m)}-\beta\|_{\Gamma}^{2} + \mathrm{pen}(m) \Bigr) \\
    &\quad + \mathbb{E}(\|\widetilde{X}-X\|_{L^{2}}^{2}) \\
    &\quad + \frac{1}{n} \left( \mathbb{E} \|\widetilde{X} - X\|_{L^2}^4 \right)^{1/2} + \frac{1}{p} +\frac{1}{n}+ \frac{1}{np}\biggr],
\end{aligned}
\end{equation}
where $C_{3}$ is a positive constant which depends on $K_1, \theta, l, \tau_{l}, \delta, \sigma^{2}, c'$, $a$.
\end{theorem}
The sequence of Proposition \ref{prop:1} and Theorems \ref{thm:3.1}--\ref{thm:3.2} reflects a progressive refinement of the risk bounds. Proposition \ref{prop:1} establishes an oracle type inequality expressed in terms of the empirical covariance operator of the reconstructed data, which is the most directly observable quantity in our framework. Theorem \ref{thm:3.1} then replaces this empirical operator by the covariance operator of the reconstructed process, thereby removing the effect of sampling variability. Finally, Theorem \ref{thm:3.2} transfers the result to the covariance operator of the true underlying process by controlling the discrepancy between the reconstructed and true covariance structures.

\medskip\noindent
The bounds obtained in Theorems \ref{thm:3.1} and \ref{thm:3.2} are of oracle type. They show that the estimator $\widetilde{\beta}$ performs, up to a multiplicative constant and additional remainder terms, nearly as well as the best estimator in the collection $(S_m)_{m\in\mathcal{M}_{n,p}}$. More precisely, the leading term in the bound of Theorem \ref{thm:3.1} is governed by the oracle criterion
\[\min_{m\in\mathcal{M}_{n,p}}\left(\|\beta^{(m)}-\beta\|_{\widetilde{\Gamma}}^{2}+\mathrm{pen}(m)\right),\]
whereas Theorem \ref{thm:3.2} involves
\[\min_{m\in\mathcal{M}_{n,p}}\left(\|\beta^{(m)}-\beta\|_{\Gamma}^{2}+\mathrm{pen}(m)\right).\]
The additional terms involving $\|\widetilde{X}-X\|_{L^2}$ quantify the error induced by reconstructing the curves from discrete noisy observations and therefore measure the impact of the preliminary smoothing step on the final estimation procedure. The remaining terms, of order $1/n$, $1/p$, and $1/(np)$, arise from the proof of the oracle inequalities.

\section{Convergence rates over Sobolev spaces}
In this part, we derive convergence rates for the error $\mathbb{E}\left(\|\widetilde{\beta}-\beta\|_{\Gamma}^{2}\right)$.
\begin{theorem}\label{thm:4.1}
Suppose that the assumptions of Theorem \ref{thm:3.2} are satisfied, with the exception that we consider $a\geq 1$ instead of $a>1/2$ in \textbf{(H3)}. Then, for all $\beta\in W^{\text{per}}(k,L)$ :
\begin{itemize}
     \item If $p \gtrsim \bigl(\frac{n}{\ln^{2}n}\bigr)^{2a}$ then,
     \begin{equation}
         \mathbb{E}\Bigl(\|\widetilde{\beta}-\beta\|_{\Gamma}^{2}\Bigr)=\mathcal{O}\Bigl(n^{-\frac{2a+2k}{2a+2k+1}}\Bigr).
     \end{equation}
     \item If $n^{\frac{2a}{2a+2k+1}}\lesssim p \lesssim \bigl(\frac{n}{\ln^{2}n}\bigr)^{2a}$ then,
     \begin{equation}
         \mathbb{E}\Bigl(\|\widetilde{\beta}-\beta\|_{\Gamma}^{2}\Bigr)=\mathcal{O}\Bigl(n^{-\frac{2a+2k}{2a+2k+1}}+p^{-\frac{2a-1}{2a}}\Bigr).
     \end{equation}
     
     \item If $p \lesssim n^{\frac{2a}{2a+2k+1}}$ then,
     \begin{equation}
         \mathbb{E}\Bigl(\|\widetilde{\beta}-\beta\|_{\Gamma}^{2}\Bigr)=\mathcal{O}\Bigl(p^{-\frac{2a-1}{2a}}\Bigr).
     \end{equation}
     
\end{itemize}

\end{theorem}

\medskip\noindent
The stronger assumption $a \geq 1$ is required in Theorem \ref{thm:4.1} to control the contribution of the reconstruction error and to ensure that the tail sum of the eigenvalues decays sufficiently fast. Moreover, the three regimes described in Theorem \ref{thm:4.1} reveal a phase transition phenomenon that depends on the relative magnitudes of $p$ and $n$. When $p$ is large, the problem behaves essentially as if the curves were fully observed. In contrast, when $p$ is small, the estimation error is dominated by the discretization error, and increasing the sample size $n$ alone is not sufficient to improve the convergence rate.
Furthermore, when the number of observation points is sufficiently large, namely $$p \gtrsim \left(\frac{n}{\ln^2 n}\right)^{2a},$$
the convergence rate obtained in Theorem \ref{thm:4.1} coincides with the optimal rate established in the literature for functional linear models with fully observed curves. In particular, we recover the rate derived by Brunel and Roche \cite{Brunel2015penalized}. Finally, as shown in the proof of Theorem \ref{thm:4.1}, the dependence of the rate on $p$ is primarily driven by the observation noise rather than by the reconstruction procedure itself.

\section{Lower Bound for the case \texorpdfstring{$p$}{p} large}
In this section we derive a lower bound for the case 1 of Theorem \ref{thm:4.1}.
\begin{theorem}\label{thm:5.1}
    Let us assume that there exist two constants $c'>0$ and $ a>1/2$ such that for all $j\geq 1$, 
    $$j^{-2a}/c'\leq \lambda_{j}\leq c'j^{-2a}.$$
    We also make the further assumption that 
    \[\epsilon_{i}\sim \mathcal{N}(0,\sigma^{2}) \qquad \forall i\in\{1,\ldots,n\},\]
    and that 
    \[\eta_{i,h}\sim \mathcal{N}(0,\tau^{2})\qquad \forall i\in\{1,\ldots,n\}, \forall h\in\{0,\ldots,p-1\}.\]
    Then for $p\geq \left(\frac{n}{\ln^{2}{n}}\right)^{2a}$ and for all $\beta\in W^{\text{per}}(k,L)$,
    \[\mathbb{E}\Big(\|\widetilde{\beta}-\beta\|_{\Gamma}^{2}\Big)\geq C_{4}n^{-\frac{2a+2k}{2a+2k+1}},\]
    where $C_4$ is a positive constant which depends on $k,L,c',\sigma$.
\end{theorem}
In Theorem \ref{thm:4.1}, only the upper bound $\lambda_j \leq c'j^{-2a}$ is required. Indeed, the proof relies on controlling approximation and reconstruction errors, which involve tail sums of the eigenvalues and therefore only require an upper estimate on their decay.
In contrast, the proof of Theorem \ref{thm:5.1} is based on a minimax lower-bound construction. To ensure that the candidate slope functions defined in (\ref{eq:beta_test}) belong to $W^{\text{per}}(k,L)$ and remain sufficiently separated in the prediction norm $\|\cdot\|_\Gamma$, it is necessary to control both $\lambda_j$ and $1/\lambda_j$. This requires the two-sided condition $\frac{1}{c'}j^{-2a}\leq \lambda_j \leq c'j^{-2a}$. Furthermore, Theorem \ref{thm:5.1} establishes that, when $p$ is sufficiently large, the rate $n^{-\frac{2a+2k}{2a+2k+1}}$ constitutes a lower bound for the prediction risk over the class $W^{\text{per}}(k,L)$. Combined with the upper bound derived in Theorem \ref{thm:4.1}, this shows that our estimator attains the minimax rate over this class.

\section{Simulations}
Across the next four paragraphs, we present a framework for evaluating our estimator on a simulated dataset. This framework consists of generating the true functional curves, computing the corresponding scalar outputs, creating noisy discrete observations, reconstructing the curves, choosing adequate values for $\theta$ and $\delta$, estimating the slope function, and finally assessing the model’s performance. We study here three slope functions 
\begin{equation}\label{eq:beta_1}
    \beta_1(t) = t(t-1), \qquad t\in[0,1],
\end{equation}
which is in $W^{\text{per}}(1,1)$ and also studied in Brunel and Roche \cite{Brunel2015penalized}. Since for all $t\in[0,1]$ $\beta_{1}^{(1)}(t)=2t-1$, therefore $\beta_{1}^{(1)}(0)\neq \beta_{1}^{(1)}(1)$ and $\beta_1\notin W^{\text{per}}(2,1)$.
\begin{equation}\label{eq:beta_2}
    \beta_2(t) = 3\exp\left(-\frac{(t - 0.30)^2}{2 \cdot 0.07^2}\right) - 2\exp\left(-\frac{(t - 0.70)^2}{2 \cdot 0.08^2}\right), \qquad t\in[0,1],
\end{equation}
which is defined as a linear combination of two Gaussian functions and
\begin{equation}\label{eq:beta_3}
    \beta_3(t) = 4\sin(4\pi t) - \operatorname{sign}(t - 0.3) - \operatorname{sign}(0.72 - t), \qquad t\in[0,1],
\end{equation}
which corresponds to the heavisine function.
\begin{figure}[htbp]
    \centering
    
    \begin{subfigure}[b]{0.32\textwidth}
        \centering
        \includegraphics[width=\textwidth]{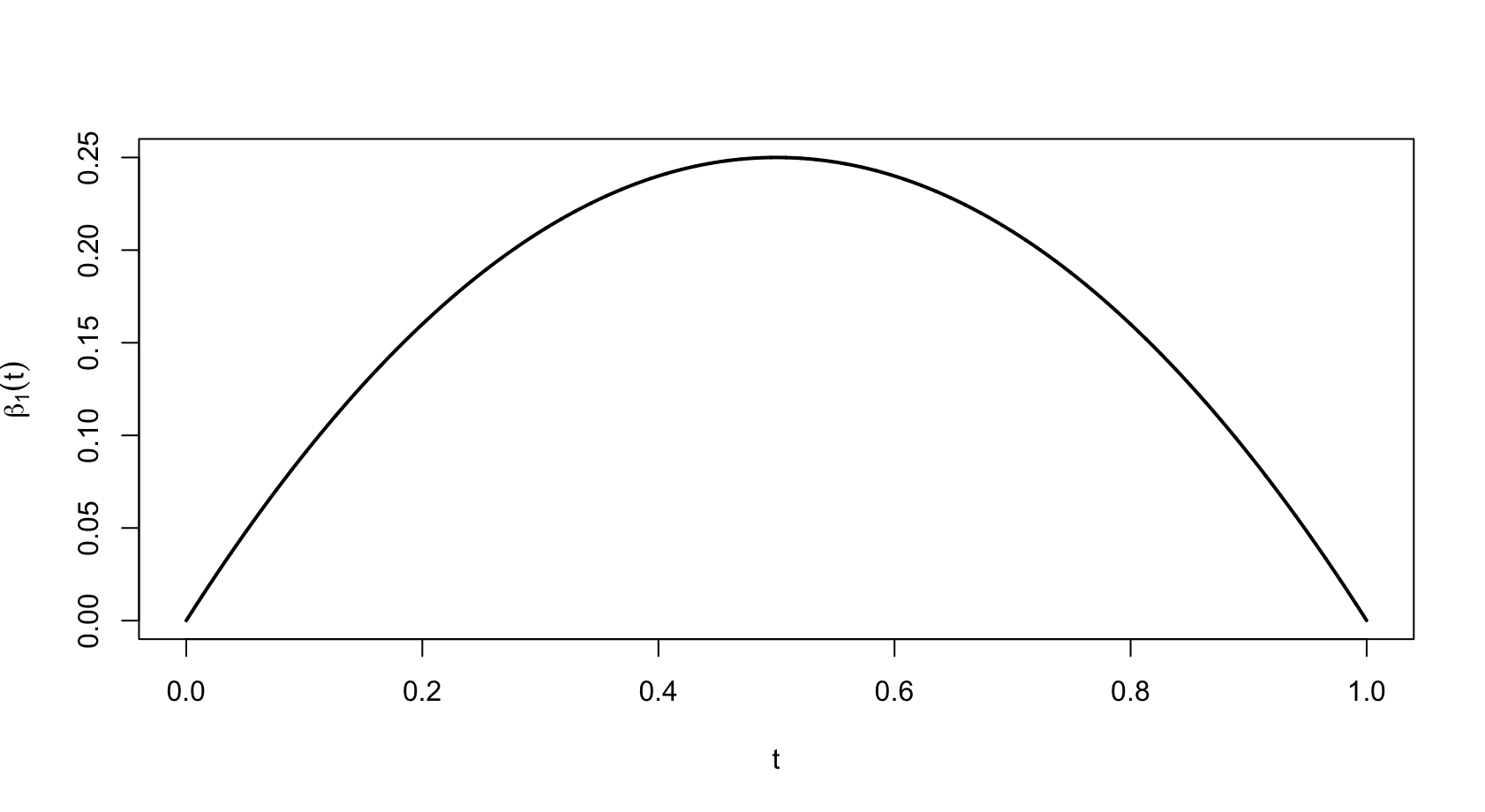} 
        \label{fig:beta_1}
    \end{subfigure}
    \hfill
    \begin{subfigure}[b]{0.32\textwidth}
        \centering
        \includegraphics[width=\textwidth]{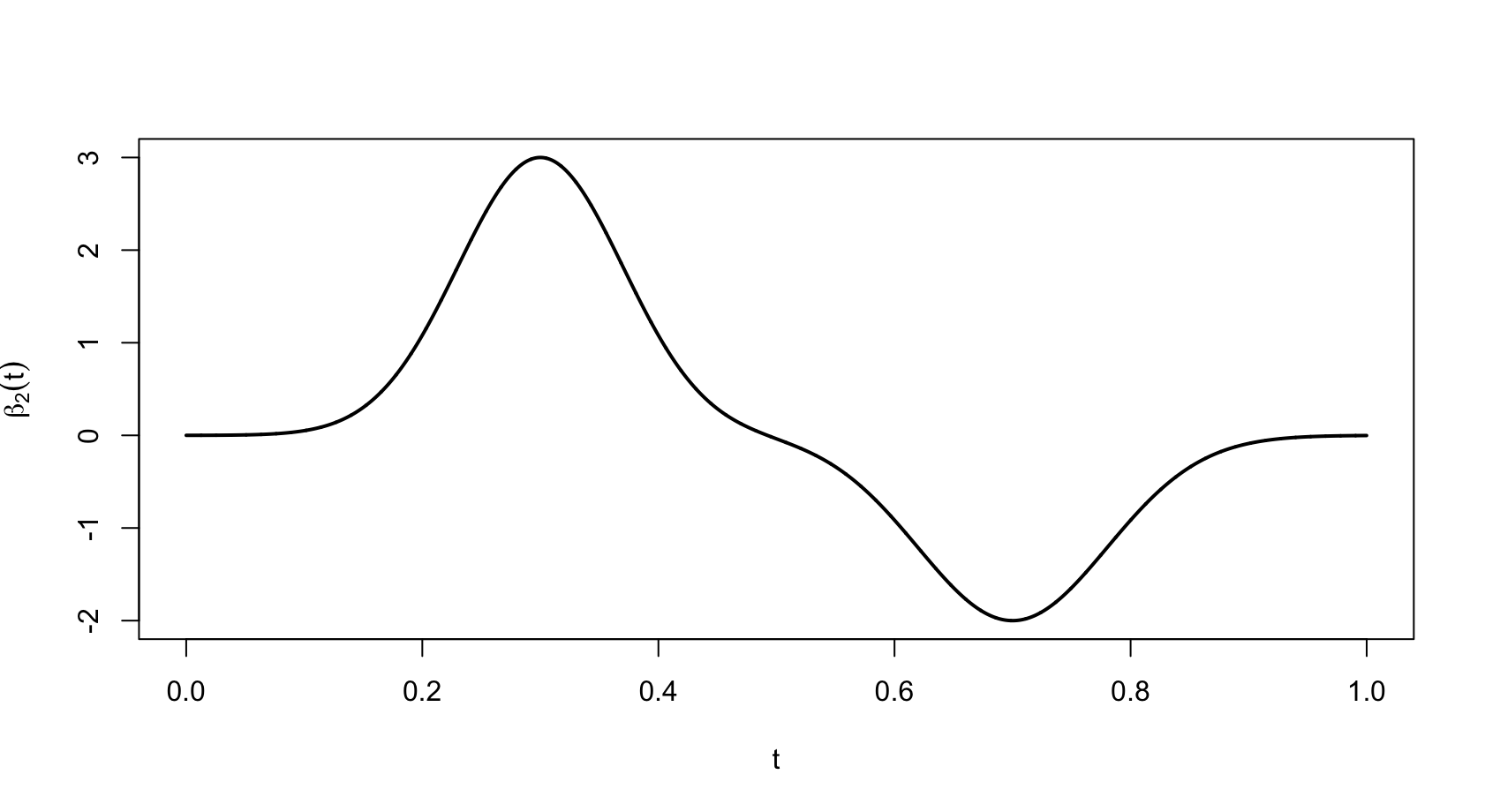} 
        \label{fig:beta_22}
    \end{subfigure}
    \hfill
    \begin{subfigure}[b]{0.32\textwidth}
        \centering
        \includegraphics[width=\textwidth]{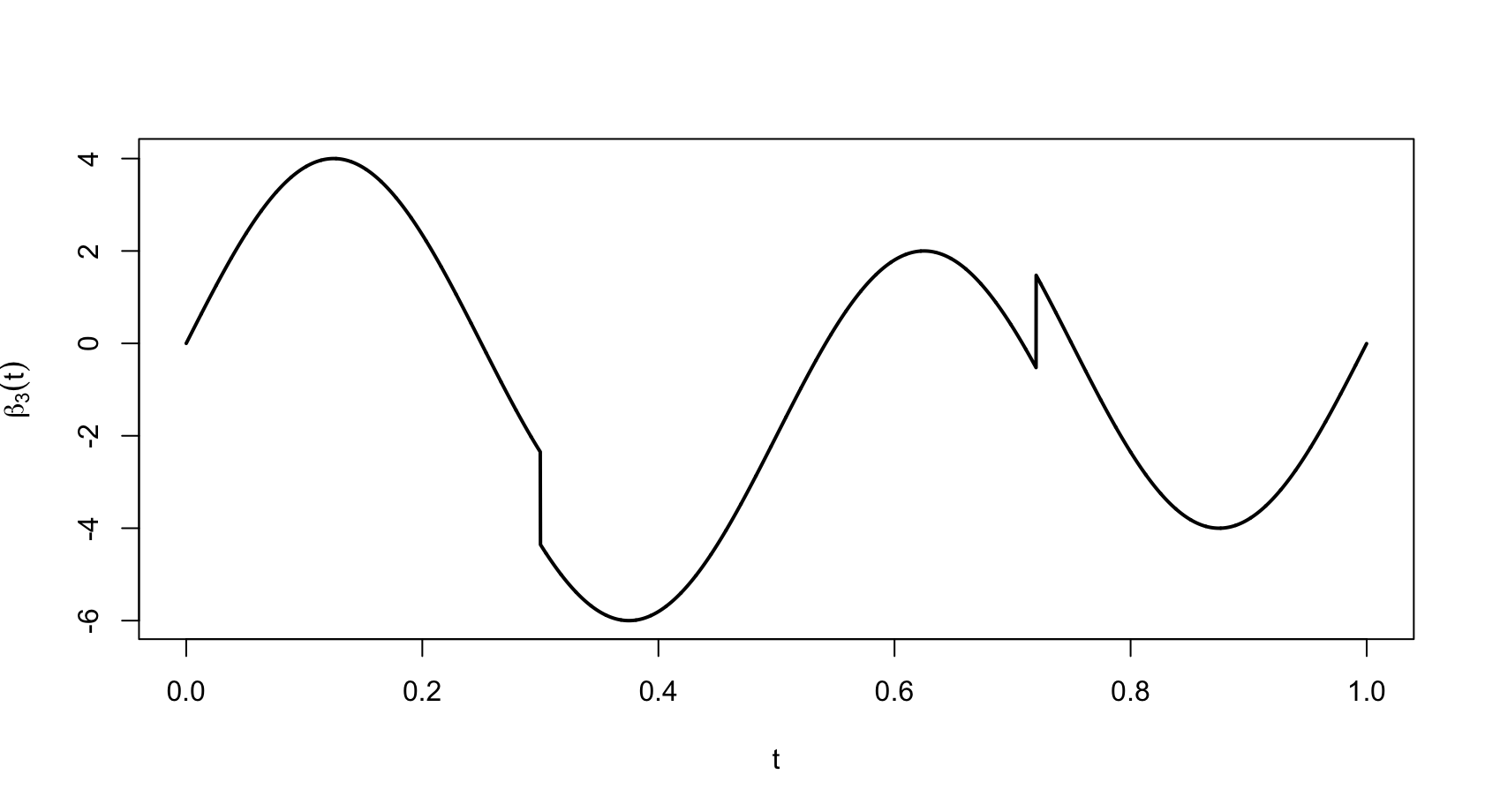} 
        \label{fig:beta_3}
    \end{subfigure}
    
    \caption{Plot of $\beta_1$ (left), $\beta_2$ (middle), $\beta_3$ (right)}
    \label{fig:beta_simu}
\end{figure}

\subsection{Simulation setup and data generation}
The simulation study begins with the generation of the true functional curves, denoted $X_i$, representing the underlying functional processes. These curves are constructed using a truncated Karhunen--Loève decomposition:
\begin{equation}
X(t)=\sum_{j=1}^{2J+1}\xi_{j}\phi_{j}(t), \qquad t\in[0,1].
\end{equation}
Here, the sequence of independent, centered random variables $\{\xi_{1},...,\xi_{2J+1}\}$ has variances $\mathrm{Var}(\xi_{j})=\lambda_{j}$. We set $J=500$ and consider the $\xi_{j}$'s to be Gaussian.

\medskip\noindent
To generate the true curves $(X_i)_{i=1,\ldots,n}$, we first construct a high-resolution grid $G_1$ over the observation interval $[0, 1]$, defined as:
\[G_1 := \left\{ \frac{j}{p_{\mathrm{sim}} - 1} \;\middle|\; j = 0, \dots, p_{\mathrm{sim}} - 1 \right\}.\]
This grid consists of $p_{\mathrm{sim}} = 10000$ equally spaced points and provides a discrete approximation of the underlying continuous functions. Because it is not feasible to simulate the curves at every point in continuous time, this fine, regular grid serves as a practical compromise between approximation accuracy and computational cost. Next, we define a lower-resolution observation grid, $G_2$, consisting of $p$ regularly spaced points where the noisy and discrete data are actually observed. This grid is defined as:
\[G_2 := \left\{ t_h = \frac{h}{p} \;\middle|\; h = 0,\ldots,p-1 \right\}.\]
Because these observation points do not necessarily coincide with those of the high-resolution grid, we augment the latter to include $G_2$. The true curves are generated on this combined grid, $G_1 \cup G_2$.

\medskip\noindent
Next, independent Gaussian noise terms $\eta_{i,h}$, with zero mean and variance $\tau^2$ are added to the discretized observations. The value of $\tau^{2}$ will vary and thus will be precised in each sections. This step aims to mimic measurement error in practical settings, yielding the observed sample $(Z_i)_{i=1,...,n}$, defined as:
\begin{equation}
Z_{i}(t_{h})=X_{i}^{obs}(t_{h})+\eta_{i,h}, \qquad t_h\in G_2.
\end{equation}

\noindent
Once the functional predictors are generated, the corresponding scalar response $Y_i$ for each curve is simulated according to the functional linear model defined in \eqref{eq:Y_i_first}. To compute the integral term $\int_{0}^{1}X_i(t)\beta(t)dt$, we employ a discrete numerical approximation evaluated over the grid $G_1$. In particular,
\[\int_{0}^{1}X_i(t)\beta(t)dt \approx \frac{1}{p_{\mathrm{sim}}-1}\sum_{j=0}^{p_{\mathrm{sim}}-1}X_i(t_j)\beta(t_j).\]
Finally, the noise variables $\epsilon_i$ are generated independently from a zero-mean Gaussian distribution with variance $\sigma^2=0.1$.

\subsection{Estimation of the slope function}
Prior to estimating $\beta$, the functional predictors are reconstructed from the noisy and discretized observations. Following the procedure described in Section \ref{sec:curve_reconstruct}, we first compute the coefficients $(\widetilde{x}_{i,j})_{i=1,...,n\text{ ; }j=1,...,D_{N_{n,p}}}$ as defined in Equation (\ref{eq:3}). These coefficients are then used to reconstruct the curves $\widetilde{X}_{i}$ according to Equation (\ref{eq:4}). This procedure yields a sample of reconstructed curves $(\widetilde{X}_{i})_{i=1,...,n}$, which can be interpreted as smoothed versions of the original observations.
\begin{figure}[htbp] 
    \centering 
    \includegraphics[width=0.9\textwidth]{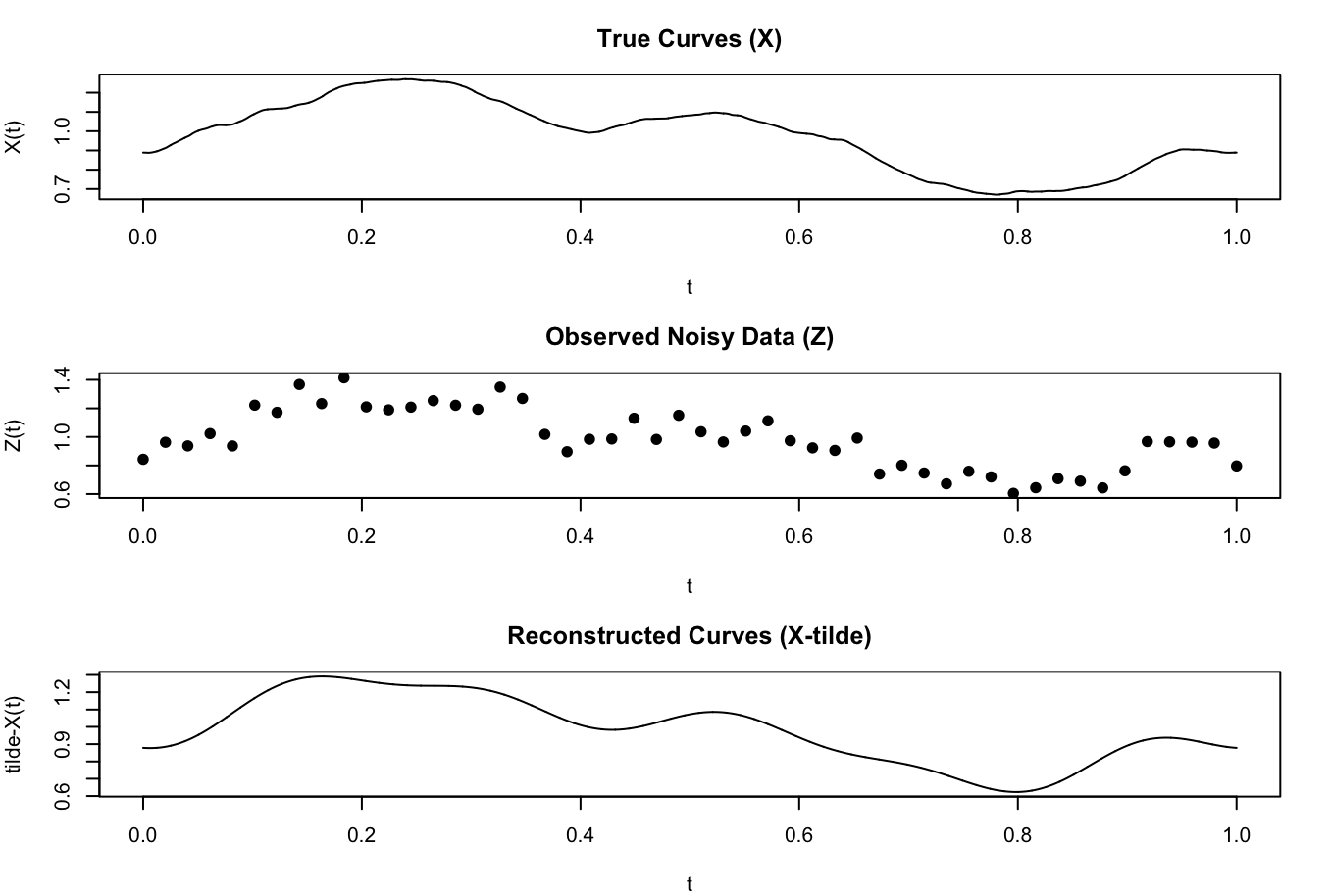} 
    \caption{Simulation of a functional curve: true, observed, and reconstructed. Here $p=50$, $p_{sim}=10000$, $D_{N_{n,p}}=11$, $\tau^{2}=0.1$ and $\lambda_{j}=1/j^4$ for all $j=1,\ldots,2J+1$.}
    \label{fig:1} 
\end{figure}
The choice of the basis dimension $D_{N_{n,p}}$ matters. In particular, $D_{N_{n,p}}$ must be chosen so as to ensure the invertibility of the empirical covariance matrix with high probability, and therefore to be able to compute $\widetilde{\beta}$ defined in \eqref{eq:widetilde_beta}. Theorem~\ref{thm:3.1} shows that this condition is satisfied whenever
\[D_{N_{n,p}} < \min\left(\frac{n}{\ln^2 n}, p\right).\]
Accordingly, in our simulations, we select the largest value of $D_{N_{n,p}}$ satisfying this constraint in order to achieve the most accurate curve reconstruction possible.

\medskip\noindent
Based on the reconstructed curves, we proceed to estimate the slope function $\beta$. To this end, we adopt a penalized model selection approach to determine the appropriate model complexity, as described in Section \ref{sec:model}. The performance of the resulting estimator is supported by an oracle inequality (Theorem \ref{thm:3.1}), provided that the penalty parameters satisfy $\theta > 4$ and $\delta > 0$. In practice, these parameters must be calibrated; the corresponding selection procedure is detailed in Section \ref{sec:kappa}.

\subsection{Prediction error}\label{sec:prediction_error}
As discussed in the previous sections, we focus on the prediction error of the proposed estimator. More specifically, we aim to analyze the behavior of
\[\mathbb{E}\left(\|\widetilde{\beta}-\beta\|_{\Gamma}^{2}\right).\]
As $\|\widetilde{\beta}-\beta\|_{\Gamma}^{2} = \sum_{j\geq 1}\lambda_{j}\langle\widetilde{\beta}-\beta,\phi_{j}\rangle_{L^{2}}^{2}$, we approximate this quantity by using a finite number of elements in the Fourier basis,
\[\|\widetilde{\beta}-\beta\|_{\Gamma}^{2} \approx \sum_{j= 1}^{2J+1}\lambda_{j}\langle\widetilde{\beta}-\beta,\phi_{j}\rangle_{L^{2}}^{2}\approx\sum_{j=1}^{2J+1}\lambda_j\left((\widetilde{\alpha}_j-\beta_j)^{2}\mathds{1}_{j\leq D_{\widehat{m}}}+\beta_{j}^{2}\mathds{1}_{j>D_{\widehat{m}}}\right),\]
where the last equality is derived from the definition of $\widetilde{\beta}$ given in Equation (\ref{eq:widetilde_beta}).
Finally we estimate our prediction error using a Monte-Carlo method with $n_{MC}$ independent samples. In particular we derive $n_{MC}$ estimators of $\beta$ that we denote as $\widetilde{\beta}^{(l)}$ for $l=1,...,n_{MC}$. For each replication $l \in \{1, \dots, n_{\mathrm{MC}}\}$, we obtain an estimator $\widetilde{\beta}^{(l)}$ defined by Equation (\ref{eq:widetilde_beta}) and its definition is given on $\overline{G}$ by $\widetilde{\beta}^{(l)} = \sum_{j=1}^{D_{\widehat{m}^{(l)}}} \widetilde{\alpha}_j^{(l)} \phi_j$, where $\widehat{m}^{(l)}$ denotes the model selected at iteration $l$. We then compute an estimate of the prediction error:
\[\mathbb{E}\left(\|\widetilde{\beta}-\beta\|_{\Gamma}^{2}\right) \approx \frac{1}{n_{MC}}\sum_{l=1}^{n_{MC}}\widehat{e}_l,\]
where $\widehat{e}_l=\sum_{j=1}^{2J+1}\lambda_j\left((\widetilde{\alpha}_{j}^{(l)}-\beta_j)^{2}\mathds{1}_{j\leq D_{\widehat{m}^{(l)}}}+\beta_{j}^{2}\mathds{1}_{j>D_{\widehat{m}^{(l)}}}\right)$. For our simulations we use $n_{MC}=50$.

\subsection{Calibration of \texorpdfstring{$\theta$ and $\delta$}{theta and delta}} \label{sec:kappa}
The constant $\kappa := \theta(1+\delta)$, which appears in the penalty term defined by \eqref{eq:penalty}, must be calibrated. To determine a suitable value for it, we evaluate its predictive performance across a grid of candidate values ranging from $1$ to $20$ (with a step size of $0.5$), holding the sample size $n=1000$ and the number of observation points $p=1000$ fixed. Specifically, we conduct a Monte Carlo simulation study to assess the candidate values across the three distinct slope functions $\beta_1$, $\beta_2$, and $\beta_3$. For each slope function, the prediction error, which is approximated via the procedure detailed in Section \ref{sec:prediction_error}, is averaged over $50$ independent replications. We then calculate the overall mean error across all three scenarios to identify a single, universal $\kappa$ that minimizes the aggregate prediction error. The results of this calibration process are visualized in Figure \ref{fig:kappa}. As indicated by the vertical dashed line, the aggregate prediction error attains its minimum at $\kappa=3$. However, given the constraints $\theta>4$ and $\delta>0$, the definition of the penalty parameter strictly requires $\kappa > 4$. Consequently, we select $\kappa=4.01$, a value situated appropriately above this theoretical boundary but close to the lowest aggregate prediction error, for all subsequent simulations.
\begin{figure}[htbp]
    \centering 
    \includegraphics[width=0.7\textwidth]{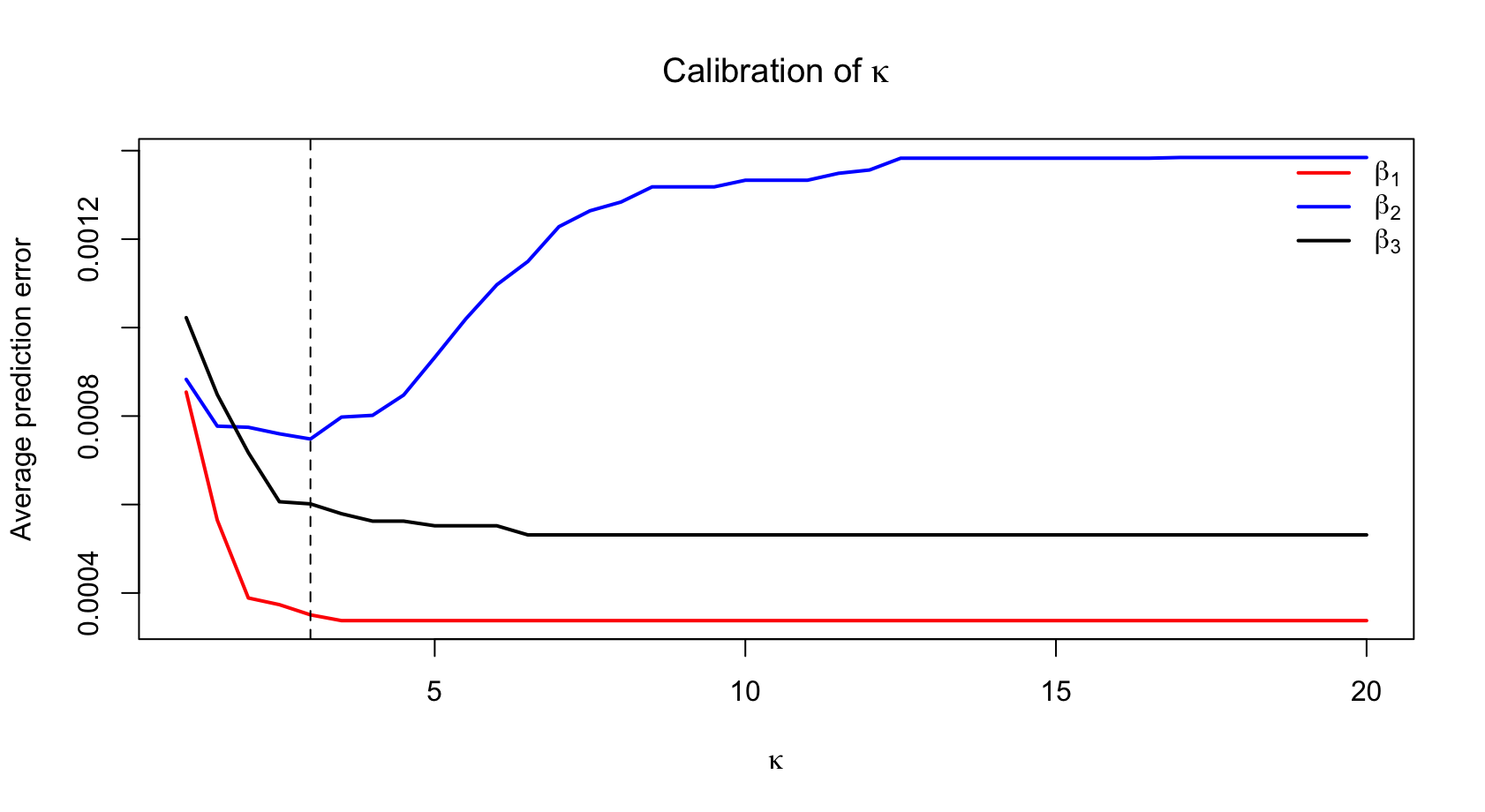} 
    \caption{Average prediction error as a function of the penalty parameter $\kappa \in [1, 20]$ across the three slope functions $\beta_1$, $\beta_2$, $\beta_3$.}
    \label{fig:kappa} 
\end{figure}

\subsection{Results}
We now study the performance of our estimator when various slope functions are considered.

\subsubsection{Study of \texorpdfstring{$\beta_1$}{beta 1}}

Now we consider $\beta_1$ as the true slope function, which belongs to $W^{\text{per}}(1,1)$ and $\lambda_{j}=1/j^4$ for all $j=1,\ldots,2J+1$. 
\begin{figure}[htbp]
    \captionsetup[subfigure]{labelformat=simple, labelsep=period}
    \centering
    \begin{subfigure}[b]{0.45\textwidth}
        \includegraphics[width=\textwidth]{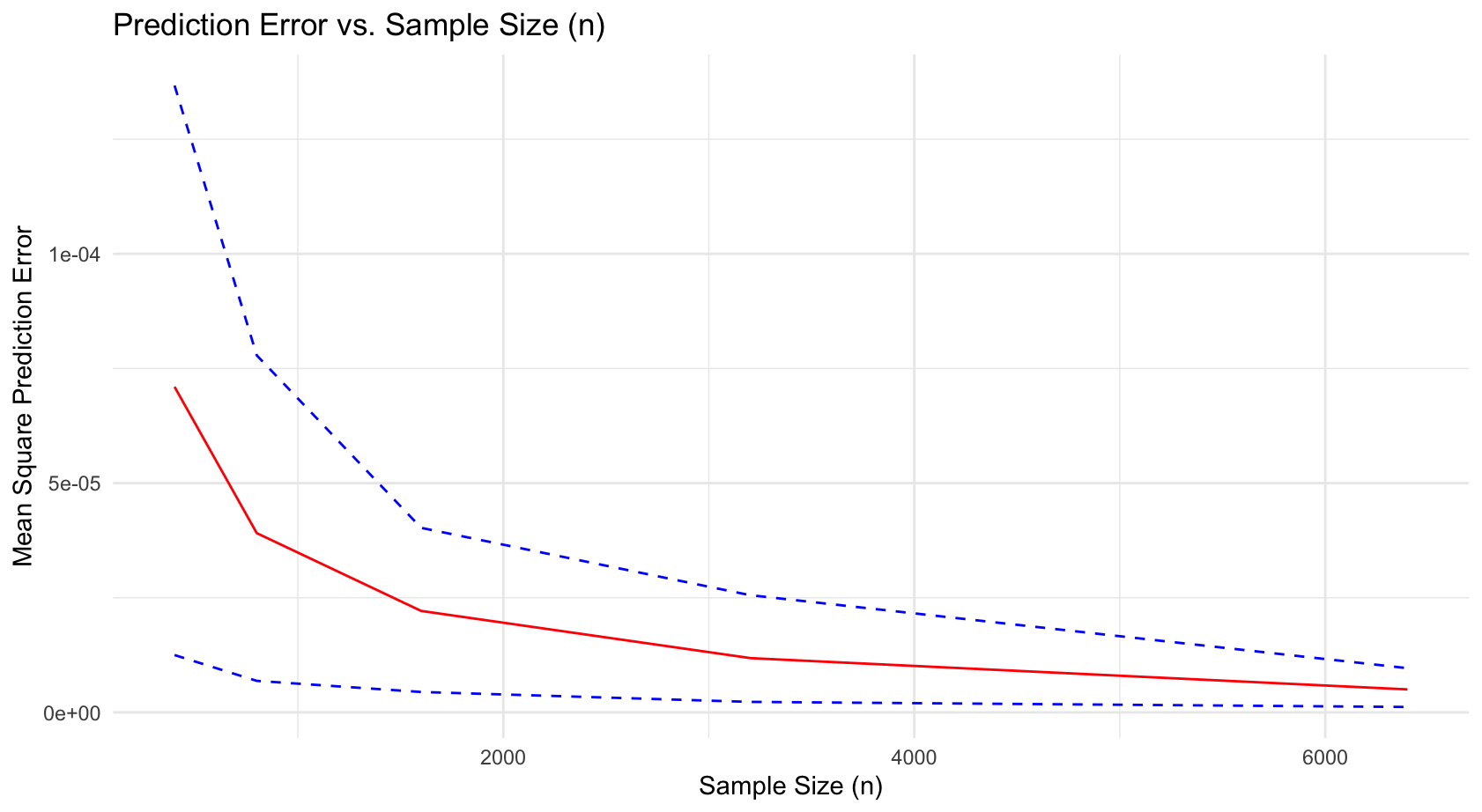}
        \caption{Average prediction error vs. sample size $n$}
        \label{fig:sub-1}
    \end{subfigure}
    \hfill
    \begin{subfigure}[b]{0.45\textwidth}
        \includegraphics[width=\textwidth]{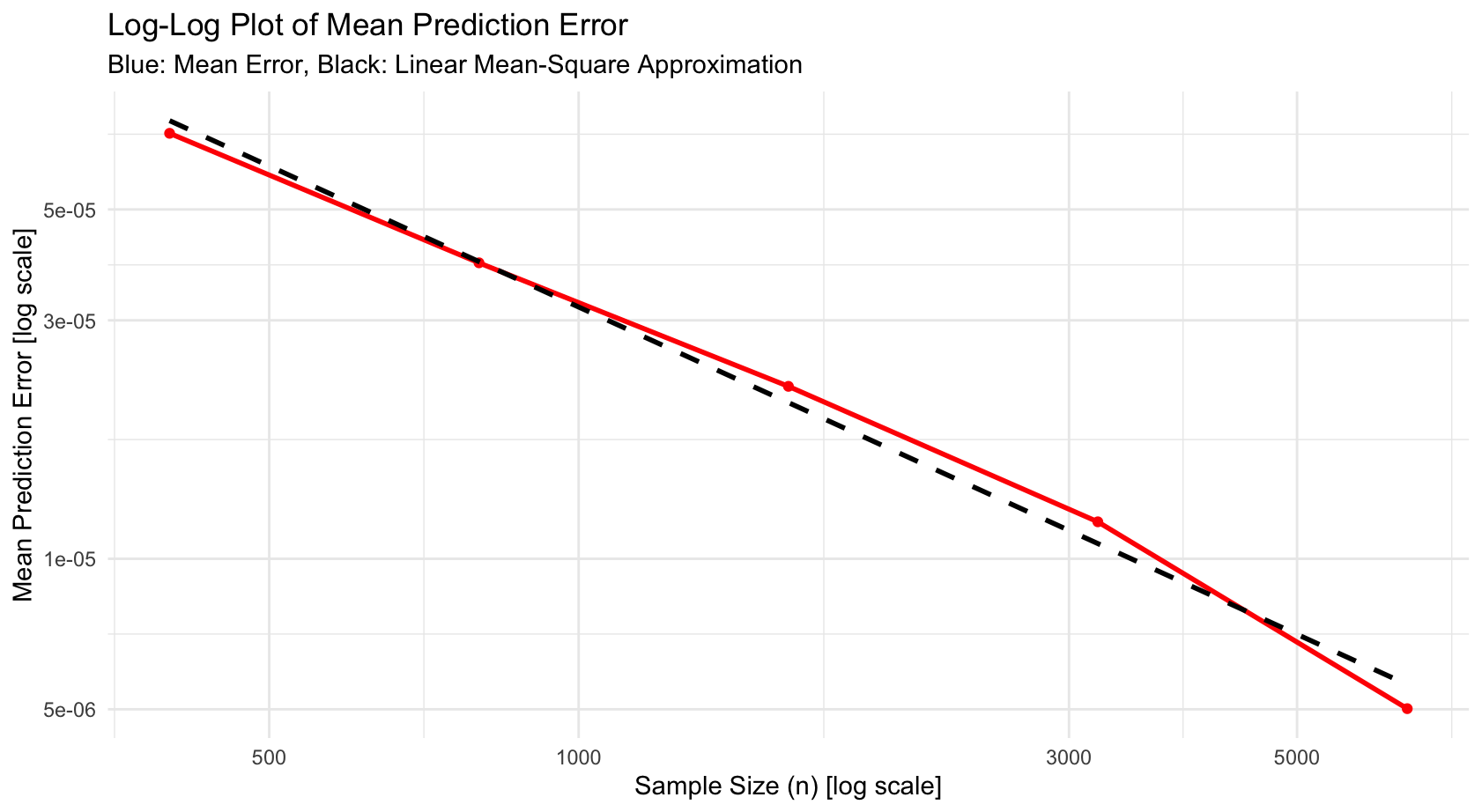}
        \caption{Log-log plot of average prediction error vs. sample size $n$}
        \label{fig:sub-2}
    \end{subfigure}

    \vspace{0.5cm} 

    \begin{subfigure}[b]{0.45\textwidth}
        \includegraphics[width=\textwidth]{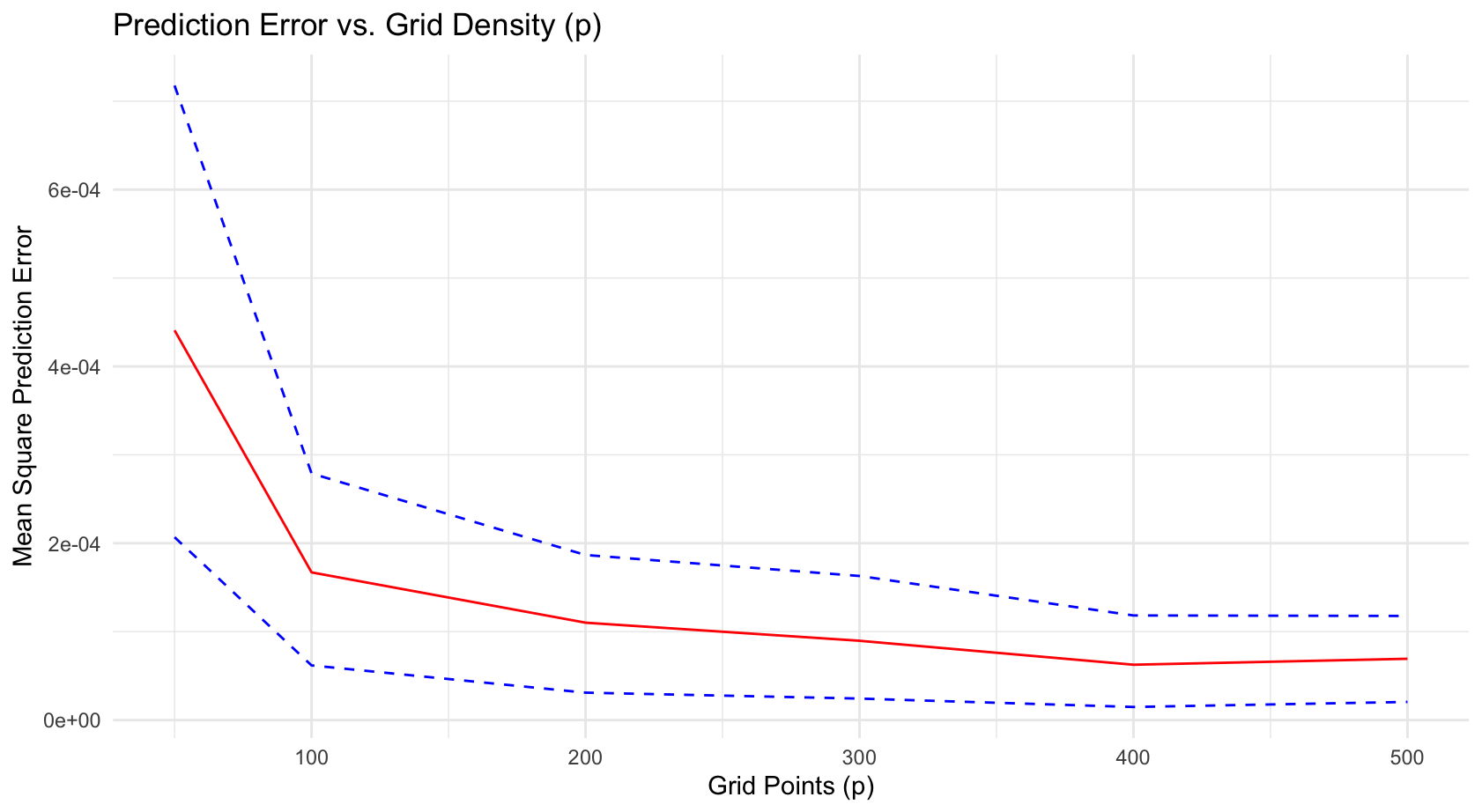}
        \caption{Average prediction error vs. number of observation points $p$}
        \label{fig:sub-3}
    \end{subfigure}
    \hfill
    \begin{subfigure}[b]{0.45\textwidth}
        \includegraphics[width=\textwidth]{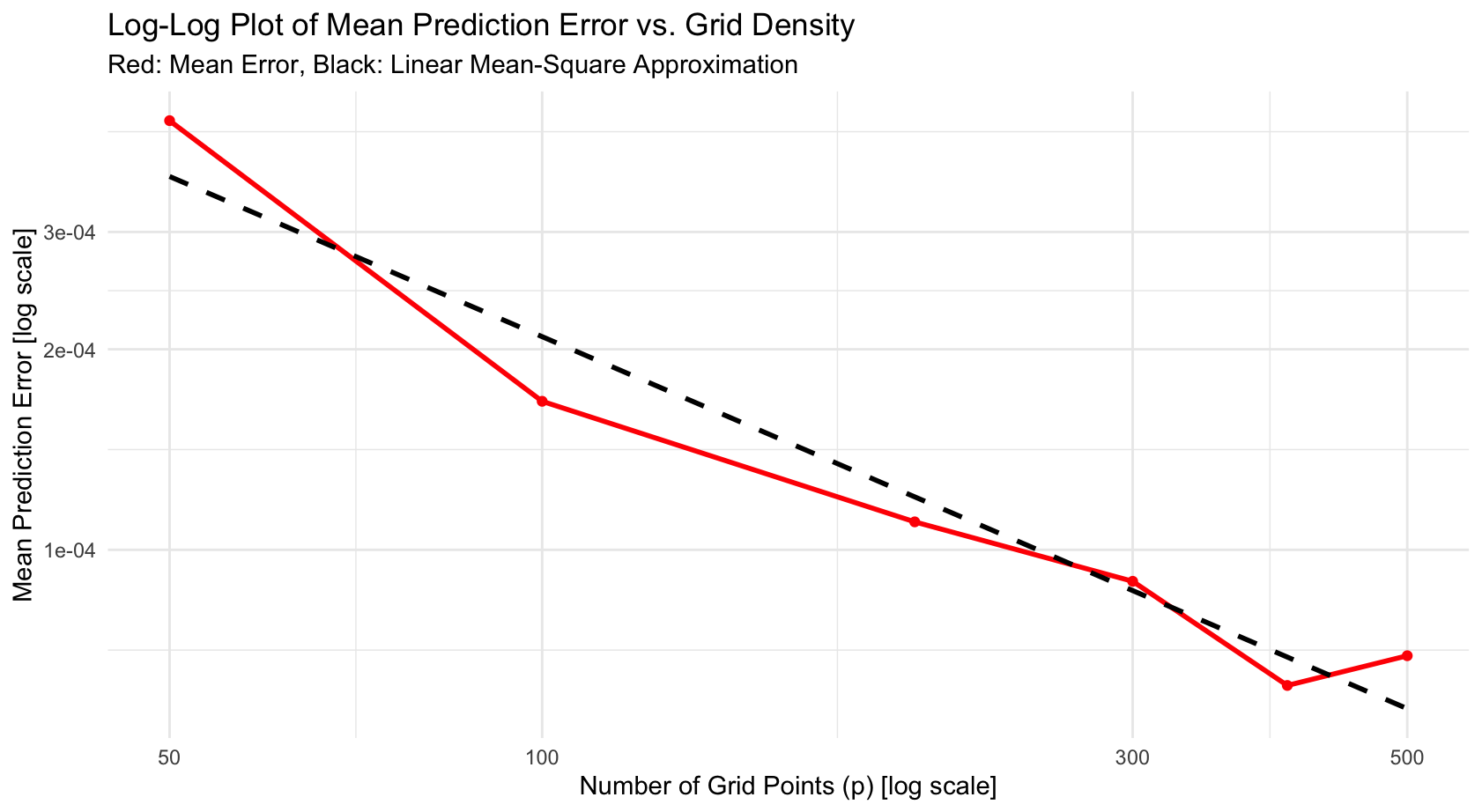}
        \caption{Log-log plot of average prediction error vs. number of observation points $p$}
        \label{fig:sub-4}
    \end{subfigure}

    \caption{Monte Carlo approximation of the mean squared $\Gamma$-norm error as a function of the sample size $n$ and the grid density $p$. The solid red curves represent the empirical error averaged over the Monte Carlo replications. In panels a. and c. the blue dashed curves represent the 10th and 90th empirical deciles of the Monte Carlo errors. Panels b. and d. display the corresponding mean errors on a log-log scale; in these panels, the black dashed lines represent linear fits used to assess the empirical rate of decay of the error.}
    \label{fig:2}
\end{figure}

\medskip\noindent
In Figures \ref{fig:sub-1}, \ref{fig:sub-2} we analyze the impact of the sample size $n$ by setting a fixed $p=7500$. The variance of the noises $\tau^{2}$ and $\sigma^{2}$ are both set to 0.1. The sample size, $n$, varies across the set $\{200, 400, 800, 1600, 3200, 6400\}$. The resulting log-log plot yields a slope of $-0.92$, which is quite close to the value that we can expect given the convergence rate established in Theorem \ref{thm:4.1}. Indeed, here we have $k=1$ and $a=2$ hence we expect to have an error which decreases in $n^{-(2k+2a)/(2a+2k+1)}=n^{-6/7}$.

\medskip\noindent
In Figures \ref{fig:sub-3}, \ref{fig:sub-4} we study the impact of the size of the grid $p$ by setting a fixed $n=1000$. The variance of the noise $\sigma^{2}=0.1$ but we use here a much higher variance for the noise of the observation $\tau^{2}=4$. The number of observations, $p$, on the grid varies across the set $\{50, 100, 200, 300, 400, 500\}$. The resulting log-log plot yields a slope of $-0.79$ which matches the rate that we expect from Theorem \ref{thm:4.1}. Indeed, here we have $a=2$ hence we expect to have an error which decreases in $p^{-(2a-1)/(2a)}=p^{-3/4}$. However, for smaller values of $\tau$, the rate of decrease is significantly higher. This could be explained by the fact that, when $\tau$ is small, the influence of noise on the reconstruction error of the curves is reduced; instead, the error is primarily driven by the choice of reconstruction scheme.

\medskip
We also studied the case where the true curves are less smooth. In particular we consider now $\lambda_{j}=1/j^2$ for all $j=1,\ldots,2J+1$. All the parameters and considered function remain unchanged. 
\begin{figure}[htbp]
    \captionsetup[subfigure]{labelformat=simple, labelsep=period}
    \centering
    \begin{subfigure}[b]{0.45\textwidth}
        \includegraphics[width=\textwidth]{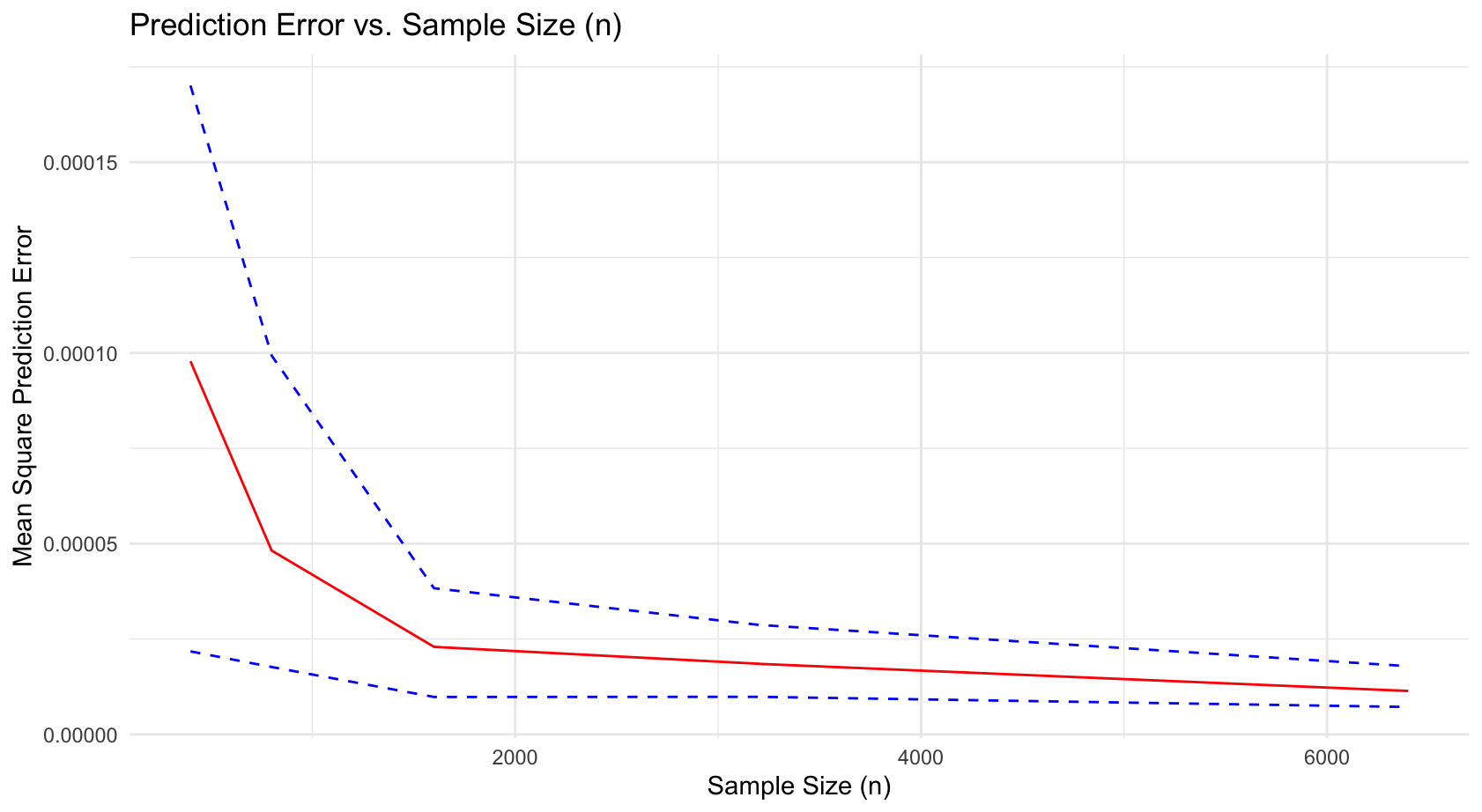}
        \caption{Average prediction error vs. sample size $n$}
        \label{fig:sub-5}
    \end{subfigure}
    \hfill
    \begin{subfigure}[b]{0.45\textwidth}
        \includegraphics[width=\textwidth]{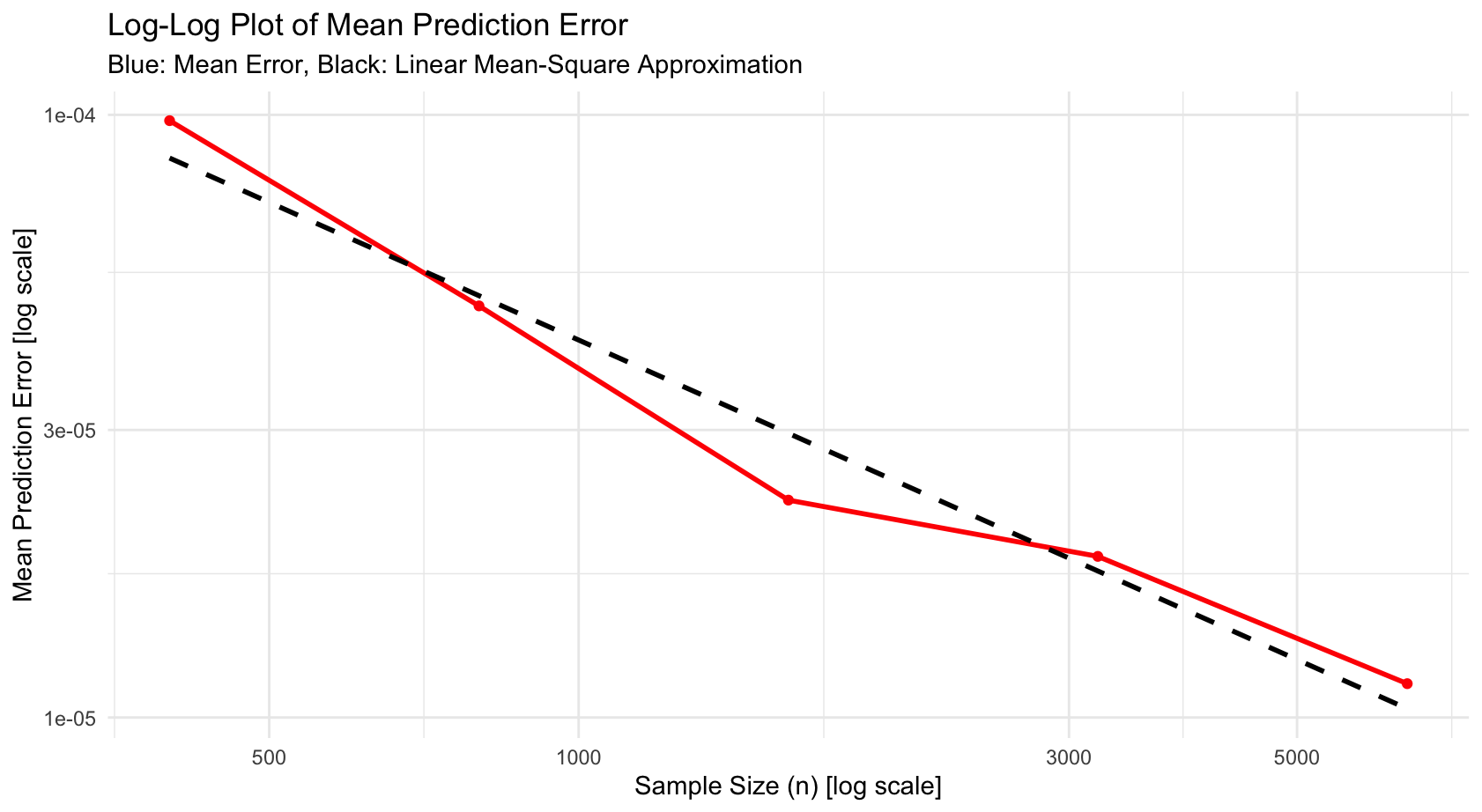}
        \caption{Log-log plot of average prediction error vs. sample size $n$}
        \label{fig:sub-6}
    \end{subfigure}

    \vspace{0.5cm} 

    \begin{subfigure}[b]{0.45\textwidth}
        \includegraphics[width=\textwidth]{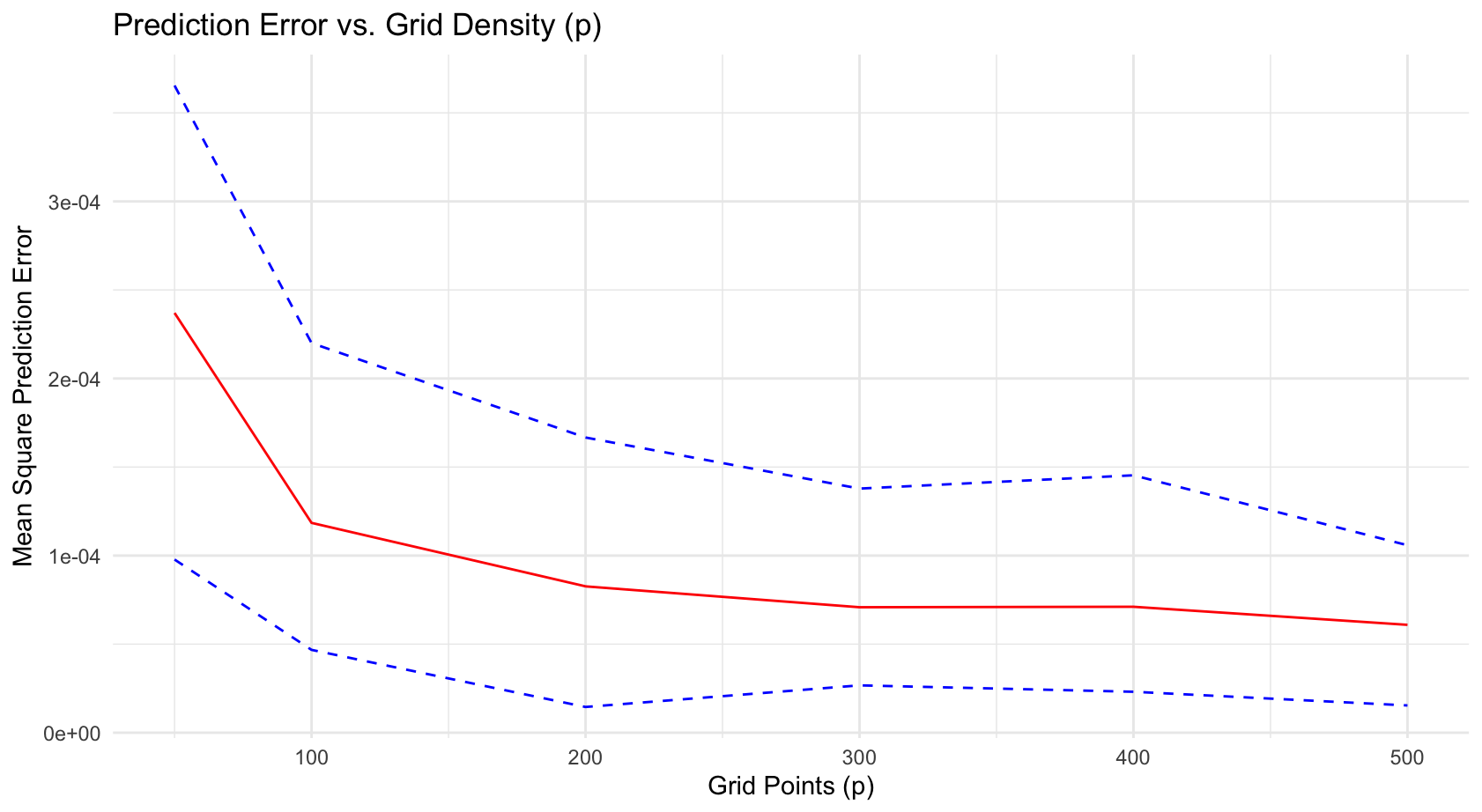}
        \caption{Average prediction error vs. number of observation points $p$}
        \label{fig:sub-7}
    \end{subfigure}
    \hfill
    \begin{subfigure}[b]{0.45\textwidth}
        \includegraphics[width=\textwidth]{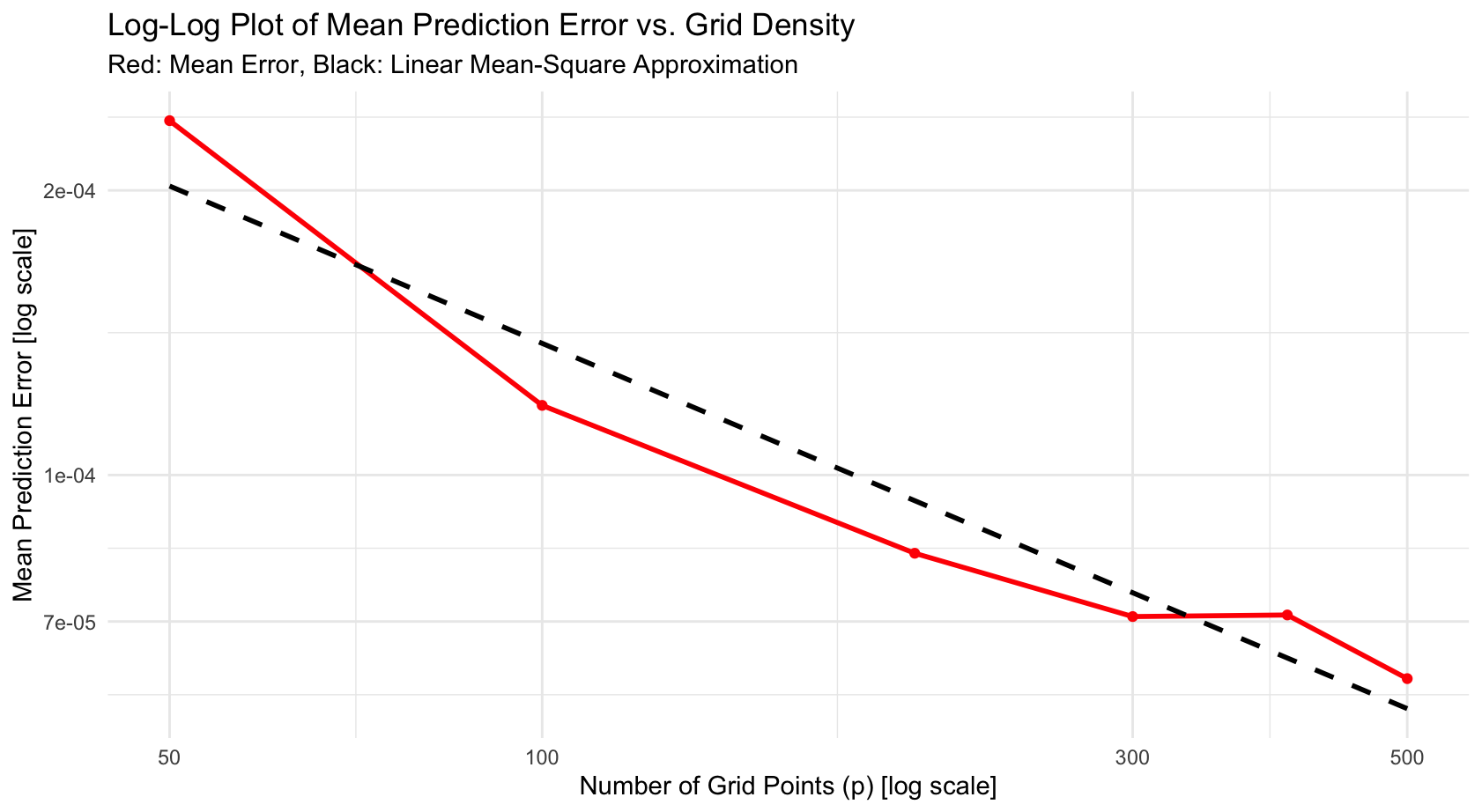}
        \caption{Log-log Plot of average prediction error vs. number of observation points $p$}
        \label{fig:sub-8}
    \end{subfigure}

    \caption{Monte Carlo approximation of the mean squared $\Gamma$-norm error as a function of the sample size $n$ and the grid density $p$. The solid red curves represent the empirical error averaged over the Monte Carlo replications. In panels a. and c. the blue dashed curves represent the 10th and 90th empirical deciles of the Monte Carlo errors. Panels b. and d. display the corresponding mean errors on a log-log scale; in these panels, the black dashed lines represent linear fits used to assess the empirical rate of decay of the error.}
    \label{fig:3}
\end{figure}
On Figure \ref{fig:3} we observe that the resulting log-log plot for the case $p$ fixed - varying $n$ yields a slope of $-0.76$, which is close to the value that we can expect given the convergence rate established in Theorem \ref{thm:4.1}. Indeed, here we have $k=1$ and $a=1$ therefore we expect to have an error which decreases in $n^{-(2k+2a)/(2a+2k+1)}=n^{-0.8}$. For the case $n$ fixed - varying $p$ the log-log plot gives an estimated slope of approximately $-0.55$, which indicates that the empirical decay of the prediction error is quite consistent with the expected $p^{-(2a-1)/(2a)}=p^{-1/2}$ behavior.

\subsubsection{Study of \texorpdfstring{$\beta_2$}{beta 2}}
In this section, we consider the true slope function $\beta_2$. Because $\beta_2$ does not perfectly satisfy the periodic boundary condition $\beta_2(0) = \beta_2(1)$, this setup allows us to evaluate how our estimator perform even if the true slope function does not perfectly lie in the space $W^{\text{per}}$. In this part we pick $a=2$ and aim to compare our result with the first study of $\beta_1$. The same parameters are used for the simulations.
\begin{figure}[htbp]
    \captionsetup[subfigure]{labelformat=simple, labelsep=period}
    \centering
    \begin{subfigure}[b]{0.45\textwidth}
        \includegraphics[width=\textwidth]{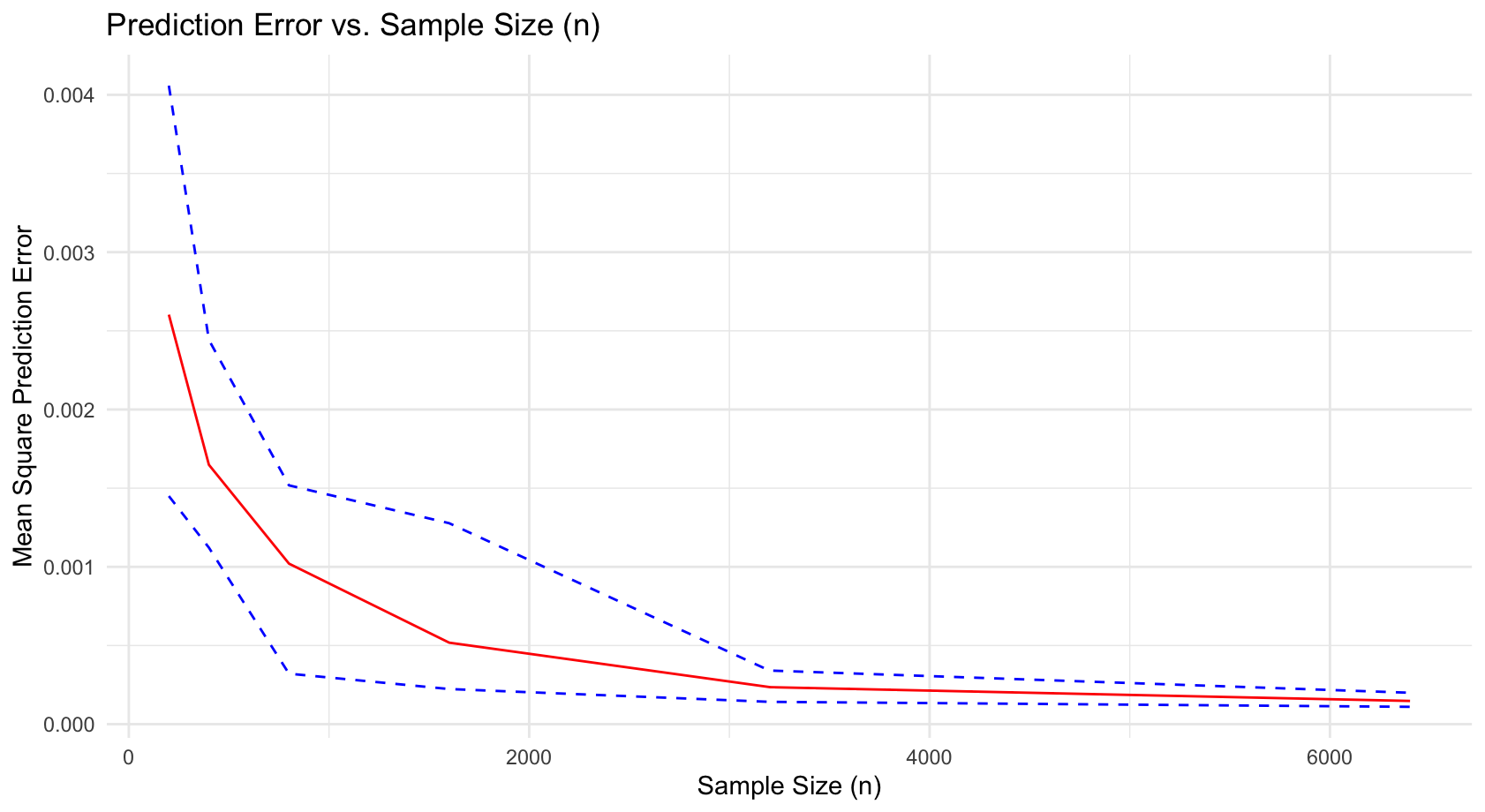}
        \caption{Average prediction error vs. sample size $n$}
        \label{fig:sub-9}
    \end{subfigure}
    \hfill
    \begin{subfigure}[b]{0.45\textwidth}
        \includegraphics[width=\textwidth]{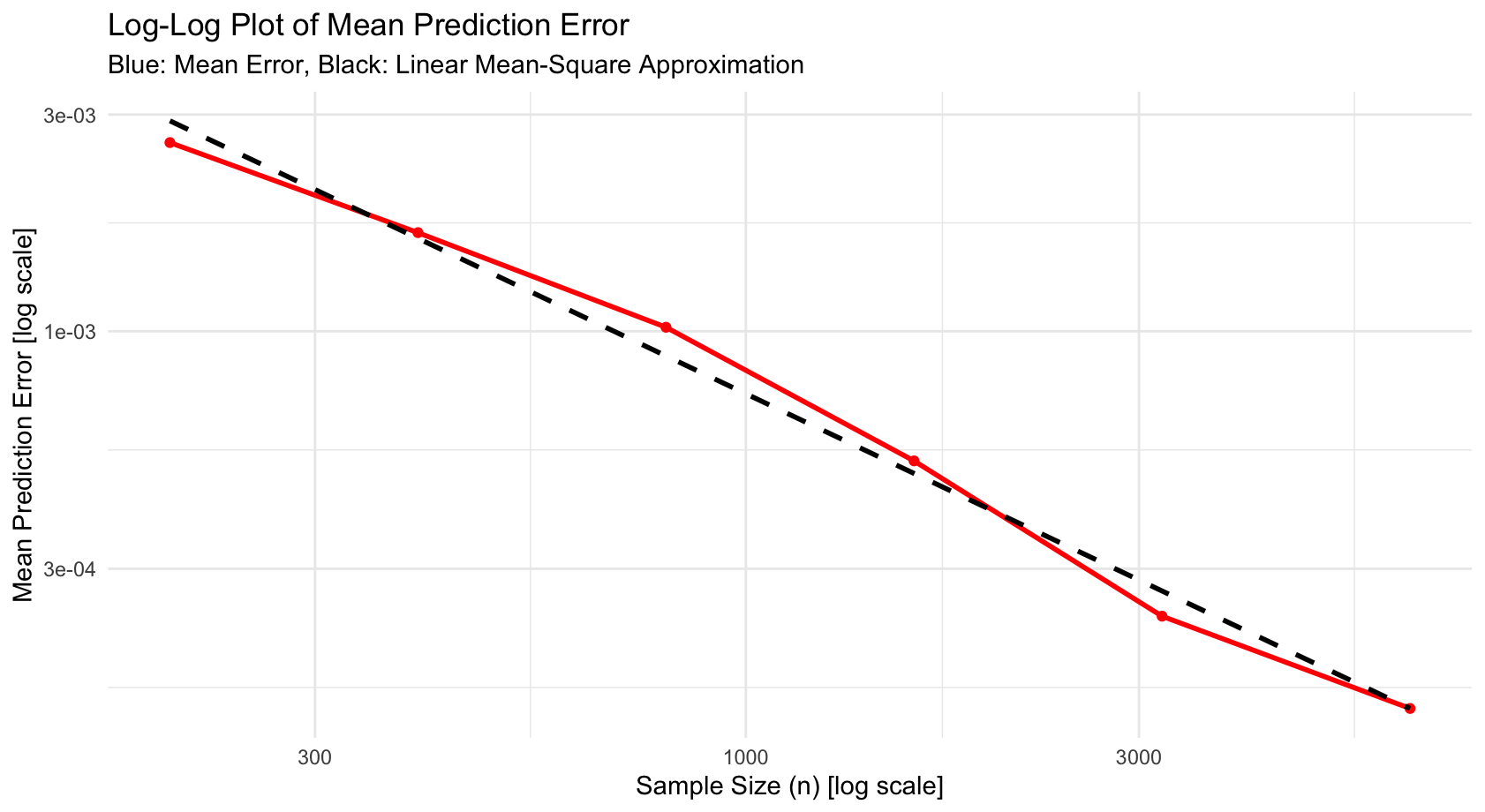}
        \caption{Log-log plot of average prediction error vs. sample size $n$}
        \label{fig:sub-10}
    \end{subfigure}

    \vspace{1cm} 

    \begin{subfigure}[b]{0.45\textwidth}
        \includegraphics[width=\textwidth]{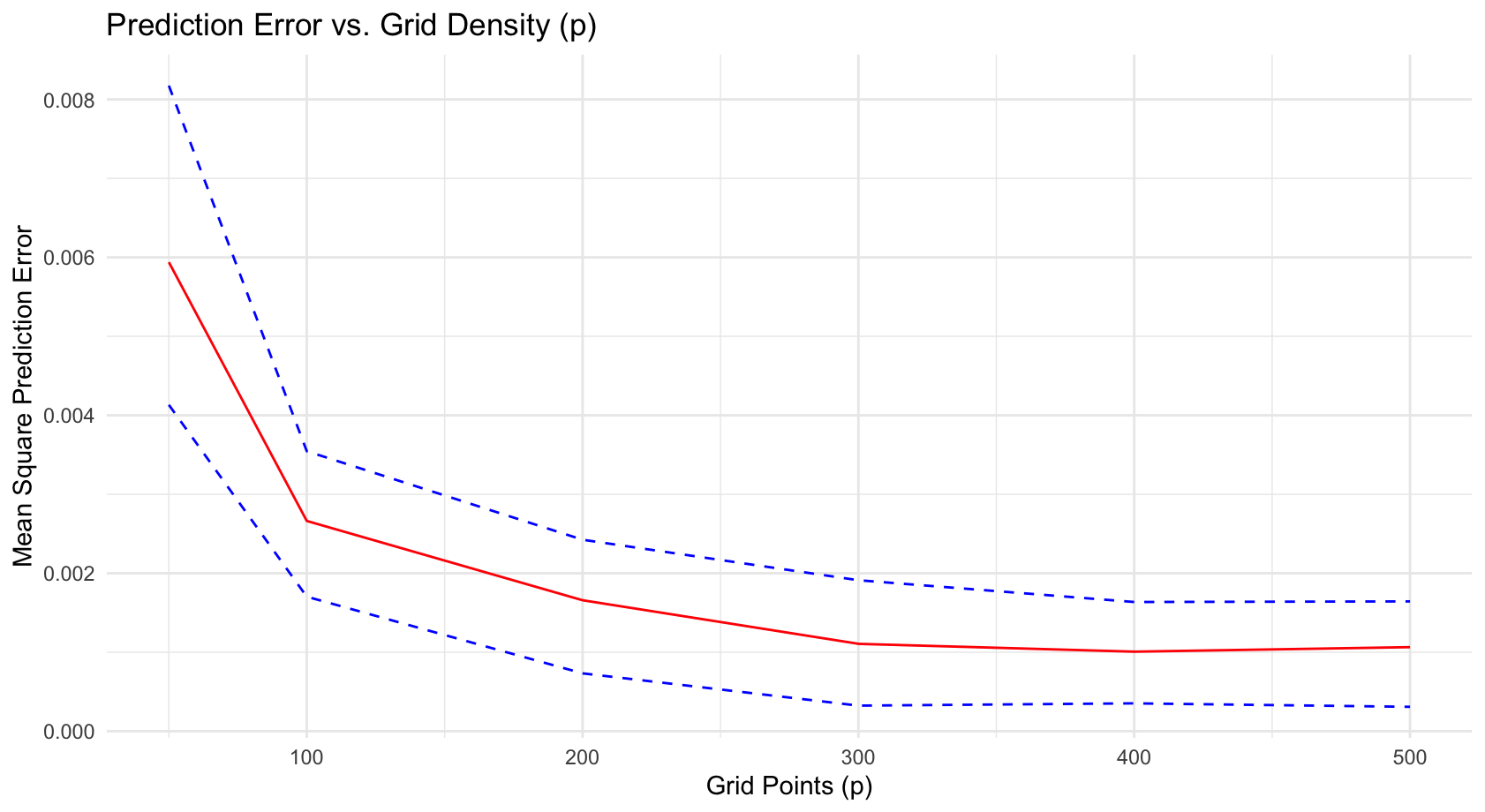}
        \caption{Average prediction error vs. number of observation points $p$}
        \label{fig:sub-11}
    \end{subfigure}
    \hfill
    \begin{subfigure}[b]{0.45\textwidth}
        \includegraphics[width=\textwidth]{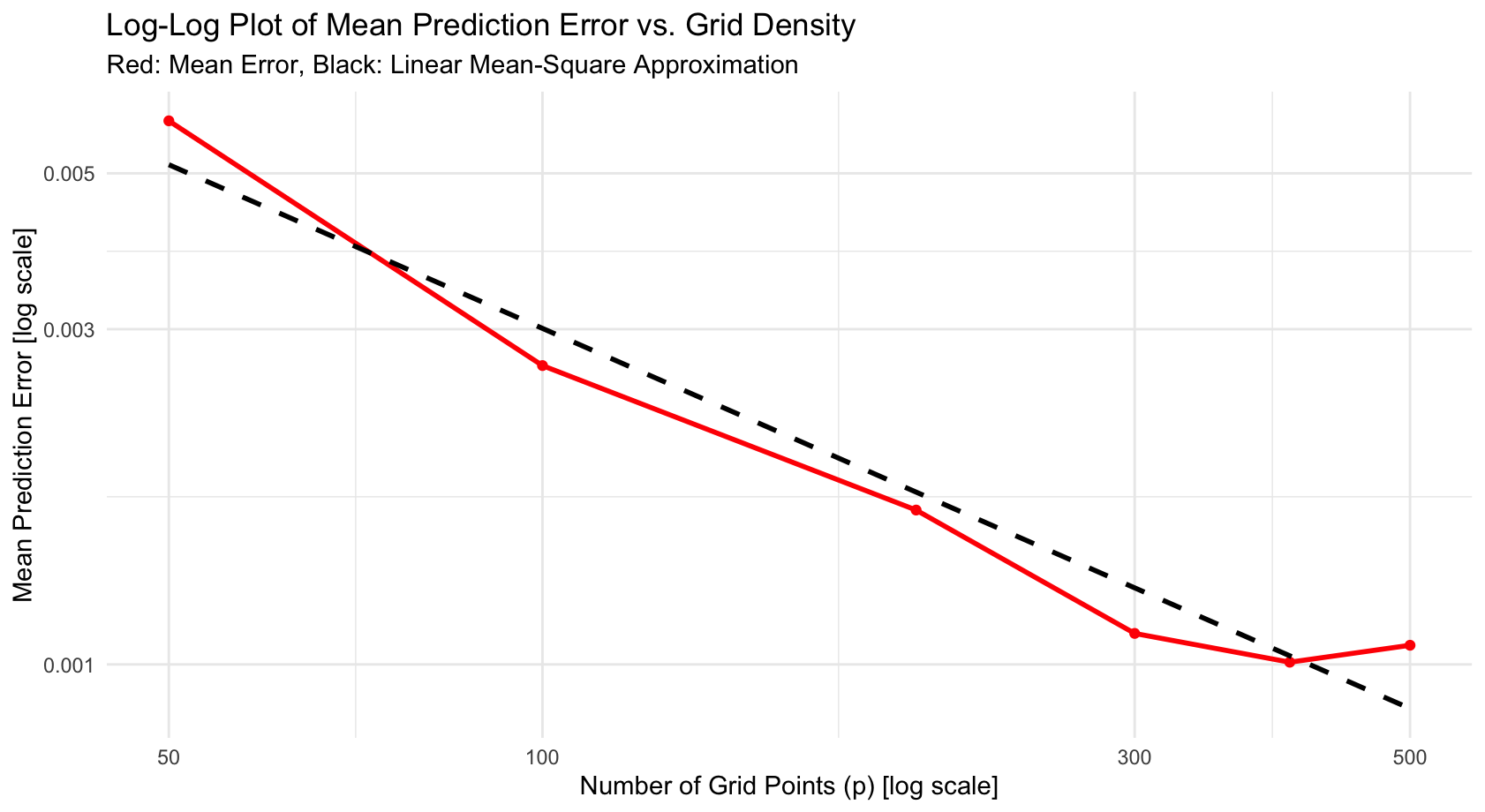}
        \caption{Log-log plot of average prediction error vs. number of observation points $p$}
        \label{fig:sub-12}
    \end{subfigure}

    \caption{Monte Carlo approximation of the mean squared $\Gamma$-norm error as a function of the sample size $n$ and the grid density $p$. The solid red curves represent the empirical error averaged over the Monte Carlo replications. In panels a. and c. the blue dashed curves represent the 10th and 90th empirical deciles of the Monte Carlo errors. Panels b. and d. display the corresponding mean errors on a log-log scale; in these panels, the black dashed lines represent linear fits used to assess the empirical rate of decay of the error.}
    \label{fig:beta_2}
\end{figure}
The log-log plot in Figures \ref{fig:sub-10} and \ref{fig:sub-12} yield respectively slopes of $-0.86$ and $-0.77$, which are very close to the one that we got previously and which match our theoretical results. 

\subsubsection{Study of \texorpdfstring{$\beta_3$}{beta 3}}
In this section, we evaluate the estimator's ability to handle change points by examining its behavior around the two discontinuities located at $t_1=0.30$ and $t_2=0.72$ of the true slope function $\beta_3$ defined in Equation \eqref{eq:beta_3}. Instead of using a threshold-based detection method, we adopt a visual approach to better understand the local behavior of our estimator. We plot the true slope function against $10$ estimated curves obtained from independent samples. To illustrate the asymptotic behavior of the estimator, we compare two distinct settings: a moderate sample size ($n=800$, $p=5000$, $D_{N_{n,p}}=15$) and a large sample size ($n=6400$, $p=5000$, $D_{N_{n,p}}=81$).
\begin{figure}[htbp]
    \captionsetup[subfigure]{labelformat=simple, labelsep=period}
    \centering
    
    \begin{subfigure}{0.8\textwidth}
        \centering
        \includegraphics[width=\textwidth]{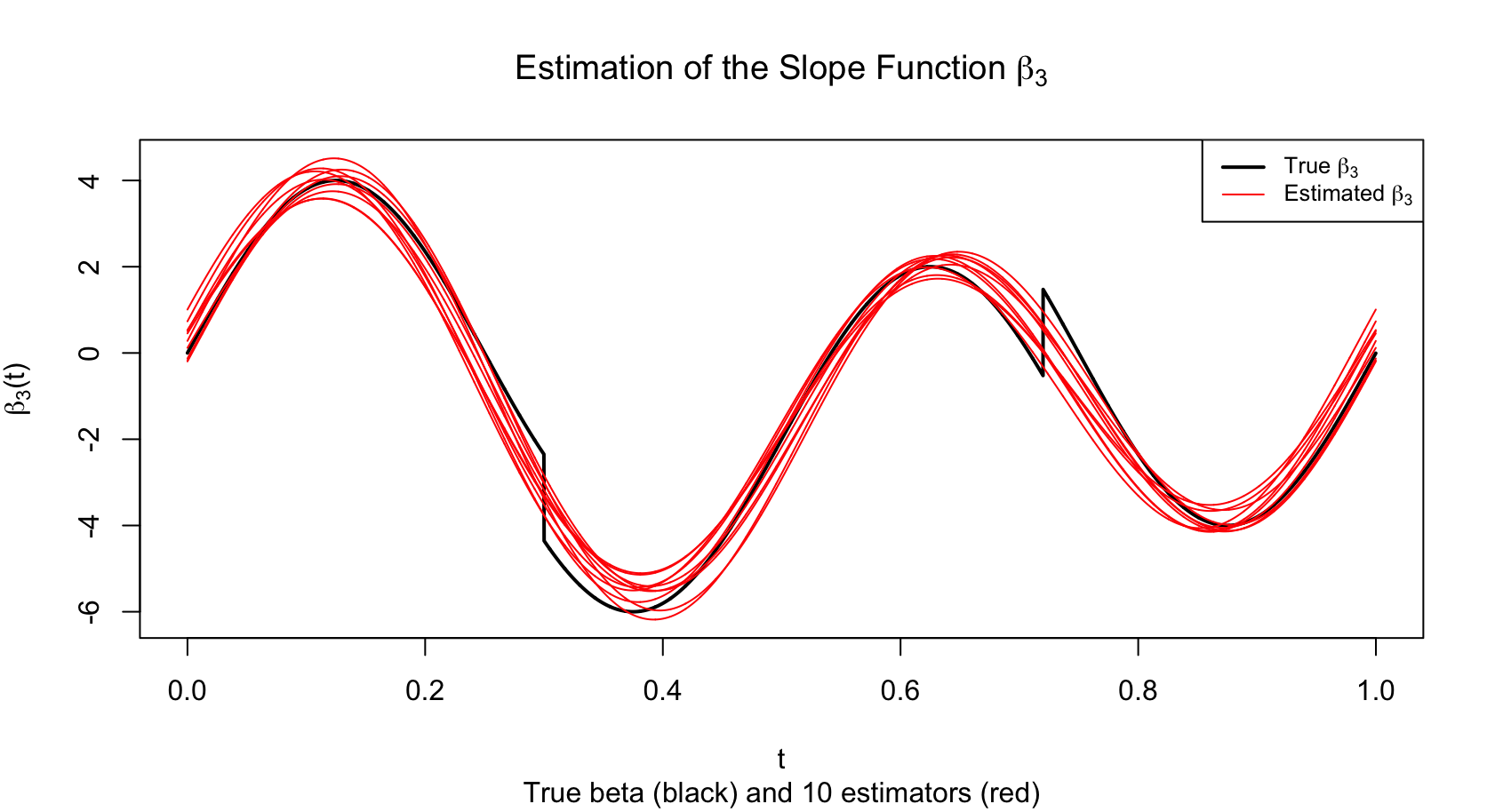}
        \caption{$n=800$, $p=5000$, and $D_{N_{n,p}}=15$.}
        \label{fig:beta3_800}
    \end{subfigure}

    \begin{subfigure}{0.8\textwidth}
        \centering
        \includegraphics[width=\textwidth]{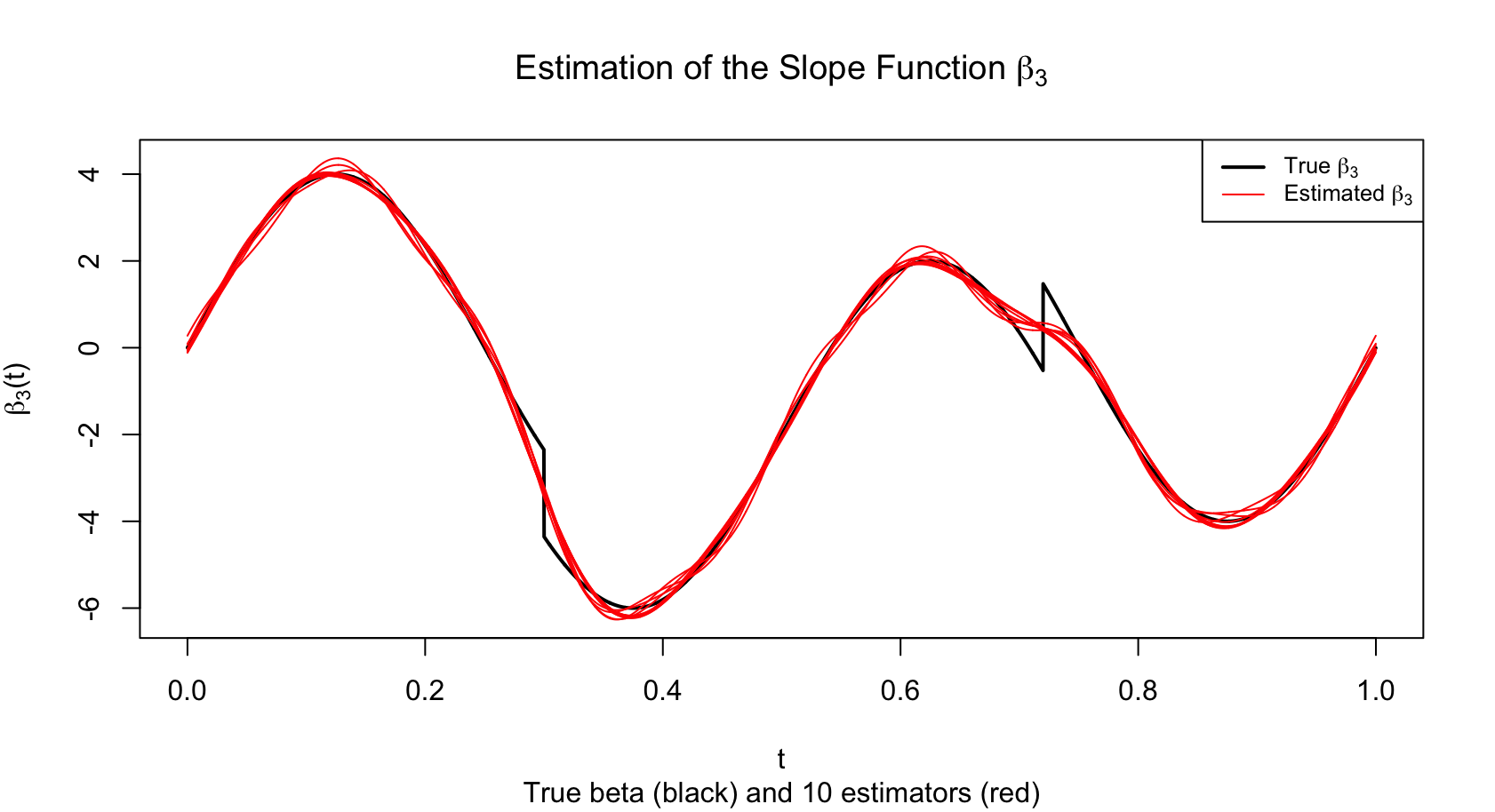}
        \caption{$n=6400$, $p=5000$, and $D_{N_{n,p}}=81$.}
        \label{fig:beta3_6400}
    \end{subfigure}
    
    \caption{Estimation of the slope function $\beta_3$ for $p=5000$ and different sample sizes $n$.}
    \label{fig:beta3_comparison}
\end{figure}
Figure \ref{fig:beta3_800} displays the results for the setting with $n=800$. The estimated curves (in red) capture the global trend of $\beta_3$ but completely smooth over the discontinuities. On the other hand, Figure \ref{fig:beta3_6400} illustrates the results for $n=6400$. With a larger sample size and a higher dimension for the basis ($D_{N_{n,p}}=81$), the variance of the estimators is significantly reduced, yielding a good fit on the continuous segments of $\beta_3$. More importantly, around the jump location $t_2$, the estimators actively attempt to fit the discontinuities as it benefits from an increased flexibility provided by $D_{N_{n,p}}=81$. 

\subsection{Effect of the unknown variance \texorpdfstring{$\sigma^{2}$}{sigma 2}}
Throughout this work we have assumed that $\sigma^{2}$, the variance of $\epsilon$ defined in Equation \eqref{eq:Y1}, is known. In practice however, this is generally not the case and $\sigma$ must therefore be estimated. To this end, for each model dimension $m$, we consider the empirical variance estimator \begin{equation}\label{eq:estimate_sigma}
    \widehat{\sigma}_{m}^{2}:=\frac{1}{n}\sum_{i=1}^{n}\left(
Y_i - \left\langle \widehat{\beta}_{m}, \widetilde{X}_i \right\rangle_{L^{2}}
\right)^{2}.
\end{equation}
We then select the dimension $\widehat{m}_{\mathrm{plug}}$ by minimizing the following penalized criterion:
\begin{equation}\label{eq:penalized_criterion}
    \widehat{m}_{\mathrm{plug}}\in\argmin_{m\in\mathcal{M}_{n,p}}\left(\gamma_n(\widehat{\beta}_m)+\theta(1+\delta)\frac{D_m\widehat{\sigma}_{m}^{2}}{n}\right).
\end{equation}
\begin{table*}[htbp]
\caption{Approximate mean squared $\Gamma$-norm prediction errors for the slope function $\beta_1$ with variance $\sigma^{2}$ known or unknown. The upper panel reports errors as $n$ varies with $p=7500$ fixed, while the lower panel reports errors as $p$ varies with $n=1000$ fixed.}
\label{tab:beta_prediction_errors_variance}
\begin{tabular}{@{}lccccccc@{}}
\hline
\multicolumn{8}{c}{Fixed $p=7500$} \\
\hline
Slope & $\sigma^{2}$ & $n=200$ & $n=400$ & $n=800$ & $n=1600$ & $n=3200$ & $n=6400$ \\
\hline
$\beta_1$ & known   & $1.42 \times 10^{-4}$ & $6.90\times 10^{-5}$ & $3.52\times 10^{-5}$ & $1.73\times 10^{-5}$ & $1.09\times 10^{-5}$ & $6.55\times 10^{-6}$ \\
$\beta_1$ & unknown & $1.27 \times 10^{-4}$ & $8.56 \times 10^{-5}$ & $3.73\times 10^{-5}$ & $1.83 \times 10^{-5}$ & $8.68\times 10^{-6}$ & $6.11\times 10^{-6}$ \\[6pt]
\hline
\multicolumn{8}{c}{Fixed $n=1000$} \\
\hline
Slope & $\sigma^{2}$ & $p=50$ & $p=100$ & $p=200$ & $p=300$ & $p=400$ & $p=500$ \\
\hline
$\beta_1$ & known   & $2.24\times 10^{-4}$ & $1.06\times 10^{-4}$ & $6.65\times 10^{-5}$ & $5.09\times 10^{-5}$ & $3.92\times 10^{-5}$ & $3.94\times 10^{-5}$ \\
$\beta_1$ & unknown & $2.18 \times 10^{-4}$ & $9.32\times 10^{-5}$ & $4.96\times 10^{-5}$ & $4.72\times 10^{-5}$ & $3.87\times 10^{-5}$ & $3.44\times 10^{-5}$ \\
\hline
\end{tabular}
\end{table*}

\begin{table*}[htbp]
\caption{Approximate mean squared $\Gamma$-norm prediction errors for the irregular slope function $\beta_2$ with variance $\sigma^{2}$ known or unknown. The upper panel reports errors as $n$ varies with $p=7500$ fixed, while the lower panel reports errors as $p$ varies with $n=1000$ fixed.}
\label{tab:beta2_prediction_errors_variance}
\begin{tabular}{@{}lccccccc@{}}
\hline
\multicolumn{8}{c}{Fixed $p=7500$} \\
\hline
Slope & $\sigma^{2}$ & $n=200$ & $n=400$ & $n=800$ & $n=1600$ & $n=3200$ & $n=6400$ \\
\hline
$\beta_2$ & known   & $2.60\times 10^{-3}$ & $1.64\times 10^{-3}$ & $1.02\times 10^{-3}$ & $5.18\times 10^{-4}$ & $2.36\times 10^{-4}$ & $1.48\times 10^{-4}$ \\
$\beta_2$ & unknown & $3.31\times 10^{-3}$ & $1.84\times 10^{-3}$ & $1.07\times 10^{-3}$ & $4.56\times 10^{-4}$ & $2.65\times 10^{-4}$ & $1.55\times 10^{-4}$ \\[6pt]
\hline
\multicolumn{8}{c}{Fixed $n=1000$} \\
\hline
Slope & $\sigma^{2}$ & $p=50$ & $p=100$ & $p=200$ & $p=300$ & $p=400$ & $p=500$ \\
\hline
$\beta_2$ & known   & $5.93\times 10^{-3}$ & $2.66\times 10^{-3}$ & $1.66\times 10^{-3}$ & $1.11\times 10^{-3}$ & $1.01\times 10^{-3}$ & $1.06\times 10^{-3}$ \\
$\beta_2$ & unknown & $5.63\times 10^{-3}$ & $2.81\times 10^{-3}$ & $1.51\times 10^{-3}$ & $1.22\times 10^{-3}$ & $1.07\times 10^{-3}$ & $9.19\times 10^{-4}$ \\
\hline
\end{tabular}
\end{table*}
Tables \ref{tab:beta_prediction_errors_variance} and \ref{tab:beta2_prediction_errors_variance} compare the mean squared $\Gamma$-norm prediction errors obtained when $\sigma^2$ is known and when it is estimated using \eqref{eq:estimate_sigma}. The comparison is performed for $\beta_1$ \eqref{eq:beta_1} and $\beta_2$ \eqref{eq:beta_2}. As in the previous simulations, we consider two asymptotic regimes: first, $n$ varies while $p=7500$ is fixed; second, $p$ varies while $n=1000$ is fixed. The results show that replacing $\sigma^2$ by the residual-based estimator $\widehat{\sigma}_m^2$ has little impact on the performance of the estimator. Indeed, in the two tables the prediction errors obtained with known and unknown variance remain of the same order of magnitude and exhibit similar decreasing behavior as $n$ or $p$ increases. This suggests that the procedure remains reliable in the more realistic setting where the regression noise variance is unknown.

\section{Application to meteorological data}
We apply our estimation procedure to meteorological data provided by Météo-France, originally covering the period 1950--2024. For this application, we retain observations from the recent period 2019--2024 and average them over time to avoid making the estimated curve depend on the particular weather realization of a single year. After removing observations with missing values and excluding February 29 to obtain a common yearly calendar, the final dataset consists of $n = 305$ temperature curves, each corresponding to one location, and $p = 365$ daily measurements. The annual average precipitation recorded at each site is used as the scalar response variable, while the corresponding annual temperature profile is treated as a functional covariate. This yields a scalar-on-function regression framework, whose objective is to investigate how the shape of the temperature curve is associated with annual rainfall across the selected French departments highlighted in Figure \ref{fig:departments}. Figures \ref{fig:curves} and \ref{fig:curves_station} show examples of the functional data for four stations located in Finistère.
\begin{figure}[htbp]
    \captionsetup[subfigure]{labelformat=simple, labelsep=period}
    \centering
    \begin{subfigure}[b]{0.8\textwidth}
        \centering
        \includegraphics[width=\textwidth]{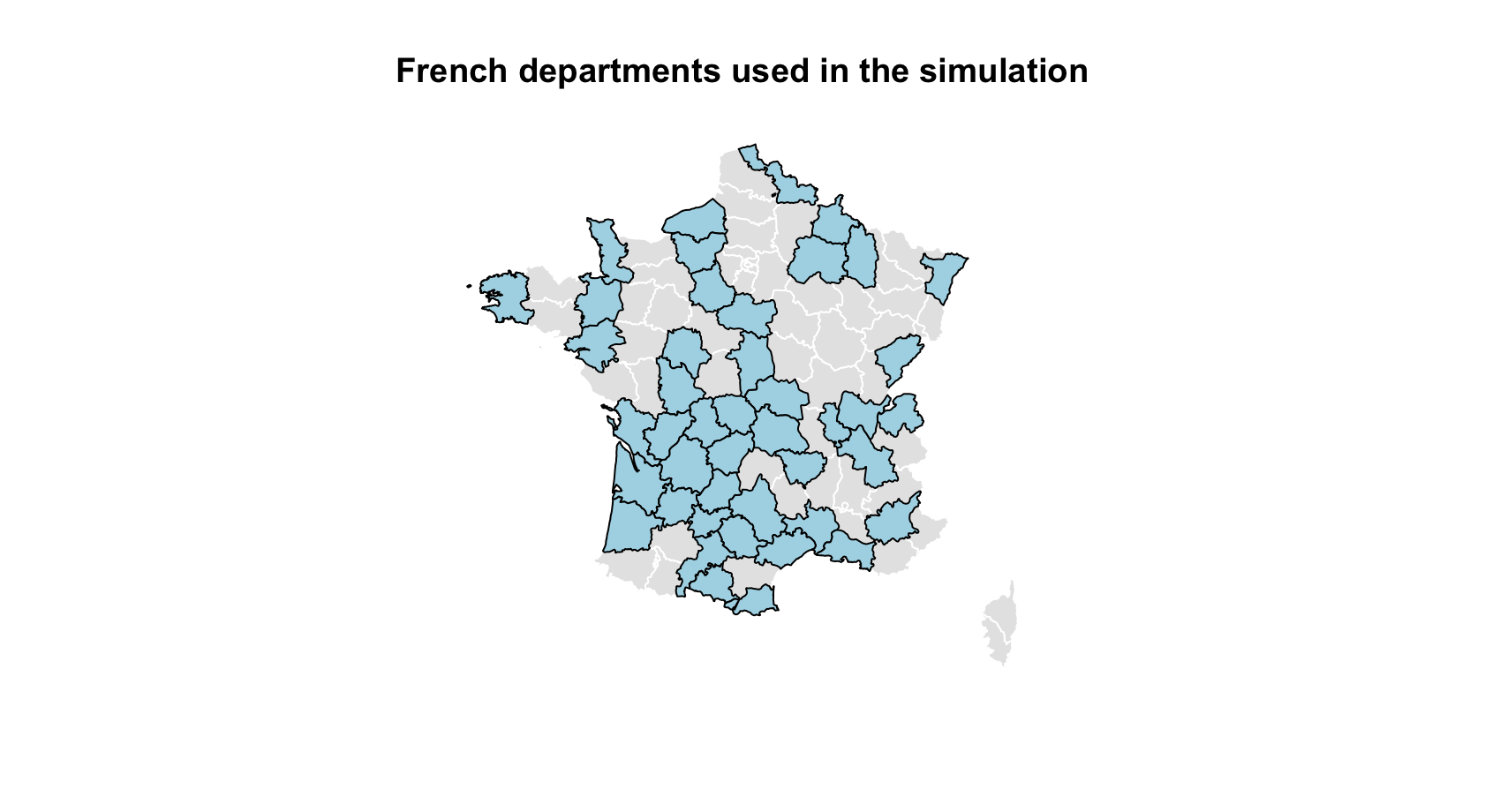}
        \caption{French departments included in the meteorological dataset.}
        \label{fig:departments}
    \end{subfigure}
    
    \begin{subfigure}[b]{0.7\textwidth}
        \centering
        \includegraphics[width=\textwidth]{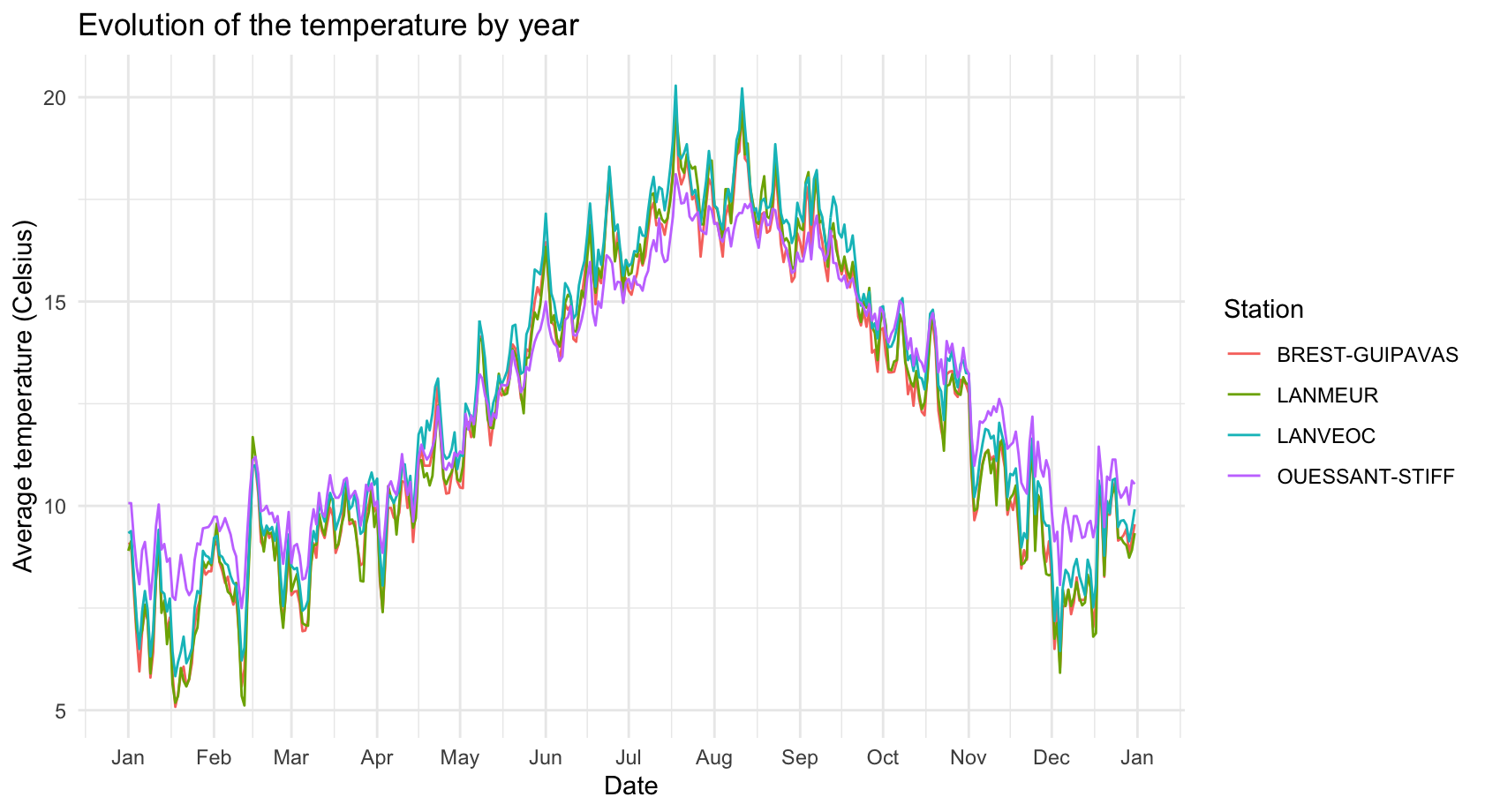}
        \caption{Averaged annual temperature curves across four stations in Finistère.}
        \label{fig:curves}
    \end{subfigure}

    \begin{subfigure}[b]{0.7\textwidth}
        \centering
        \includegraphics[width=\textwidth]{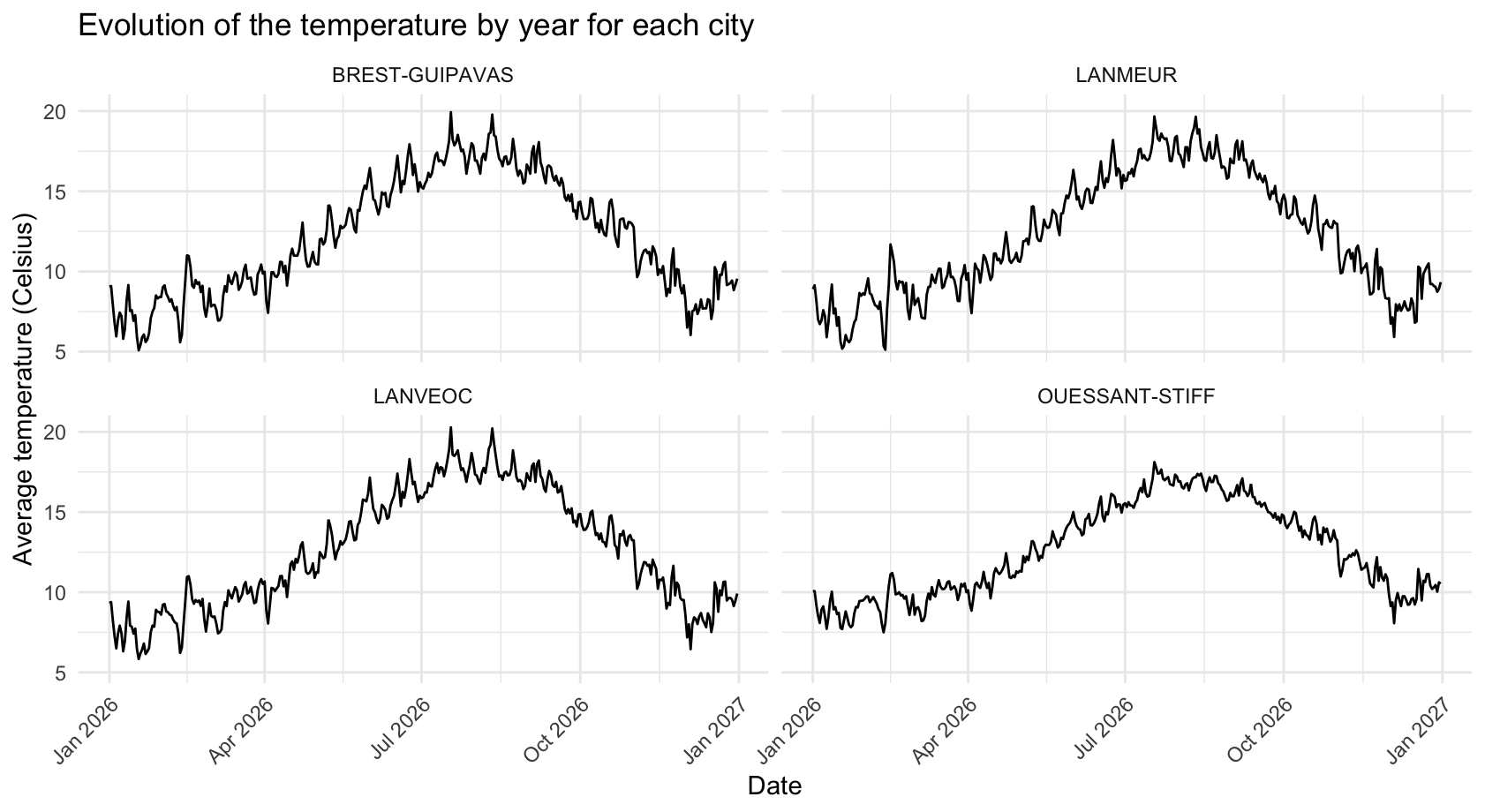}
        \caption{Averaged annual temperature curves for the selected stations in Finistère.}
        \label{fig:curves_station}
    \end{subfigure}
    \label{fig:combined_vertical}
    
    \caption{Meteorological dataset map and the associated temperature curves in Finistère.}
\end{figure}
The initial step involves partitioning the dataset into a training set ($80\%$ of observations) and a test set ($20\%$ of observations) to ensure an unbiased evaluation of the final model. We denote as $\mathrm{trainIDX}$ and $\mathrm{testIDX}$ the indexes of the data present in respectively the train and test set, as well as $n_{\mathrm{train}}$ and $n_{\mathrm{test}}$ their dimensions. Both sets are subsequently centered to standardize the data. More particularly, the empirical means are computed exclusively from the training set. For the scalar response, we define
\begin{equation*}
\overline{Y}_{\mathrm{train}}=\frac{1}{n_{\mathrm{train}}}\sum_{i \in \mathrm{trainIDX}} Y_i,
\end{equation*}
while, for each time point $t_j$, the mean temperature curve is given by
\begin{equation*}
\overline{Z}_{\mathrm{train}}(t_j)=\frac{1}{n_{\mathrm{train}}}\sum_{i \in \mathrm{trainIDX}} Z_i(t_j).
\end{equation*}
Both the training and test observations are then centered using these training-set means. Thus, the centered training data are defined as
\begin{equation*}
\widetilde{Y}_{\mathrm{train}}=Y_{\mathrm{train}}-\overline{Y}_{\mathrm{train}},
\end{equation*}
and for all $j=0,\ldots,p-1$
\begin{equation*}
\widetilde{Z}_{\mathrm{train}}(t_j)=Z_{\mathrm{train}}(t_j)-\overline{Z}_{\mathrm{train}}(t_j).
\end{equation*}
Similarly, the test data are centered using the same means computed from the training set:
\begin{equation*}
\widetilde{Y}_{\mathrm{test}}=Y_{\mathrm{test}}-\overline{Y}_{\mathrm{train}},
\end{equation*}
and for all $j=0,\ldots,p-1$
\begin{equation*}
\widetilde{Z}_{\mathrm{test}}(t_j)=Z_{\mathrm{test}}(t_j)-\overline{Z}_{\mathrm{train}}(t_j).
\end{equation*}
We then reconstruct the curves with the same methodology as in Section \ref{sec:curve_reconstruct}. Afterward we estimate the variance $\sigma^{2}$ using only the elements of training set $\widetilde{Z}_{\mathrm{train}}$ and $\widetilde{Y}_{\mathrm{train}}$, by using the estimator defined in Equation (\ref{eq:estimate_sigma}). The functional linear regression model is subsequently fitted on the training set, using the estimator of $\sigma^{2}$ to select the dimension $\widehat{m}_{\mathrm{plug}}$ (see \eqref{eq:penalized_criterion}). After this step we obtain an estimator of the slope function of the form $$\widetilde{\beta}(t)=\sum_{j=1}^{\widehat{m}_{\mathrm{plug}}}\widetilde{\alpha}_{j}\phi_j(t),$$and its predictive performance is evaluated on the held-out test set by calculating the root mean squared error (RMSE). Since the true curves are not available, we use the reconstructed ones to make prediction. More particularly, in order to do so, we reconstruct the curves from the noisy and discrete data of the test set ($\widetilde{Z}_{\mathrm{test}}$). Thus we obtain
\[\widetilde{X}_{\mathrm{test}_i}(t)=\sum_{j=1}^{D_{N_{n,p}}}\widetilde{x}_{i,j}\phi_{j}(t), \qquad \text{for all }i\in\mathrm{testIDX}\text{ and for all }t\in[0,1], \]and then, for each observation $i$ in the test set we predict 
\begin{equation*}
\widetilde{Y}_{\mathrm{pred}_i}=\sum_{j=1}^{\widehat{m}_{\mathrm{plug}}}\widetilde{x}_{i,j}\widetilde{\alpha}_j.
\end{equation*}
\begin{figure}[htbp]
    \centering
    \includegraphics[width=0.7\textwidth]{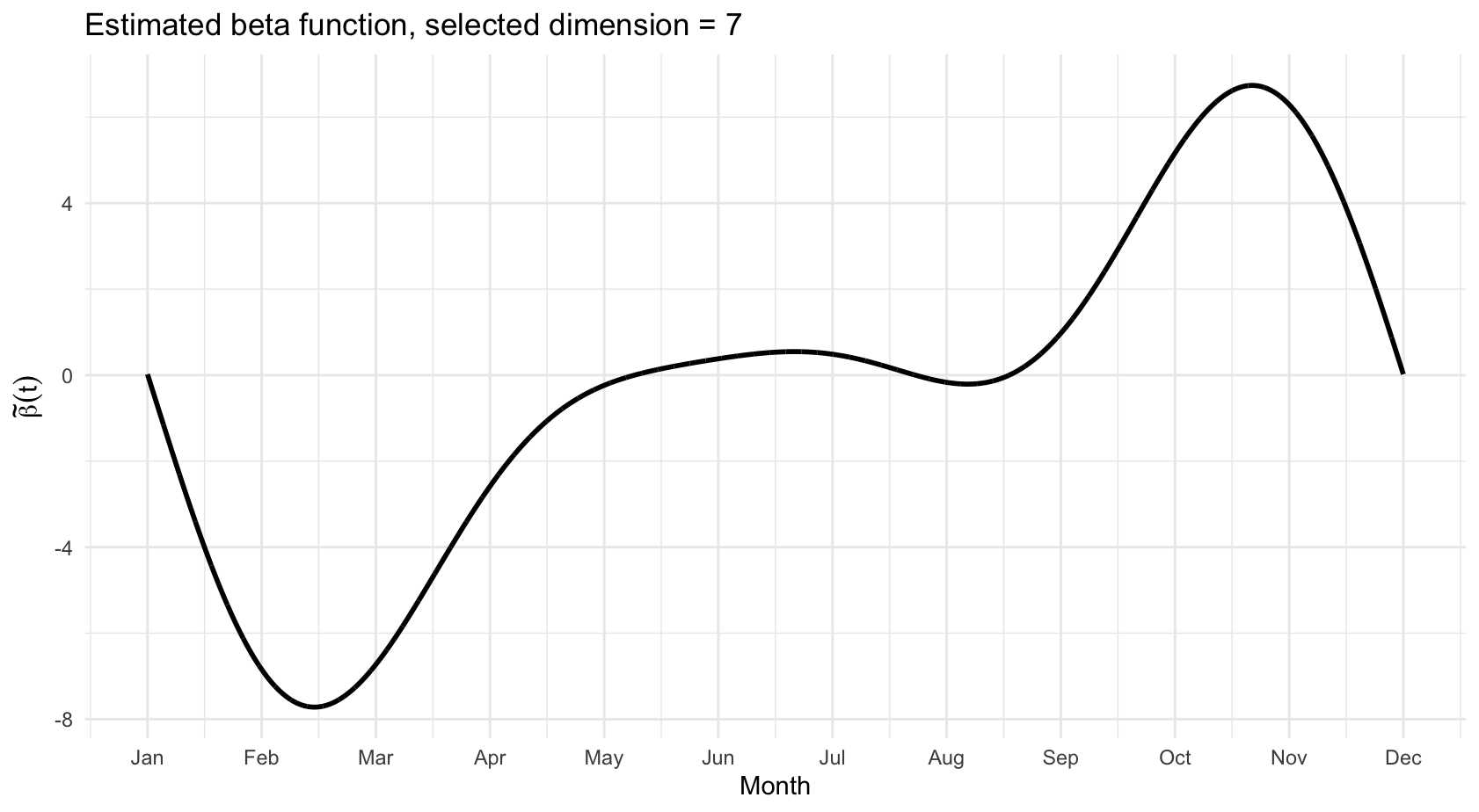}
    \caption{Estimated functional coefficient $\widetilde{\beta}(t)$, describing the effect of the annual temperature profile on average annual precipitation.}
    \label{fig:beta_true}
\end{figure}
In Figure \ref{fig:beta_true} we can see that the estimated slope function exhibits two extrema. One is negative during winter and early spring, with its lowest values around February–March, which indicates that higher temperatures during this period are associated with lower annual precipitation. The coefficient then plateau around 0 during the summer. Thus at this time the temperature curve does not impact the annual precipitation. The estimated slope function then increase quickly, reaching its maximum around October-November. This indicates a positive association between autumnal temperatures and annual precipitation. It then decreases during winter.
\begin{figure}[htbp]
\captionsetup[subfigure]{labelformat=simple, labelsep=period}
\centering

\begin{subfigure}{0.48\textwidth}
    
    \centering
    \includegraphics[width=\textwidth]{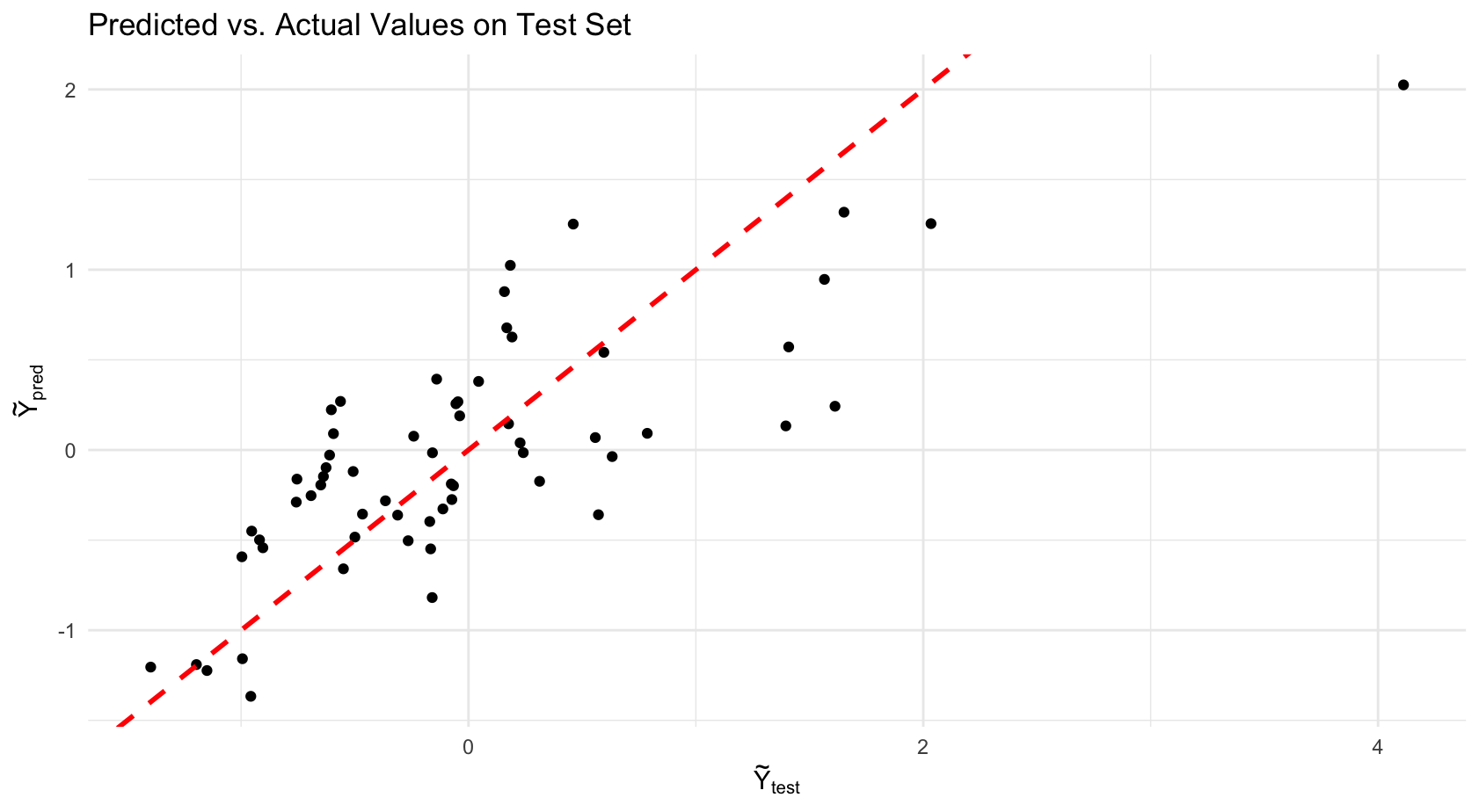}
    \caption{Plot of the actual vs predicted values, the red dashed line is $y=x$.}
    \label{fig:pred_actual}
\end{subfigure}
\hfill
\begin{subfigure}{0.48\textwidth}
    \centering
    \includegraphics[width=\textwidth]{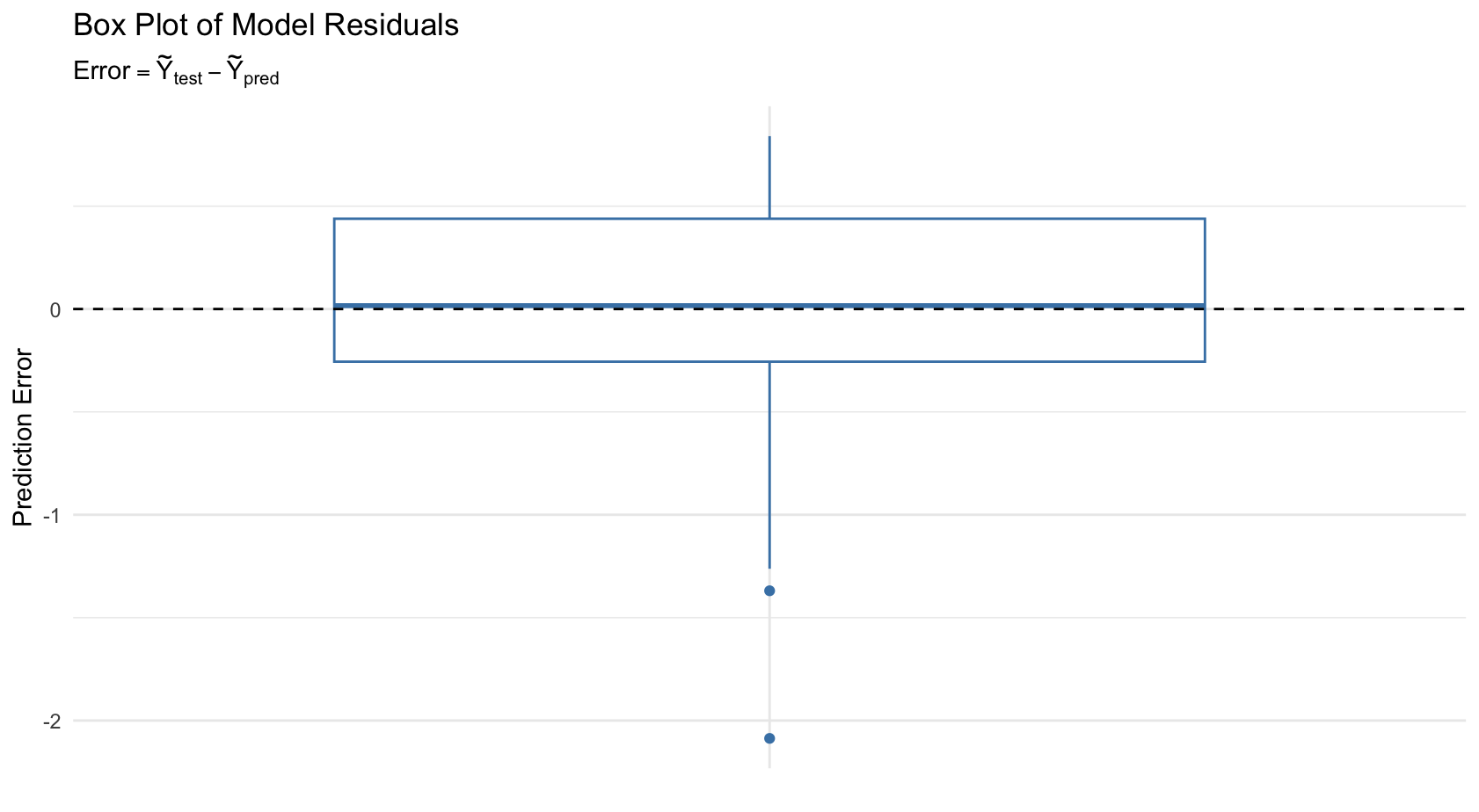}
    \caption{Boxplot of the model's residuals, showing the distribution of the prediction errors on the test set.}
    \label{fig:boxplot}
\end{subfigure}

\caption{Performance assessment of the estimator on the test set.}
\label{fig:results_app}

\end{figure}
Figure \ref{fig:pred_actual} shows a positive relationship between the observed and predicted centred values, which indicates that the model is able to reproduce the overall trend in annual precipitation. Figure \ref{fig:boxplot} confirm this. In particular, it shows that the prediction errors are mostly centred around zero therefore we don't have a systematic bias. However, the lower whisker is longer than the upper one, suggesting a slightly greater dispersion among negative residuals. Hence, some locations are more strongly overpredicted than underpredicted.

\medskip\noindent
The model achieves a RMSE of 0.58, which is quite low. However, this performance is only a small improvement over the baseline method where we use the average rainfall quantity with all the data from the training set to predict the average rainfall of the test set. This method yields a RMSE of 0.91. A primary reason for this small difference could be that the baseline method is already quite effective. Given that the target variable is the annual average rainfall, it is a stable and predictable value. The variations from year to year might be minimal, meaning that a simple prediction of the overall average is a already a strong starting point.

\begin{appendix}

\section{Gradient Calculation for the Least Squares Minimization}
\begin{lemma} \label{lem:grad}
The function $F$ which is such that
\begin{equation*}
F(v):=\frac{1}{n}\sum_{i=1}^{n}\left(Y_i - \sum_{j=1}^{D_m}v_j\langle \phi_{j},\widetilde{X}_{i}\rangle_{L^{2}}\right)^{2}, \qquad \text{for all } v \in\mathbb{R}^{D_m}  ,
\end{equation*}
is differentiable and its gradient is
\[\nabla F(v) = -2b + 2\Phi_m v,\] where $\Phi_m$ is defined by Equation \eqref{eq:Phi_matrix} and $b$ by Equation \eqref{eq:b_vector}.
\end{lemma}
\begin{proof}
    The function $F$ being convex and $\mathcal{C}^{1}$, finding a minimizer $\widetilde{\alpha} \in \mathbb{R}^{D_m} $ of $F$ is equivalent to finding an $\widetilde{\alpha}$ in $\mathbb{R}^{D_m}$ such that $\nabla F(\widetilde{\alpha})=0$. Let us compute $\nabla F$.
\newline
First let us consider $k\in \{1,\ldots,D_m\}$, then 
\begin{equation*}
\begin{split}
\frac{\partial F(v)}{\partial v_k} 
&= -\frac{2}{n}\sum_{i=1}^{n}\langle\phi_k,\widetilde{X}_{i}\rangle_{L^{2}}\left(Y_i-\sum_{j=1}^{D_m}v_j\langle\phi_j,\widetilde{X}_i\rangle_{L^{2}}\right) \\
&= -\frac{2}{n}\sum_{i=1}^{n}Y_i\left(\sum_{l=1}^{D_{N_{n,p}}}\widetilde{x}_{i,l}\langle\phi_k,\phi_l\rangle_{L^{2}}\right) \\
&\quad + \frac{2}{n}\sum_{i=1}^{n}\sum_{j=1}^{D_m}v_j\left(\sum_{l=1}^{D_{N_{n,p}}}\widetilde{x}_{i,l}\langle\phi_k,\phi_l\rangle_{L^{2}}\right)\left(\sum_{l'=1}^{D_{N_{n,p}}}\widetilde{x}_{i,l'}\langle\phi_j,\phi_{l'}\rangle_{L^{2}}\right) \\
&= -\frac{2}{n}\sum_{i=1}^{n}Y_i\widetilde{x}_{i,k} + \frac{2}{n}\sum_{i=1}^{n}\widetilde{x}_{i,k}\sum_{j=1}^{D_m}v_j\widetilde{x}_{i,j}.
\end{split}
\end{equation*}
The transition from the first to second line is obtained by replacing the $\widetilde{X}_{i}$'s by their definition given in (\ref{eq:4}) and the one from the second to third by using the fact that the $\phi_j$'s are orthonormal.
Hence, 
\[\nabla F(v) = -2b + 2\Phi_m v.\]
\end{proof}
\section{Proof of Section 3}
\subsection{Proof of Proposition \ref{prop:1}} \label{sec:A}
Let us start first by decomposing $\mathbb{E}(\| \widetilde{\beta} - \beta\|_{\widetilde{\Gamma}_{n}}^{2})$ into two terms :
\begin{equation*}
    \begin{split}
\mathbb{E}(\| \widetilde{\beta} - \beta\|_{\widetilde{\Gamma}_{n}}^{2}) &= \mathbb{E}(\| \widetilde{\beta} - \beta\|_{\widetilde{\Gamma}_{n}}^{2}\mathds{1}_{\overline{G}}) + \mathbb{E}(\|  \widetilde{\beta} - \beta\|_{\widetilde{\Gamma}_{n}}^{2}\mathds{1}_{\overline{G}^{c}}) \\
 & = \mathbb{E}(\| \widehat{\beta}_{\widehat{m}} - \beta\|_{\widetilde{\Gamma}_{n}}^{2}\mathds{1}_{\overline{G}}) + \mathbb{E}(\| \beta\|_{\widetilde{\Gamma}_{n}}^{2}\mathds{1}_{\overline{G}^{c}}).
\end{split}
\end{equation*}
In a first hand we will focus on the left term $\mathbb{E}(\| \widehat{\beta}_{\widehat{m}} - \beta\|_{\widetilde{\Gamma}_{n}}^{2}\mathds{1}_{\overline{G}})$. Let $m\in\mathcal{M}_{n,p}$ be fixed. Consider $\beta^{(m)}$ the orthogonal projection of $\beta$ over $S_m$ with respect to the scalar product $\langle \cdot,\cdot\rangle_{\widetilde{\Gamma}}$. By definition of $\widehat{\beta}_{m}$, we have that $\gamma_n(\widehat{\beta}_{m})\leq \gamma_n(\beta^{(m)})$. Moreover, the definition of $\widehat{m}$ implies that
\[\gamma_n(\widehat{\beta}_{\widehat{m}})+\mathrm{pen}(\widehat{m}) \leq \gamma_n(\widehat{\beta}_{m})+\mathrm{pen}(m) \leq \gamma_n(\beta^{(m)})+\mathrm{pen}(m).\]
Therefore, replacing the $\gamma_{n}$'s by their definition given in (\ref{eq:gamma_n}) we have :
\[\frac{1}{n}\sum_{i=1}^{n}(Y_i - \langle\widehat{\beta}_{\widehat{m}},\widetilde{X}_{i}\rangle_{L^{2}})^{2} + \mathrm{pen}(\widehat{m}) \leq \frac{1}{n}\sum_{i=1}^{n}(Y_i - \langle\beta^{(m)},\widetilde{X}_{i}\rangle_{L^{2}})^{2} + \mathrm{pen}(m).\]
As for all $i=1,...,n$, $Y_i=\langle\beta,\widetilde{X}_{i}\rangle_{L^{2}} + \langle\beta,X_{i}-\widetilde{X}_{i}\rangle_{L^{2}} + \epsilon_{i}$ the previous equation is equivalent to the following one :
\[\begin{aligned}
\frac{1}{n}\sum_{i=1}^{n}\Bigl(
  &\bigl\langle \beta - \widehat\beta_{\widehat m},\widetilde X_i\bigr\rangle_{L^2}
  + \epsilon_i
  + \bigl\langle \beta, X_i - \widetilde X_i\bigr\rangle_{L^2}
\Bigr)^2
+ \mathrm{pen}(\widehat m)
\\
&\leq
\frac{1}{n}\sum_{i=1}^{n}\Bigl(
  \bigl\langle \beta - \beta^{(m)},\widetilde X_i\bigr\rangle_{L^2}
  + \epsilon_i
  + \bigl\langle \beta, X_i - \widetilde X_i\bigr\rangle_{L^2}
\Bigr)^2
+ \mathrm{pen}(m).
\end{aligned}\]
Simplifying the terms leads then to this inequality,
\[\|\beta-\widehat{\beta}_{\widehat{m}}\|_{\widetilde{\Gamma}_{n}}^{2}\leq \|\beta-\beta^{(m)}\|_{\widetilde{\Gamma}_{n}}^{2}+2\nu_{n}(\widehat{\beta}_{\widehat{m}}-\beta^{(m)})+2r_{n}(\widehat{\beta}_{\widehat{m}}-\beta^{(m)})+\text{pen}(m)-\text{pen}(\widehat{m}).\]
where, \begin{equation*}
\begin{aligned}
\nu_{n}(f)
&= \frac{1}{n}\sum_{i=1}^{n}
   \epsilon_{i}
   \bigl\langle f,\widetilde{X}_{i}\bigr\rangle_{L^{2}},\\
r_{n}(f)
&= \frac{1}{n}\sum_{i=1}^{n}
   \bigl\langle f,\widetilde{X}_{i}\bigr\rangle_{L^{2}}
   \bigl\langle \beta,X_{i}-\widetilde{X}_{i}\bigr\rangle_{L^{2}}.
\end{aligned}
\end{equation*}
Let us focus on the $\nu_{n}$ part. As $\nu_{n}$ is a linear process, we have that \[2\nu_{n}(\widehat{\beta}_{\widehat{m}}-\beta^{(m)})\leq 2\|\widehat{\beta}_{\widehat{m}}-\beta^{(m)}\|_{\widetilde{\Gamma}_{n}}\sup_{f\in S_{m\vee \widehat{m}}^{\widetilde{\Gamma}_{n}}}\nu_{n}(f),\]
with $S_{m\vee \widehat{m}}^{\widetilde{\Gamma}_{n}}=\{f\in S_{m\vee \widehat{m}}, \| f \|_{\widetilde{\Gamma}_{n}}=1\}$. Using the inequality $2xy\leq \frac{1}{\theta}x^{2}+\theta y^{2}$ for all $x,y$ and for all $\theta >  0$, we get that 
\[2\nu_{n}(\widehat{\beta}_{\widehat{m}}-\beta^{(m)}) \leq \frac{1}{\theta}\|\widehat{\beta}_{\widehat{m}}-\beta^{(m)}\|_{\widetilde{\Gamma}_{n}}^{2} + \theta\sup_{f\in S_{m\vee \widehat{m}}^{\widetilde{\Gamma}_{n}}}(\nu_{n}(f))^{2}. \]
Let us focus now on the $r_n$ part. Using the Cauchy--Schwarz inequality for the scalar product in $\mathbb{R}^{n}$ we get that
\[r_{n}(\widehat{\beta}_{\widehat{m}}-\beta^{(m)})\leq \Bigl(\frac{1}{n}\sum_{i=1}^{n}\langle\widehat{\beta}_{\widehat{m}}-\beta^{(m)},\widetilde{X}_{i}\rangle_{L^{2}}^{2}\Bigr)^{1/2}\Bigl(\frac{1}{n}\sum_{i=1}^{n}\langle\beta,X_{i}-\widetilde{X}_{i}\rangle_{L^{2}}^{2}\Bigr)^{1/2}.\]
Using now the Cauchy--Schwarz inequality for the scalar product $\langle\cdot,\cdot\rangle_{L^{2}}$, we get that 
\[r_{n}(\widehat{\beta}_{\widehat{m}}-\beta^{(m)})\leq \|\widehat{\beta}_{\widehat{m}}-\beta^{(m)}\|_{\widetilde{\Gamma}_{n}}\Bigl(\frac{1}{n}\sum_{i=1}^{n}\|\beta\|_{L^{2}}^{2}\| X_{i}-\widetilde{X}_{i}\|_{L^{2}}^{2}\Bigr)^{1/2}.\]
Therefore, 
\[2r_{n}(\widehat{\beta}_{\widehat{m}}-\beta^{(m)}) \leq \frac{1}{\theta}\| \widehat{\beta}_{\widehat{m}}-\beta^{(m)}\|_{\widetilde{\Gamma}_{n}}^{2} +\theta \frac{\|\beta\|_{L^{2}}^{2}}{n}\sum_{i=1}^{n}\| X_{i}-\widetilde{X}_{i}\|_{L^{2}}^{2}.\]
These two steps lead to the following inequality,
\begin{align*}
\|\widehat{\beta}_{\widehat{m}}-\beta\|_{\widetilde{\Gamma}_{n}}^{2}\leq &\|\beta^{(m)}-\beta\|_{\widetilde{\Gamma}_{n}}^{2}+\frac{2}{\theta}\|\widehat{\beta}_{\widehat{m}}-\beta^{(m)}\|_{\widetilde{\Gamma}_{n}}^{2}+\theta\sup_{f\in S_{m\vee \widehat{m}}^{\widetilde{\Gamma}_{n}}}(\nu_{n}(f))^{2} \\
&+\frac{\theta\|\beta\|_{L^{2}}^{2}}{n}\sum_{i=1}^{n}\| X_{i}-\widetilde{X}_{i}\|_{L^{2}}^{2}+\text{pen}(m)-\text{pen}(\widehat{m}).
\end{align*}
Now, as $\|\widehat{\beta}_{\widehat{m}}-\beta^{(m)}\|_{\widetilde{\Gamma}_{n}}^{2} \leq 2\|\widehat{\beta}_{\widehat{m}}-\beta\|_{\widetilde{\Gamma}_{n}}^{2} + 2 \|\beta-\beta^{(m)}\|_{\widetilde{\Gamma}_{n}}^{2}$, we obtain for all $\theta>4$:
\begin{align*}
\Bigl(1-\frac{4}{\theta}\Bigr)\|\widehat{\beta}_{\widehat{m}}-\beta\|_{\widetilde{\Gamma}_{n}}^{2}\leq &\Bigl(1+\frac{4}{\theta}\Bigr)\|\beta -\beta^{(m)}\|_{\widetilde{\Gamma}_{n}}^{2} +\theta\sup_{f\in S_{m\vee \widehat{m}}^{\widetilde{\Gamma}}}(\nu_{n}(f))^{2} \\
&+\frac{\theta\|\beta\|_{L^{2}}^{2}}{n}\sum_{i=1}^{n}\| X_{i}-\widetilde{X}_{i}\|_{L^{2}}^{2}+\text{pen}(m)-\text{pen}(\widehat{m}),
\end{align*}
which leads to :
\[\begin{aligned}
\mathbb{E}\bigl(\|\widehat{\beta}_{\widehat m}-\beta\|_{\widetilde{\Gamma}_{n}}^{2}\mathds{1}_{\overline {G}}\bigr)
\leq &\frac{\theta+4}{\theta-4}\mathbb{E}\bigl(\|\beta-\beta^{(m)}\|_{\widetilde{\Gamma}_{n}}^{2}\bigr)+\frac{\theta^{2}}{\theta-4}\mathbb{E}\Bigl(\sup_{f\in S_{m\vee \widehat m}^{\widetilde{\Gamma}_{n}}}(\nu_{n}(f))^{2}\Bigr)\\
&+\frac{\theta^{2}}{\theta-4}\frac{\|\beta\|_{L^{2}}^{2}}{n}\mathbb{E}\Bigl(\sum_{i=1}^{n}\|X_{i}-\widetilde{X}_{i}\|_{L^{2}}^{2}\Bigr)+\frac{\theta}{\theta-4}
\mathbb{E}\bigl(\mathrm{pen}(m)-\mathrm{pen}(\widehat m)\bigr).
\end{aligned}\]
Now, for $x>0$ and $Z$ a random variable we have that $Z=(Z-x)_{+}+x-(x-Z)_{+}\leq (Z-x)_{+}+x$. Therefore,
\[\sup_{f\in S_{m\vee \widehat m}^{\widetilde{\Gamma}_{n}}}(\nu_{n}(f))^{2} \leq \Bigl(\sup_{f\in S_{m\vee \widehat m}^{\widetilde{\Gamma}_{n}}}(\nu_{n}(f))^{2}-p(m,\widehat{m})\Bigr)_{+}+p(m,\widehat{m}),\]
with $p(m,\widehat{m})=(1+\delta)\sigma^{2}D_{m\vee\widehat{m}}/n=\frac{1}{\theta}\mathrm{pen}(m\vee\widehat{m})$. Obviously, for all $x,y\geq0$, $\max{(x,y)}\leq x+y$. Thus $\mathrm{pen}(m)-\mathrm{pen}(\widehat{m})+\theta p(m,\widehat{m})\leq \mathrm{pen}(m)-\mathrm{pen}(\widehat{m})+\mathrm{pen}(m)+\mathrm{pen}(\widehat{m})=2\mathrm{pen}(m)$. Therefore, 
\[\begin{aligned}\mathbb{E}\bigl(\|\widehat{\beta}_{\widehat{m}}-\beta\|_{\widetilde{\Gamma}_{n}}^{2}\mathds{1}_{\overline{G}}\bigr)&\le\frac{\theta+4}{\theta-4} \mathbb{E}\bigl(\|\beta-\beta^{(m)}\|_{\widetilde{\Gamma}_{n}}^{2}\bigr) \\
&\quad + \frac{\theta^{2}}{\theta-4} \sum_{m'\in \mathcal{M}_{n,p}}\mathbb{E}\Bigl( \Bigl[ \sup_{f\in S_{m\vee m'}^{\widetilde{\Gamma}_{n}}} (\nu_{n}(f))^{2} - p(m,m') \Bigr]_{+} \Bigr) \\
&\quad + \frac{\theta^{2}}{\theta-4} \frac{\|\beta\|_{L^{2}}^{2}}{n} \sum_{i=1}^{n} \mathbb{E}\bigl(\|X_{i}-\widetilde{X}_{i}\|_{L^{2}}^{2}\bigr) \\
&\quad + \frac{2\theta}{\theta-4} \mathrm{pen}(m).
\end{aligned}\]
Using the Lemma 1 of Brunel et al. \cite{brunel2013nonasymptoticadaptivepredictionfunctional}, we find that
\[\sum_{m'\in \mathcal{M}_{n,p}}\mathbb{E}\Bigl([\sup_{f\in S_{m\vee m'}^{\widetilde{\Gamma}_{n}}}(\nu_{n}(f))^{2}-p(m,m')]_{+}\Bigr) \leq \frac{C(l,\delta)}{n}\sigma^{2}.\]
(The proof of this result is based on the Corollary 5.1 of Baraud \cite{baraud2000model}).
To conclude this proof we need to control the term $\mathbb{E}(\| \beta\|_{\widetilde{\Gamma}_{n}}^{2}\mathds{1}_{\overline{G}^{c}})$. Applying Cauchy--Schwarz inequality for the scalar product $\langle Z_{1},Z_{2}\rangle=\mathbb{E}(Z_{1}Z_{2})$ we  have that,
\[\mathbb{E}(\| \beta\|_{\widetilde{\Gamma}_{n}}^{2}\mathds{1}_{\overline{G}^{c}})\leq \Bigl[\mathbb{E}(\|\beta\|_{\widetilde{\Gamma}_{n}}^{4})\mathbb{P}(\overline{G}^{c})\Bigr]^{1/2}.\]
Using the Lemma \ref{lem:A3} in the appendix leads to,
\[\mathbb{P}(\overline{G}^{c})\leq D_{N_{n,p}}\exp{\Bigl(-\frac{n}{4\ln{n}(K_{1}^{2}+1)^{2}}\Bigr)}.\]
This quantity is negligible for a sufficiently large $n$ as we chose $D_{N_{n,p}}< n/\ln^{2}{n}$. Now, by definition of the norm $\|\cdot\|_{\widetilde{\Gamma}_n}$ we have,
\begin{equation*}
    \begin{aligned}
        \mathbb{E} \left(\|\beta\|_{\widetilde{\Gamma}_{n}}^{4}\right) &= \mathbb{E}\left(\left(\frac{1}{n}\sum_{i=1}^{n}\langle\beta,\widetilde{X}_{i}\rangle_{L^{2}}^{2}\right)^{2}\right) \\
        &= \mathbb{E}\left(\frac{1}{n^{2}}\sum_{i=1}^{n}\langle\beta,\widetilde{X}_{i}\rangle_{L^{2}}^{4}\right) + \mathbb{E}\left(\frac{2}{n^{2}}\sum_{1\leq i<i'\leq n}\langle\beta,\widetilde{X}_i\rangle_{L^{2}}^{2}\langle\beta,\widetilde{X}_{i'}\rangle_{L^{2}}^{2}\right) \\
        &= \frac{1}{n}\mathbb{E}\left(\langle\beta,\widetilde{X}_1\rangle_{L^{2}}^{4}\right) + \frac{n-1}{n}\|\beta\|_{\widetilde{\Gamma}}^{4}.
    \end{aligned}
\end{equation*}
Using the fact that $\sqrt{x+y}\leq \sqrt{x}+\sqrt{y}$ for all $x,y\geq 0$ and that $n$ is a positive integer we obtain
\begin{equation*}
\begin{aligned}
    \mathbb{E}\left(\|\beta\|_{\widetilde{\Gamma}_{n}}^{4}\right)^{1/2}&\leq \left(\frac{1}{n}\mathbb{E}\left(\langle\beta,\widetilde{X}_1\rangle_{L^{2}}^{4}\right) + \frac{n-1}{n}\|\beta\|_{\widetilde{\Gamma}}^{4}\right)^{1/2} \\
    &\leq \frac{1}{\sqrt{n}}\mathbb{E}\left(\langle\beta,\widetilde{X}_1\rangle_{L^{2}}^{4}\right)^{1/2}+\|\beta\|_{\widetilde{\Gamma}}^{2} \\
    &\leq \widetilde{C}_{\beta}.
\end{aligned}
\end{equation*}
where
\begin{equation}\label{eq:widetilde_c_beta}
    \widetilde{C}_{\beta}=\mathbb{E}\left(\langle \beta, \widetilde{X}_{1}\rangle_{L^{2}}^{4}\right)^{1/2}+\|\beta\|_{\widetilde{\Gamma}}^{2}+1.
\end{equation}
Hence since $D_{N_{n,p}}<\frac{n}{\ln^{2}n}$ we have for $n$ large enough
\begin{equation*}
    \begin{aligned}
        \Bigl[\mathbb{E}(\|\beta\|_{\widetilde{\Gamma}_{n}}^{4})\mathbb{P}(\overline{G}^{c})\Bigr]^{1/2}&\leq \widetilde{C}_{\beta}\frac{n^{1/2}}{\ln{n}}\exp\left(-\frac{n}{8\ln{n}(K_{1}^{2}+1)^{2}}\right) \\
        &\leq \frac{\widetilde{C}_{\beta}}{n}.
    \end{aligned}
\end{equation*}
Lemma \ref{lem:C_beta} allows us to conclude the proof of this proposition.

\begin{lemma}
\label{lem:A1}
    Let us define for all $l,j\in\mathbb{N}\backslash \{0\}$ \[\langle\phi_{l},\phi_{j}\rangle_{p}:=\frac{1}{p}\sum_{h=0}^{p-1}\phi_{l}(t_{h})\phi_{j}(t_{h}),\]
    where we recall that $t_h=\frac{h}{p}$ and $h=0,\ldots,p-1$. We have 
    $$
\langle\phi_{l},\phi_{j}\rangle_{p} = \left\{
    \begin{array}{ll}
        1 & \mbox{if } l=j=1, \\
        \mathds{1}_{\{l \equiv j [2p]\}} + \mathds{1}_{\{l \equiv -j [2p]\}}  & \mbox{if } l \text{ and }j\text{ are even }, \\
        \mathds{1}_{\{l \equiv j [2p]\}} - \mathds{1}_{\{l \equiv -j +2[2p]\}}  & \mbox{if } l \text{ and }j\text{ are odd and strictly larger than 1 }, \\
        \sqrt{2}\mathds{1}_{\{l \equiv 0 [2p]\}} & \mbox{if }j=1 \text{ and } l \text{ is even}, \\
        \sqrt{2}\mathds{1}_{\{j \equiv 0 [2p]\}} & \mbox{if }l=1 \text{ and } j \text{ is even}, \\
        0  & \mbox{otherwise}.
    \end{array}
\right.
$$
Moreover, if $l$ and $j$ are strictly lower than $p$ then $\langle\phi_{l},\phi_{j}\rangle_{p}=\delta_{lj}.$
\end{lemma}

\begin{proof}
We distinguish four cases.

\textbf{Case 1:} $j$ and $l$ are even, meaning that there exist two positive integers $a$ and $b$ such that $j=2a$ and $l=2b$.
\newline
Using the fact that for all $x\in\mathbb{R}$, $\cos(x)=(e^{ix}+e^{-ix})/{2}$, we have that 
\[\sum_{h=0}^{p-1}\cos\Bigl(\frac{2\pi ah}{p}\Bigr)\cos\Bigl(\frac{2\pi bh}{p}\Bigr)=\frac{1}{4}\sum_{h=0}^{p-1}\left(e^{\frac{i2\pi (a-b)h}{p}}+e^{-\frac{i2\pi (a-b)h}{p}}+e^{\frac{i2\pi (a+b)h}{p}}+e^{-\frac{i2\pi (a+b)h}{p}}\right).\]
Now, let us define $S(r):=\sum_{h=0}^{p-1}e^{i\frac{2\pi rh}{p}}$ for all integer $r$. 
If $r \equiv 0[p]$ then $S(r)=\sum_{h=0}^{p-1}1=p$. Otherwise, 
\[S(r)=\frac{1-e^{i2\pi r}}{1-e^{\frac{i2\pi r}{p}}}=0.\]
Moreover $S(-r)=\overline{S(r)}=S(r)$ as for all integer $r$, $S(r)$ is a real number. Hence,
\[\sum_{h=0}^{p-1}\cos\Bigl(\frac{2\pi ah}{p}\Bigr)\cos\Bigl(\frac{2\pi bh}{p}\Bigr)=\frac{p}{2}\Bigl[\mathds{1}_{\{b\equiv a[p]\}}+\mathds{1}_{\{b\equiv -a[p]\}}\Bigr],\]
and
\[
\langle\phi_{l},\phi_{j}\rangle_{p} = \mathds{1}_{\{l \equiv j [2p]\}} + \mathds{1}_{\{l \equiv -j [2p]\}}.
\]

\textbf{Case 2: }$j$ and $l$ are odd and strictly bigger than 1, meaning that there exist two positive integer $a$ and $b$ such that $j=2a+1$ and $l=2b+1$. Using the fact that for all $x\in\mathbb{R}$, $\sin(x)=(e^{ix}-e^{-ix})/(2i)$ we have that 
\[\sum_{h=0}^{p-1}\sin\Bigl(\frac{2\pi ah}{p}\Bigr)\sin\Bigl(\frac{2\pi bh}{p}\Bigr)=-\frac{1}{4}\sum_{h=0}^{p-1}[-e^{\frac{i2\pi (a-b)h}{p}}-e^{-\frac{i2\pi (a-b)h}{p}}+e^{\frac{i2\pi (a+b)h}{p}}+e^{-\frac{i2\pi (a+b)h}{p}}].\]
Therefore,
\[\begin{aligned}
\sum_{h=0}^{p-1} \sin\Bigl(\frac{2\pi ah}{p}\Bigr) \sin\Bigl(\frac{2\pi bh}{p}\Bigr) 
&= \frac{p}{2} \bigl( \mathds{1}_{\{b \equiv a [p]\}} - \mathds{1}_{\{b \equiv -a [p]\}} \bigr),
\end{aligned}\]
and $$\langle\phi_{l},\phi_{j}\rangle_{p}=\mathds{1}_{\{l \equiv j [2p]\}} - \mathds{1}_{\{l \equiv -j+2 [2p]\}}.$$

\textbf{Case 3: } $j$ is equal to 1. 
If $l$  is also equal to 1 then $\langle\phi_{l},\phi_{j}\rangle_{p}=1$. If $l$ is odd and strictly larger than 1, then there exists a positive integer $b$ such that $l=2b+1$ and 
\[\langle\phi_{l},\phi_{j}\rangle_{p}=\frac{\sqrt{2}}{p}\sum_{h=0}^{p-1}\sin\Bigl(\frac{2\pi bh}{p}\Bigr)=0.\]
If $l$ is even then there exists a positive integer $b$ such that $l=2b$ and 
\[\langle\phi_{l},\phi_{j}\rangle_{p} = \frac{\sqrt{2}}{p}\sum_{h=0}^{p-1}\cos\Bigl(\frac{2\pi bh}{p}\Bigr) = \frac{\sqrt{2}}{p}\sum_{h=0}^{p-1}1 = \sqrt{2}\]
provided that $b \equiv 0 [p]$, and $\langle\phi_l,\phi_j\rangle_p=0$ otherwise. Consequently, we have
\[\langle\phi_l,\phi_j\rangle_p=\sqrt{2}\]
whenever $l\equiv 0 [2p]$ and $j=1$. By symmetry, the case $l=1$ is treated analogously.

\textbf{Case 4: } $j$ is odd and strictly larger than 1, and $l$ is even. In this situation there exist two positive integers $a$ and $b$ such that $j=2a+1$ and $l=2b$. Then,
\[\sum_{h=0}^{p-1}\sin\Bigl(\frac{2\pi ah}{p}\Bigr)\cos\Bigl(\frac{2\pi bh}{p}\Bigr)=\frac{1}{2}\sum_{h=0}^{p-1}\Bigl[\sin\Bigl(\frac{2\pi (a+b)h}{p}\Bigr)+\sin\Bigl(\frac{2\pi (a-b)h}{p}\Bigr)\Bigl].\]
For any integer $r$, $\sum_{h=0}^{p-1}\sin((2\pi rh)/p)=0$, hence
\[\sum_{h=0}^{p-1}\sin\Bigl(\frac{2\pi ah}{p}\Bigr)\cos\Bigl(\frac{2\pi bh}{p}\Bigr)=0,\] and $\langle\phi_{l},\phi_{j}\rangle_{p}=0$. By symmetry of the roles of $l$ and $j$ we have the same result if $l$ is odd and strictly bigger than 1, and $j$ is even.

\end{proof}

\begin{lemma}
\label{lem:A2}
    Let us assume that $D_{N_{n,p}} < p$, then we have for all $j,k\in\{1,\ldots,D_{N_{n,p}}\}$
    \[\mathbb{E}(\widetilde{x}_{i,j} \widetilde{x}_{i,k}) =\delta_{jk}\widetilde{\lambda}_{j},\] where 
    \begin{equation}\label{eq:tilde_lambda}
        \widetilde{\lambda}_{j}=\lambda_{j} + \overline{\lambda}_{j} + \frac{\tau^{2}}{p},
    \end{equation}
    and
    \begin{align*}
\overline{\lambda}_{j} 
&= \mathds{1}_{\{j \text{ even}\}} \sum_{q \ge 1} \bigl( \lambda_{j+2qp} + \lambda_{2qp-j} \bigr) \\
&\quad + \mathds{1}_{\{j \text{ odd},\, j > 1\}} \sum_{q \ge 1} \bigl( \lambda_{j+2qp} + \lambda_{2qp-j+2} \bigr) \\
&\quad + 2\mathds{1}_{\{j=1\}} \sum_{q \ge 1} \lambda_{2qp}.
\end{align*}
\end{lemma}
\begin{proof}
For all $i\in\{1,\ldots,n\}$ and $j\in\{1,\ldots,D_{N_{n,p}}\}$, let $\widetilde{x}_{i,j}=\overline{x}_{i,j}+\overline{\eta}_{i,j}$ where 
\begin{equation}\label{eq:eta_bar_x_bar}
    \overline{x}_{i,j}=\frac{1}{p}\sum_{h=0}^{p-1}X_{i}(t_{h})\phi_{j}(t_{h}) \quad\text{and} \quad \overline{\eta}_{i,j}=\frac{1}{p}\sum_{h=0}^{p-1}\eta_{i,h}\phi_{j}(t_{h}) .
\end{equation}
Then, as the $\eta_{i,h}$'s are supposed to be i.i.d, independent of everything else, centered and of variance $\tau^{2}$ we have for all $j,k\in\{1,\ldots,D_{N_{n,p}}\}$,
\[\mathbb{E}(\widetilde{x}_{i,j}\widetilde{x}_{i,k}) = \mathbb{E}(\overline{x}_{i,j}\overline{x}_{i,k})+\mathbb{E}(\overline{\eta}_{i,j}\overline{\eta}_{i,k}).\]
Moreover $\mathbb{E}(\overline{\eta}_{i,j}\overline{\eta}_{i,k})=\frac{\tau^{2}}{p}\langle\phi_{j},\phi_{k}\rangle_{p}$ as the $\eta_{i,h}$'s are supposed to be i.i.d, centered and of variance $\tau^{2}$. As $j,k\leq D_{N_{n,p}}<p$ the previous lemma leads to
\[\mathbb{E}(\overline{\eta}_{i,j}\overline{\eta}_{i,k})=\frac{\tau^{2}}{p}\delta_{jk}.\]
We can then write $\overline{x}_{i,j}$ as 
\begin{equation}\label{eq:e_ij}
    \overline{x}_{i,j}=x_{i,j}+e_{i,j} \quad \text{where} \quad  x_{i,j}=\langle X_{i},\phi_{j}\rangle_{L^{2}}
\end{equation}
Therefore $\mathbb{E}(\overline{x}_{i,j}\overline{x}_{i,k})=\mathbb{E}(x_{i,j}x_{i,k})+ \mathbb{E}(x_{i,j}e_{i,k})+\mathbb{E}(x_{i,k}e_{i,j}) + \mathbb{E}(e_{i,k}e_{i,j})$. 

\medskip\noindent
The first term can be rewritten as $\mathbb{E}(x_{i,j}x_{i,k})=\lambda_{j}\delta_{jk}$ where the $\lambda_{j}$'s are the eigenvalues of $\Gamma$. Indeed, for any \(j,k \geq 1\) and for any $i\in\{1,\ldots,n\}$ we have, using the fact that the $X_i$'s are periodic and second order stationary,
\begin{equation}\label{eq:x_uncorrelated}
\begin{aligned}
    \mathbb{E}(x_{i,j} x_{i,k})
&= \mathbb{E}\left(\langle X_i,\phi_j\rangle_{L^{2}} \,\langle X_i,\phi_k\rangle_{L^{2}}\right) \\[6pt]
&= \left\langle \mathbb{E}\left(\langle \phi_j, X_i\rangle_{L^{2}}\, X_i \right), \phi_k \right\rangle_{L^{2}} \\[6pt]
&= \langle \Gamma \phi_j, \phi_k \rangle_{L^{2}} \\[6pt]
&= \langle \lambda_j \phi_j, \phi_k \rangle_{L^{2}} \\[6pt]
&= \lambda_j \,\delta_{jk}.
\end{aligned}
\end{equation}
Moreover for all $i\in\{1,\ldots,n\}$ and $t\in[0,1]$, $X_{i}(t)=\sum_{l\geq 1}x_{i,l}\phi_{l}(t)$, therefore $X_{i}(t_{h})=\sum_{l\geq 1}x_{i,l}\phi_{l}(t_{h})$ and using Fubini we get,
\begin{equation*}
\begin{aligned}
\mathbb{E}\bigl(x_{i,j} e_{i,k}\bigr)
&= \mathbb{E}\left(
      x_{i,j}\Bigl(
        \tfrac1p\sum_{h=0}^{p-1} X_i(t_h)\phi_k(t_h)-x_{i,k}
      \Bigr)
    \right) \\
&= \mathbb{E}\left(
      x_{i,j}\sum_{l\geq 1} x_{i,l}\Bigl(
        \tfrac1p\sum_{h=0}^{p-1} \phi_l(t_h)\phi_k(t_h)-\delta_{lk}
      \Bigr)
    \right) \\
&= \sum_{l\geq 1} \mathbb{E}\bigl(x_{i,j}x_{i,l}\bigr)
   \Bigl(
     \tfrac1p\sum_{h=0}^{p-1} \phi_l(t_h)\phi_k(t_h)-\delta_{lk}
   \Bigr) \\
&= \lambda_j\Bigl(
     \langle \phi_j,\phi_k\rangle_{p}-\delta_{jk}
   \Bigr) \\
&= 0,
\end{aligned}
\end{equation*}
as $j,k \le D_{N_{n,p}} < p$. Similarly, we obtain $\mathbb{E}(x_{i,k}e_{i,j}) = 0$. To compute $\mathbb{E}(e_{i,j}e_{i,k})$, we first express $e_{i,j}$ for any $j \in \{1, \dots, D_{N_n}\}$ as
\[e_{i,j} = \sum_{l \ge 1} x_{i,l} \bigl( \langle\phi_{l},\phi_{j}\rangle_{p} - \delta_{lj} \bigr).\]
Using Lemma~\ref{lem:A1}, we have
If $j=1$, $$
\langle\phi_{l},\phi_{1}\rangle_{p} = \left\{
    \begin{array}{ll}
        \sqrt{2} & \mbox{if } l \in (2pq)_{q\geq 1}, \\
        1 & \mbox{if } l=j, \\
        0 & \mbox{otherwise.}
    \end{array}
\right.
$$
If $j>1$ and $j$ is even, 
$$
\langle\phi_{l},\phi_{j}\rangle_{p} = \left\{
    \begin{array}{ll}
        1 & \mbox{if } l\in (j+2qp)_{q\geq 0} \text{ or } l\in(-j+2qp)_{q\geq 1},\\
        0 & \mbox{otherwise.}
    \end{array}
\right.
$$
If $j>1$ and $j$ is odd,
$$
\langle\phi_{l},\phi_{j}\rangle_{p} = \left\{
    \begin{array}{ll}
        1 & \mbox{if } l\in (j+2qp)_{q\geq 0}, \\
        -1 & \mbox{if } l\in (-j+2+2qp)_{q\geq 0},\\
        0 & \mbox{otherwise.}
    \end{array}
\right.
$$
Therefore,  
\begin{equation}\label{eq:e_ij_detailed}
    e_{i,j} = \left\{
    \begin{array}{ll}
        \sqrt{2}\sum_{q\geq1}x_{i,2qp} & \mbox{if } j=1, \\\
        \sum_{q\geq1}[x_{i,2s+2qp}+x_{i,-2s+2qp}] & \mbox{if } j=2s \text{ with } s\in\{1,\ldots,\lfloor \frac{D_{N_{n,p}}}{2} \rfloor\},\\
        \sum_{q\geq1}[x_{i,2s+1+2qp}-x_{i,2qp+2-(2s+1)}] & \mbox{if } j=2s+1 \text{ with } s\in\{1,\ldots,\lfloor \frac{D_{N_{n,p}}}{2} \rfloor\}.
    \end{array}
\right.
\end{equation}
\newline
Given Equation \eqref{eq:x_uncorrelated}), we notice that, when expanding $\mathbb{E}(e_{i,j}e_{i,k})$ into sums of terms of the form $\mathbb{E}(x_{i,t}x_{i,s})$ every term vanishes unless the indices coincide. In other words, $\mathbb{E}(e_{i,j}e_{i,k})=0$ as soon as the index sets used in $e_{i,j}$ and $e_{i,k}$ are disjoint.

Let us first describe the sets of indices involved:
\begin{itemize}
    \item For $j=1$: 
    \[
      T_1 = \{2qp : q\geq 1\}, \quad \text{all even multiples of $p$.}
    \]
    \item For $j=2s$ with $1\leq s\leq \lfloor (p-1)/2\rfloor$:
    \[
      T_{2s} = \{2s+2qp,-2s+2qp : q\geq 1\}, \quad \text{all even.}
    \]
    \item For $j=2s+1$ with $1\leq s\leq \lfloor (p-1)/2\rfloor$:
    \[
      T_{2s+1} = \{2s+1+2qp,2qp+1-2s : q\geq 1\}, \quad \text{all odd.}
    \]
\end{itemize}

\textbf{Trivial cases . } 
From this, two trivial vanishing cases appear: when $j$ and $k$ don't have the same parity and when $j=1$ and $k>1$ (or $k=1$ and $j>1$). As the role of $j$ and $k$ are symmetric, we can assume that $j$ is odd and $k$ is even. Then $T_j$ contains only odd indices while $T_k$ contains only even ones. Hence $T_j\cap T_k=\varnothing$ and 
\[\mathbb{E}(e_{i,j}e_{i,k})=0.\]
For the second case, let us consider $j=1$ versus $k>1$. If $k$ is odd, this falls into the previous parity case. If $k=2s$ is even, we must check $T_1\cap T_{2s}=\varnothing$. Let us suppose that there exist $q,q'\ge 1$ such that
\[2qp = 2s+2q'p \quad \text{or} \quad 2qp = -2s+2q'p.\]
This would imply $2s\equiv 0 [2p]$, or \ $s\equiv 0 [p]$. But since $1\leq s\leq \lfloor \frac{D_{N_{n,p}}}{2} \rfloor\leq \lfloor (p-1)/2\rfloor$, this is impossible. Therefore $T_1\cap T_{2m}=\varnothing$ and
\[\mathbb{E}(e_{i,1}e_{i,2s})=0.\]
To conclude, $j$ and $k$ have different parity, or if one of them is $1$ and the other is strictly larger than 1, then $\mathbb{E}(e_{i,j}e_{i,k})=0$. 

\medskip\noindent
\textbf{$j$ and $k$ are even . } Then there exist
$a,b\in\{1,\dots,\lfloor \frac{D_{N_{n,p}}}{2} \rfloor\}$ with $j=2a$ and $k=2b$.
Using the notation introduced above, the index sets are
\[
T_{2a}=\{2a+2qp, -2a+2qp : q\geq 1\},\qquad
T_{2b}=\{2b+2q'p, -2b+2q'p : q'\geq 1\}.
\]
To determine whether $\mathbb{E}(e_{i,2a}e_{i,2b})$ can be nonzero we must look for coincidences between elements of $T_{2a}$ and $T_{2b}$.
Thus we consider equalities of the form (with $q,q'\geq 1$):
\begin{align}
&2a+2qp = 2b+2q'p, \tag{I}\\
&2a+2qp = -2b+2q'p, \tag{II}\\
&-2a+2qp = 2b+2q'p, \tag{III}\\
&-2a+2qp = -2b+2q'p. \tag{IV}
\end{align}
\begin{itemize}
\item Cases (I) and (IV) :  From (I) we get
\[2(a-b)=2p(q'-q)\quad\Longleftrightarrow\quad a-b = p(q'-q).\]
Since $|a-b|\le \lfloor \frac{D_{N_{n,p}}}{2} \rfloor\le \lfloor (p-1)/2\rfloor < p$, the only multiple of $p$
in that range is $0$. Hence $a-b=0$, or $a=b$. If $a=b$ then
(I) reduces to $2qp=2q'p$, or $q=q'$. The same reasoning applied to
(IV) yields the identical conclusion: (IV) can occur only when
$a=b$ and $q=q'$. 

\medskip\noindent
\item Cases (II) and (III) :  From (II) we obtain
\[2(a+b)=2p(q'-q)\quad\Longleftrightarrow\quad a+b = p(q'-q).\]
But $2\leq a+b \leq 2\lfloor \frac{D_{N_{n,p}}}{2} \rfloor\leq 2\lfloor (p-1)/2\rfloor \leq p-1$, therefore $1\leq a+b\leq p-1$
and therefore $a+b$ cannot equal a nonzero multiple of $p$. Thus
(II) is impossible. The same argument applied to (III),
$-2(a+b)=2p(q'-q)$, shows (III) is impossible as well.

\medskip\noindent
Combining the four cases we conclude that if $a\neq b$ then no equality between elements of $T_{2a}$ and $T_{2b}$ can occur, therefore $T_{2a}\cap T_{2b}=\varnothing$ and
\[\mathbb{E}(e_{i,2a}e_{i,2b})=0\qquad (a\neq b).\]
If $a=b$ the only coincidences are the trivial ones coming from matching identical indices $2a+2qp$ with $2a+2qp$ and $-2a+2qp$ with $-2a+2qp$. Using Equation \eqref{eq:x_uncorrelated} we get
\[\mathbb{E}(e_{i,2a}^2)=\sum_{q\ge1}\lambda_{2a+2qp}+\sum_{q\ge1}\lambda_{-2a+2qp}.\]
Hence,
\[\mathbb{E}(e_{i,j}e_{i,k}) =
\begin{cases}
0, & j\neq k,\\[4pt]
\displaystyle \sum_{q\geq 1}\lambda_{j+2qp}+\sum_{q\geq 1}\lambda_{-j+2qp}, & j=k.
\end{cases}\
\]

\textbf{$j$ and $k$ are odd and strictly bigger than 1. } Then there
exist $a,b\in\{1,\dots,\lfloor \frac{D_{N_{n,p}}}{2} \rfloor\}$ with
$j=2a+1$ and $k=2b+1$. Recalling the notation introduced above, the
index sets are
\[\begin{aligned}
T_{2a+1} &= \{2a+1+2qp,\, 2qp+1-2a : q\geq 1\}, \\
T_{2b+1} &= \{2b+1+2q'p,\, 2q'p+1-2b : q'\geq 1\}.
\end{aligned}\]
To determine whether $\mathbb{E}(e_{i,2a+1}e_{i,2b+1})$ can be nonzero
we must look for coincidences between elements of $T_{2a+1}$ and
$T_{2b+1}$. There are four types of equalities to consider (with
$q,q'\geq 1$):
\begin{align}
& 2a+1+2qp = 2b+1+2q'p, \tag{I}\\
& 2a+1+2qp = 2q'p+1-2b, \tag{II}\\
&2qp+1-2a = 2b+1+2q'p, \tag{III}\\
& 2qp+1-2a = 2q'p+1-2b. \tag{IV}
\end{align}
We treat each case separately.
\item Cases (I) and (IV) : From (I) we get
\[2(a-b)=2p(q'-q)\quad\Longleftrightarrow\quad a-b=p(q'-q).\]
Since $|a-b|\le \lfloor \frac{D_{N_{n,p}}}{2} \rfloor\leq \lfloor (p-1)/2\rfloor < p$, the only multiple of $p$ in that range is $0$. Hence $a-b=0$, or $a=b$. If $a=b$ then (I) reduces to $2qp=2q'p$, hence $q=q'$. Thus (I) can occur only
when $a=b$ and $q=q'$. The same reasoning applied to (IV) yields the identical conclusion : (IV) can occur only when $a=b$ and $q=q'$.
\item Cases (II) and (III) : From (II) we obtain
\[2(a+b)=2p(q'-q)\quad\Longleftrightarrow\quad a+b = p(q'-q).\]
But $2\leq a+b \leq 2\lfloor \frac{D_{N_{n,p}}}{2} \rfloor \leq 2\lfloor (p-1)/2\rfloor \leq p-1$, thus $1\leq a+b\leq p-1$
and therefore $a+b$ cannot equal a nonzero multiple of $p$. Hence
(II) is impossible for any $a,b,q,q'$. The same reasoning applied to (III) yields the identical conclusion.

\medskip\noindent
Combining the four cases we conclude that if $a\neq b$ then no equality between elements of $T_{2a+1}$ and $T_{2b+1}$ can occur, therefore \(T_{2a+1}\cap T_{2b+1}=\varnothing\) and
\[\mathbb{E}(e_{i,2a+1}e_{i,2b+1})=0\qquad (a\neq b)\]
If $a=b$  the only coincidences are the trivial ones coming from matching identical indices $2a+1+2qp$ with itself and $2qp+1-2a$ with itself. Using Equation \eqref{eq:x_uncorrelated} we therefore get the expression
\[\mathbb{E}(e_{i,2a+1}^2)=\sum_{q\ge1}\lambda_{2a+1+2qp}+\sum_{q\ge1}\lambda_{2qp+1-2a}\]
Hence,
\[\mathbb{E}(e_{i,j}e_{i,k}) =
\begin{cases}
0, & j\neq k,\\[4pt]
\displaystyle \sum_{q\geq 1}\lambda_{j+2qp}+\sum_{q\geq 1}\lambda_{2qp+2-j}, & j=k.
\end{cases}
\]
\end{itemize}
\textbf{$j$ and $k$ are equal to 1 . } Let us recall that $e_{i,1}=\sqrt{2}\sum_{q\ge1} x_{i,2qp}$ and the index set is
$T_1=\{2qp:q\geq 1\}$. To check possible coincidences between indices in
the product $e_{i,1}^2$ we look for solutions $q,q'\geq 1$ of
\[2qp = 2q'p \qquad\Longleftrightarrow\qquad q=q'.\]
Hence we get
\[\mathbb{E}(e_{i,1}^2)=2\sum_{q\geq 1}\mathbb{E}\big(x_{i,2qp}^2\big)=2\sum_{q\geq 1}\lambda_{2qp}.\]
Finally get that $\mathbb{E}(\widetilde{x}_{i,j} \widetilde{x}_{i,k}) = \delta_{jk} \left( \lambda_{j} + \overline{\lambda}_{j} + \frac{\tau^{2}}{p} \right)=\delta_{jk}\widetilde{\lambda}_{j}$ and $\mathrm{Var}(\widetilde{x}_{i,j})=\widetilde{\lambda}_{j}$.

\end{proof}

\begin{lemma}
    \label{lem:A3}
    Let us assume that $\lambda_{D_{N_{n,p}}}\geq \frac{2}{n^{\alpha}}$ and $\alpha>2a$ in the definition of $s_{n}$ defined in Equation \eqref{eq:s_n}. We also suppose that $D_{N_{n,p}} < \min{\left(\frac{n}{\ln^{2}n},p\right)}$ and that (\textbf{H1}) and (\textbf{H2}) hold. Then we have
    \[\mathbb{P}(\overline{G}^{c})\leq D_{N_{n,p}}\exp\left(-\frac{n}{4(K_{1}^{2}+1)^{2}\ln{n}}\right).\]
\end{lemma}
\begin{proof}
    The proofs for this lemma and the succeeding one were developed by Brunel and Roche \cite{Brunel2015penalized}, based on their work in Lemma 6. We only adapt them to our setting. First of all, we have \[\mathbb{P}(\overline{G}^{c})=\mathbb{P}\left(\bigcup_{m\in\mathcal{M}_{n,p}}G_{m}^{c}\right)\leq \sum_{m\in\mathcal{M}_{n,p}}\mathbb{P}(\widehat{\lambda}_{m}^{(p)}<s_{n}),\] where $\widehat{\lambda}_{m}^{(p)}$ is the smallest eigenvalue of $\Phi_{m}$ and $s_{n}=\frac{2}{n^{\alpha}}(1-\frac{1}{\sqrt{\ln{n}}})$ with $\alpha>2a$. The last two quantities are defined in Section \ref{sec:2.2}. By Lemma \ref{lem:A4} we have 
    \[\mathbb{P}(\widehat{\lambda}_{m}^{(p)}<s_{n})\leq\mathbb{P}\left(\widehat{\mu}_{m}^{(p)}<\frac{s_{n}}{\min_{1\leq j\leq D_{m}}\widetilde{\lambda}_{j}}\right),\]
    where $\widehat{\mu}_{m}^{(p)}$ is the smallest eigenvalue of the matrix $\Psi_{m}$ defined in Lemma \ref{lem:A4}. Since for all $j\in\{1,\ldots,D_{N_{n,p}}\}$ we have $$\widetilde{\lambda}_{j}=\lambda_{j}+\overline{\lambda}_{j}+\frac{\tau^{2}}{p}\geq \lambda_j,$$then using the fact that the eigenvalues decrease we get
    \[\min_{1\leq j\leq D_{N_{n,p}}}\widetilde{\lambda}_{j}\geq \min_{1\leq j\leq D_{N_{n,p}}}\lambda_{j}=\lambda_{D_{N_{n,p}}}.\]
    Using the assumption that $\lambda_{D_{N_{n,p}}}\geq \frac{2}{n^{\alpha}}$ and the definition of $s_{n}$ given in \eqref{eq:s_n} we obtain, 
    \[\frac{s_{n}}{\min_{1\leq j\leq D_{m}}\widetilde{\lambda}_{j}}\leq 1-\frac{1}{\sqrt{\ln{n}}}.\]
    We can then apply Lemma \ref{lem:A5} with $\omega=1-1/\sqrt{\ln{n}}$ and get
    \[\mathbb{P}(\widehat{\lambda}_{m}^{(p)}<s_{n})\leq 2\exp\left(-\frac{n}{4(K_{1}^{2}+1)^{2}\ln{n}}\right).\] As $N_{n,p}\leq D_{N_{n,p}}/2$,
    \[\sum_{m\in\mathcal{M}_{n,p}}\mathbb{P}(\widehat{\lambda}_{m}^{(p)}<s_{n})\leq D_{N_{n,p}}\exp\left(-\frac{n}{4(K_{1}^{2}+1)^{2}\ln{n}}\right).\]
    Combining this with the first inequality of this proof gives the desired result.
\end{proof}

\begin{lemma}\label{lem:A4}
For $m\in\mathcal{M}_{n,p}$, let $\widehat{\lambda}_{m}$ be the smallest eigenvalue of the matrix $\Phi_{m}$ defined in Section \ref{sec:2.2}, and $\widehat\mu_{m}^{(p)}$ be the smallest eigenvalue of the matrix
\begin{equation}\label{eq:Psi}
    \Psi_{m}:=\left(\frac{1}{n}\sum_{i=1}^{n}\frac{\widetilde{x}_{i,j}}{\sqrt{\widetilde{\lambda}_{j}}}\frac{\widetilde{x}_{i,k}}{\sqrt{\widetilde{\lambda}_{k}}}\right)_{1\leq j,k \leq D_{m}}.
\end{equation}
Then,
\[\frac{\widehat{\lambda}_{m}^{(p)}}{\rho({\widetilde{\Gamma}})}\leq\widehat\mu_{m}^{(p)}\leq \widehat{\lambda}_{m}^{(p)}\left(\min_{1\leq j\leq D_{m}}\widetilde\lambda_{j}\right)^{-1},\]
where $\rho({\widetilde{\Gamma}})$ is the spectral radius of the operator $\widetilde{\Gamma}$.
\end{lemma}
\begin{proof}
    Now let us define
    \begin{equation}\label{eq:Lambda}
        \Lambda_{m} := \mathrm{diag}(\sqrt{\widetilde{\lambda}_{1}}, \dots, \sqrt{\widetilde{\lambda}_{D_{m}}}).
    \end{equation}
    We have that
    \[\Phi_{m}=\Lambda_{m}\Psi_{m}\Lambda_{m}.\]
    Recall that $\widehat{\mu}_{m}^{(p)}$ is the smallest eigenvalue of the matrix $\Psi_{m}$. We notice that $\widehat{\mu}_{m}^{(p)}>0$ directly implies that $\Psi_{m}$ is a positive definite matrix which by definition is invertible. As a consequence, $\det{(\Psi_{m}^{(p)})}\neq 0$. Moreover, $\Lambda_{m}$ is  invertible as it is a diagonal matrix with positive coefficient. Therefore $\det{(\Lambda_{m}^{(p)})}\neq 0$. As 
    \[\det{(\Phi_{m})}=\det{(\Lambda_{m})}\cdot\det{(\Psi_{m})}\cdot\det{(\Lambda_{m})},\]
    thus we get that $\det{(\Phi_{m})}\neq 0$ and that $\Phi_{m}$ is invertible. Hence, $\widehat{\mu}_{m}^{(p)}>0$ implies that both $\Phi_{m}$ and $\Psi_{m}$ are invertible and we have
    \[\widehat{\mu}_{m}^{(p)}=\rho(\Psi_{m}^{-1})^{-1} \qquad \text{and} \qquad \widehat{\lambda}_{m}^{(p)}=\rho(\Phi_{m}^{-1})^{-1},\]where $\rho$ is the spectral radius. We recall that for all square matrix $A$ 
    \[\| A \|_{op}=\sup_{\|a\|_{2}=1}\|Aa\|_{2}.\]If $A$ is symmetric, $\rho(A)=\| A \|_{op}$ and we get, by using the definition of the spectral radius as well as the sub-multiplicative property of matrix norms,
    \begin{align*}
\rho(\Phi_{m}^{-1}) &= \| \Phi_{m}^{-1}\|_{op} \\
&= \| \Lambda_{m}^{-1}\Psi_{m}^{-1}\Lambda_{m}^{-1} \|_{op} \\
&\leq \| \Lambda_{m}^{-1} \|_{op}^{2}\| \Psi_{m}^{-1} \|_{op} \\
&= \rho(\Lambda_{m}^{-1})^{2}\rho(\Psi_{m}^{-1}).
\end{align*}
Hence, 
\[\widehat\mu_{m}^{(p)}\cdot \min_{1\leq j\leq D_{m}}\widetilde\lambda_{j}\leq \widehat{\lambda}_{m}^{(p)}.\]
As $\Psi_{m}=\Lambda_{m}^{-1}\Phi_{m}\Lambda_{m}^{-1}$, we have, using the same reasoning that 
\[\widehat{\mu}_{m}^{(p)-1}\leq \widehat{\lambda}_{m}^{(p)-1}\max_{1\leq j\leq D_m}\widetilde{\lambda}_{j} \leq \rho(\widetilde{\Gamma})\widehat{\lambda}_{m}^{(p)-1}.\]
The second inequality can be deduced from Lemma \ref{lem:eigenvalues_gamma_tilde}.

\end{proof}

\begin{lemma} \label{lem:eigenvalues_gamma_tilde}
    Let us assume that $D_{N_{n,p}} < p$. For all $k\in\{1,\ldots,D_{N_{n,p}}\}$
    \[\widetilde{\Gamma}\phi_{k}=\widetilde{\lambda}_{k}\phi_{k},\]
    where $\widetilde{\lambda}_{k}$'s defined in (\ref{eq:tilde_lambda}) are the eigenvalues of $\widetilde{\Gamma}$ associated to $\phi_k$, and for all integer $k>D_{N_{n,p}}$,
    \[\widetilde{\Gamma}\phi_{k}=0.\]
\end{lemma}
\begin{proof}
    For all $s,t\in [0,1]$ we have using the definition (\ref{eq:kernel_tilde}) and Lemma \ref{lem:A2},
\begin{equation*}
    \begin{aligned}
        \widetilde{K}(s,t) &= \mathbb{E}\left(\widetilde{X}(s)\widetilde{X}(t)\right) \\
        &= \sum_{j=1}^{D_{N_{n,p}}}\sum_{k=1}^{D_{N_{n,p}}}\phi_{j}(t)\phi_{k}(s)\mathbb{E}(\widetilde{x}_{j}\widetilde{x}_{k}) \\
        &= \sum_{j=1}^{D_{N_{n,p}}} \widetilde{\lambda}_{j}\phi_{j}(s)\phi_{j}(t).
    \end{aligned}
\end{equation*}
Hence, for all integer $k\in \{1,...,D_{N_{n,p}}\}$, 
\begin{equation*}
    \begin{aligned}
        \widetilde{\Gamma}\phi_{k}(s) &= \int_{0}^{1}\widetilde{K}(s,t)\phi_{k}(t)dt \\
        &= \sum_{j=1}^{D_{N_{n,p}}}\widetilde{\lambda}_{j}\phi_{j}(s)\int_{0}^{1}\phi_{j}(t)\phi_{k}(t)dt \\
        &= \widetilde{\lambda}_{k}\phi_{k}(s),
    \end{aligned}
\end{equation*}
and for all integer $k>D_{N_{n,p}}$,
\begin{equation*}
    \begin{aligned}
        \widetilde{\Gamma}\phi_{k}(s) &= \int_{0}^{1}\widetilde{K}(s,t)\phi_{k}(t)dt \\
        &= \sum_{j=1}^{D_{N_{n,p}}}\widetilde{\lambda}_{j}\phi_{j}(s)\int_{0}^{1}\phi_{j}(t)\phi_{k}(t)dt \\
        &= 0.
    \end{aligned}
\end{equation*}
\end{proof}

\begin{lemma}\label{lem:A5}
Let $\omega$ be a real number such that $0 < \omega < 1$.  
For $m \in \mathcal{M}_{n,p}$, consider the smallest eigenvalue $\widehat{\mu}_{m}^{(p)}$ of the matrix $\Psi_{m}$.  
Then, under Assumptions $(\textbf{H1}),(\textbf{H2})$, we have
\[
\mathbb{P}(\widehat{\mu}_{m}^{(p)}<\omega)\leq 
2\exp\left(-n\frac{(1-\omega)^{2}}{4(K_{1}^{2}+1)^{2}}\right).
\]
\end{lemma}

\begin{proof}
We have
\[\{\widehat{\mu}_{m}^{(p)} < \omega\} = \{1 - \widehat{\mu}_{m}^{(p)} > 1 - \omega\}.\]
Since $1 - \omega > 0$,
\[\{|1 - \widehat{\mu}_{m}^{(p)}| > 1 - \omega\} = \{\| \Psi_{m} - I\|_{op} > 1 - \omega\}.\]
Thus our goal is now to control $\mathbb{P}\left(\|\Psi_{m}-I\|_{op}>1-\omega\right)$. Let
\begin{equation}\label{eq:U_im}
    U_{i}^{(m)} = \left(\frac{\widetilde{x}_{i,1}}{\sqrt{\widetilde{\lambda}_1}},\ldots,\frac{\widetilde{x}_{i,D_m}}{\sqrt{\widetilde{\lambda}_{D_m}}}\right)^{T}\in\mathbb{R}^{D_m}, \qquad i\in\{1,...,n\},
\end{equation}
and
\begin{equation}\label{eq:A_m}
    A_{m} = \left(U_{1}^{(m)T},\ldots,U_{n}^{(m)T}\right)^{T}\in \mathbb{R}^{n\times D_m}.
\end{equation}
We can easily check that \[A_m^{T}A_m=\left(\sum_{i=1}^{n}\frac{\widetilde{x}_{i,j}}{\sqrt{\widetilde{\lambda}_j}}\frac{\widetilde{x}_{i,k}}{\sqrt{\widetilde{\lambda}_k}}\right)_{1\leq j,k\leq D_m},\]
and thus
\[\Psi_{m}=\frac{1}{n}A_m^{T}A_m.\]
Our aim is to apply Theorem 4.6.1 of Vershynin \cite{vershynin2026high}. In order to do this, we must check that the vectors $U_{i}^{(m)}$'s are independent, mean-zero, isotropic and sub-gaussian. As we assumed that for all $i\in\{1,\ldots,n\}$ the true curves $X_i$ are independent, that for all $i\in\{1,\ldots,n\}$, $h\in\{0,\ldots,p-1\}$ the noises $\eta_{i,h}$ are also independent and that for all $i\in\{1,\ldots,n\}$, $h\in\{0,\ldots,p-1\}$ $X_i$ and $\eta_{i,h}$ are independent we have that the $U_{i}^{(m)}$'s are also independent. Moreover, as we made the assumption that the $X_i$'s and the $\eta_{i,h}$'s are centered we have that
\begin{equation*}
    \begin{aligned}
        \mathbb{E}(\widetilde{x}_{i,j})&=\mathbb{E}\left(\frac{1}{p}\sum_{h=0}^{p-1}X_i(t_h)\phi_j(t_h)+\frac{1}{p}\sum_{h=0}^{p-1}\eta_{i,h}\phi_j(t_h)\right) \\
        &= \frac{1}{p}\sum_{h=0}^{p-1}\mathbb{E}\left(X_i(t_h)\right)\phi_j(t_h) + \frac{1}{p}\sum_{h=0}^{p-1}\mathbb{E}(\eta_{i,h})\phi_j(t_h) \\
        &= 0.
    \end{aligned}
\end{equation*}
Finally we can easily check that 
\[U_{i}^{(m)}U_{i}^{(m)T}= \left(\frac{\widetilde{x}_{i,j}}{\sqrt{\widetilde{\lambda}_j}}\cdot\frac{\widetilde{x}_{i,k}}{\sqrt{\widetilde{\lambda}_k}}\right)_{1\leq j,k\leq D_m},\]
and Lemma \ref{lem:A2} gives us that $\mathbb{E}(\widetilde{x}_{i,j}\widetilde{x}_{i,k})=\widetilde{\lambda}_{j}\delta_{jk}$. Therefore,
\[\mathbb{E}\left(U_{i}^{(m)}U_{i}^{(m)T}\right)=I_{D_m}.\]
Lemma \ref{lem:sub-gaussian} bellow gives us that the $U_{i}^{(m)}$'s are sub-gaussian. Thus all the conditions needed to apply Theorem 4.6.1 of Vershynin \cite{vershynin2026high} are satisfied and we get that 
\begin{equation*}
    \mathbb{P}\left(\left\|\Psi_m-I\right\|_{op}>\widetilde{K}^{2}\max(\widetilde{\delta},\widetilde{\delta}^{2})\right)\leq 2\exp\left(-t^{2}\right),
\end{equation*}
where $\widetilde{\delta}=C'\sqrt{\frac{D_m}{n}}+\frac{t}{\sqrt{n}}$ and $\widetilde{K}^{2}=K_{1}^{2}+1$. Let us now find which $t$ allows us to have $\widetilde{K}\widetilde{\delta}\leq 1-\omega$. Replacing $\widetilde{\delta}$ by its definition we get that
\begin{equation*}
    \begin{aligned}
        \widetilde{K}^{2}\widetilde{\delta}\leq 1-\omega &\iff C'\sqrt{\frac{D_m}{n}}+\frac{t}{\sqrt{n}}\leq \frac{1-\omega}{\widetilde{K}^{2}} \\
        &\iff t\leq \sqrt{n}\left(\frac{1-\omega}{\widetilde{K}^{2}}-C'\sqrt{\frac{D_m}{n}}\right).
    \end{aligned}
\end{equation*}
We get that for $n$ large enough $\frac{1}{\ln n}<\frac{1-\omega}{2C'\widetilde{K}^{2}}$, therefore in this setting $D_m< \frac{n}{\ln^{2}n}<\left(\frac{1-\omega}{2C\widetilde{K}^{2}}\right)^{2}n$, which means that $C'\sqrt{\frac{D_m}{n}}<\frac{1-\omega}{2\widetilde{K}^{2}}$. Thus we can pick $$t=\frac{1}{2}\left(\frac{1-\omega}{\widetilde{K}^{2}}\right)\sqrt{n}.$$Then $\widetilde{\delta}=C'\sqrt{\frac{D_m}{n}}+\frac{t}{\sqrt{n}}\leq \left(\frac{1-\omega}{\widetilde{K}^{2}}\right)$. As $1-\omega < 1\leq \widetilde{K}^{2}$ we have that $\frac{1-\omega}{\widetilde{K}^{2}}<1$ and $\max(\widetilde{\delta},\widetilde{\delta}^{2})=\widetilde{\delta}$. Therefore,
\[\mathbb{P}\left(\|\Psi_{m}-I\|_{op}>1-\omega\right)\leq 2\exp\left(-\frac{(1-\omega)^{2}n}{4\widetilde{K}^{4}}\right).\]

\end{proof}
\begin{lemma}\label{lem:sub-gaussian}
    Under hypotheses \textbf{(H1)} and \textbf{(H2)} we have for all $i\in\{1,...,n\}$ that $U_{i}^{(m)}$ defined in (\ref{eq:U_im}) is sub-Gaussian with variance factor $K_{1}^{2}+1$.
\end{lemma}

\begin{proof}
Let $v\in\mathbb{R}^{D_m}$ such that $\|v\|_{2}=1$. By definition of the $U_{i}^{(m)}$ given in (\ref{eq:U_im}) we have that for all $t\in\mathbb{R}$
\begin{equation*}
    \begin{aligned}
        \mathbb{E}\left(e^{t\langle v,U_{i}^{(m)}\rangle_{2}}\right) &= \mathbb{E}\left(\exp{\left(t\sum_{j=1}^{D_m}v_jU_{i,j}^{(m)}\right)}\right) \\
        &=\mathbb{E}\left(\exp{\left(t\sum_{j=1}^{D_m}v_j\frac{\widetilde{x}_{i,j}}{\sqrt{\widetilde{\lambda}_j}}\right)}\right).
    \end{aligned}
\end{equation*}
As defined in equation (\ref{eq:eta_bar_x_bar}) we have that for all $j\in\{1,\ldots,D_{N_{n,p}}\}$ $\widetilde{x}_{i,j}=\overline{x}_{i,j}+\overline{\eta}_{i,j}$. Let us focus on $\mathbb{E}\left(\exp{\left(t\sum_{j=1}^{D_m}\overline{\eta}_{i,j}\frac{v_j}{\widetilde{\lambda}_j}\right)}\right)$. As we assumed that the $\eta_{i,h}$'s are independent, we have that for all $t\in\mathbb{R}$
\begin{equation*}
    \begin{aligned}
    \mathbb{E}\left(\exp{\left(t\sum_{j=1}^{D_m}\overline{\eta}_{i,j}\frac{v_j}{\widetilde{\lambda}_j}\right)}\right) &= \mathbb{E}\left(\exp{\left(\frac{t}{p}\sum_{h=0}^{p-1}\eta_{i,h}\sum_{j=1}^{D_m}\frac{v_j\phi_j(t_h)}{\widetilde{\lambda}_j}\right)}\right) \\
        &= \prod_{h=0}^{p-1}\mathbb{E}\left(\exp{\left(\frac{t}{p}\eta_{i,h}\sum_{j=1}^{D_m}\frac{v_j\phi_j(t_h)}{\sqrt{\widetilde{\lambda}_j}}\right)}\right),
    \end{aligned}
\end{equation*}
and as we assumed (\textbf{H2}) we have that for all $t'\in\mathbb{R}$, $i\in\{1,\ldots,n\}$ and $h\in\{0,\ldots,p-1\}$
    \begin{equation*}
\mathbb{E}\left(\exp{(t'\eta_{i,h})}\right) \leq \exp{\left(\frac{\tau^{2}t'^{2}}{2}\right)}.
    \end{equation*}
Hence,
\begin{equation*}
    \begin{aligned}
        \mathbb{E}\left(\exp{\left(t\sum_{j=1}^{D_m}\overline{\eta}_{i,j}\frac{v_j}{\widetilde{\lambda}_j}\right)}\right) &\leq \prod_{h=0}^{p-1}\exp{\left(\frac{\tau^{2}t^{2}}{2p^{2}}\left(\sum_{j=1}^{D_m}\frac{v_j\phi_j(t_h)}{\sqrt{\widetilde{\lambda}_j}}\right)^{2}\right)}\\
        &=\exp{\left(\frac{\tau^{2}t^{2}}{2p^{2}}\sum_{h=0}^{p-1}\left(\sum_{j=1}^{D_m}\frac{v_j\phi_j(t_h)}{\sqrt{\widetilde{\lambda}_j}}\right)^{2}\right)}. 
    \end{aligned}
\end{equation*}
Using Lemma \ref{lem:A1} we get that
\begin{equation*}
    \begin{aligned}
        \frac{1}{p^{2}}\sum_{h=0}^{p-1}\left(\sum_{j=1}^{D_m}\frac{v_j\phi_j(t_h)}{\sqrt{\widetilde{\lambda}_j}}\right)^{2} &= \frac{1}{p^{2}}\sum_{h=0}^{p-1}\sum_{j,k=1}^{D_m}\frac{v_j}{\sqrt{\widetilde{\lambda}_j}}\frac{v_k}{\sqrt{\widetilde{\lambda}_k}}\phi_j(t_h)\phi_k(t_h) \\
        &= \frac{1}{p}\sum_{j,k=1}^{D_m}\frac{v_j}{\sqrt{\widetilde{\lambda}_j}}\frac{v_k}{\sqrt{\widetilde{\lambda}_k}}\langle\phi_j,\phi_k\rangle_p \\
        &= \frac{1}{p}\sum_{j=1}^{D_m}\frac{v_{j}^{2}}{\widetilde{\lambda}_j}.
    \end{aligned}
\end{equation*}
Hence,
\begin{equation*}
    \mathbb{E}\left(\exp{\left(t\sum_{j=1}^{D_m}\overline{\eta}_{i,j}\frac{v_j}{\widetilde{\lambda}_j}\right)}\right)\leq \exp{\left(\frac{\tau^{2}t^{2}}{2p}\sum_{j=1}^{D_m}\frac{v_{j}^{2}}{\widetilde{\lambda}_j}\right)}.
\end{equation*}
As $\frac{\tau^{2}}{p}\leq \widetilde{\lambda}_j$ for all $j\in\{1,...,D_{N_{n,p}}\}$ we have that $\frac{\tau^{2}}{p\widetilde{\lambda}_j}\leq 1$, therefore
\begin{equation*}
    \mathbb{E}\left(\exp{\left(t\sum_{j=1}^{D_m}\overline{\eta}_{i,j}\frac{v_j}{\widetilde{\lambda}_j}\right)}\right)\leq \exp{\left(\frac{t^{2}}{2}\|v\|_{2}^{2}\right)} = \exp{\left(\frac{t^{2}}{2}\right)}.
\end{equation*}
Now let us focus on $\mathbb{E}\left(\exp{\left(t\sum_{j=1}^{D_m}\overline{x}_{i,j}\frac{v_j}{\sqrt{\widetilde{\lambda}_j}}\right)}\right)$. Using the definition of $\overline{x}_{i,j}$ given in (\ref{eq:eta_bar_x_bar}) and the fact that for all $i\in\{1,...,n\}$ and for all  $s\in[0,1]\quad X_i(s)=\sum_{j\geq 1}x_{i,j}\phi_j(s)$ we have,
\begin{equation}\label{eq:compute_x_bar}
    \begin{aligned}
        \sum_{j=1}^{D_m}\overline{x}_{i,j}\frac{v_j}{\sqrt{\widetilde{\lambda}_j}} &= \sum_{j=1}^{D_m}\frac{v_j}{\sqrt{\widetilde{\lambda}_j}}\left[\frac{1}{p}\sum_{h=0}^{p-1}\left(\sum_{l\geq 1}x_{i,l}\phi_l(t_h)\right)\phi_{j}(t_h)\right] \\
        &= \sum_{l\geq 1}x_{i,l}\sum_{j=1}^{D_m}\frac{v_j}{\sqrt{\widetilde{\lambda}_j}}\langle\phi_j, \phi_l\rangle_p.
    \end{aligned}
\end{equation}
Let us consider $c_l:=\sum_{j=1}^{D_m}\frac{v_j}{\sqrt{\widetilde{\lambda}_j}}\langle\phi_j, \phi_l\rangle_p$. As we assumed that (\textbf{H1}) holds, we get that
\begin{equation*}
    \begin{aligned}
        \mathbb{E}\left(\exp{\left(t\sum_{j=1}^{D_m}\overline{x}_{i,j}\frac{v_j}{\sqrt{\widetilde{\lambda}_j}}\right)}\right) &= \mathbb{E}\left(\exp{\left(t\sum_{l\geq 1}x_{i,l}c_{l}\right)}\right) \\
        &\leq \exp{\left(\frac{t^{2}K_{1}^{2}}{2}\left(\sum_{l\geq 1}\lambda_{l}c_{l}^{2}\right)\right)}.
    \end{aligned}
\end{equation*}
Using Fubini's theorem we get that
\begin{equation*}
    \begin{aligned}
        \mathbb{E}\left(\left[\sum_{l\geq 1}x_{i,l}c_{l}\right]^{2}\right) &= \sum_{l,r\geq 1}\mathbb{E}(x_{i,l}x_{i,r})c_{l}c_{r} \\
        &= \sum_{l\geq 1}\lambda_{l}c_{l}^{2},
    \end{aligned}
\end{equation*}
where we used the fact that $\mathbb{E}(x_{i,l}x_{i,r})=\lambda_l\delta_{lr}$ for all $l,r\geq 1$ (see Equation \eqref{eq:x_uncorrelated}) to get from the last equality. On an other hand, the same computations as in \eqref{eq:compute_x_bar} give that 
\begin{equation*}
    \begin{aligned}
        \sum_{j=1}^{D_m}\frac{v_j}{\sqrt{\widetilde{\lambda}_j}}\overline{x}_{i,j} = \sum_{l\geq 1}x_{i,l}d_{l}.
    \end{aligned}
\end{equation*}
Therefore,
\begin{equation*}
    \begin{aligned}
        \sum_{l\geq 1}\lambda_{l}c_{l}^{2} &= \mathbb{E}\left(\left[\sum_{l\geq 1}x_{i,l}c_{l}\right]^{2}\right) \\
        &= \mathbb{E}\left(\left[\sum_{j=1}^{D_m}\frac{v_j}{\sqrt{\widetilde{\lambda}_j}}\overline{x}_{i,j}\right]^{2}\right) \\
        &=\sum_{j,k=1}^{D_m}\frac{v_j}{\sqrt{\widetilde{\lambda}_j}}\frac{v_k}{\sqrt{\widetilde{\lambda}_k}}\mathbb{E}(\overline{x}_{i,j}\overline{x}_{i,k}).
    \end{aligned}
\end{equation*}
As for all $i\in\{1,\ldots,n\}$ and $j\in\{1,\ldots,D_{N_{n,p}}\}$, $\overline{x}_{i,j}=\widetilde{x}_{i,j}-\overline{\eta}_{i,j}$, that the $\eta_{i,h}$'s are independent of everything else and centered, we get
\[\mathbb{E}(\overline{x}_{i,j}\overline{x}_{i,k}) = \mathbb{E}(\widetilde{x}_{i,j}\widetilde{x}_{i,k})-\mathbb{E}(\overline{\eta}_{i,j}\overline{\eta}_{i,k}) = \delta_{jk}\left(\widetilde{\lambda}_j-\frac{\tau^{2}}{p}\right).\]
The last inequality can be deduced from the reasoning in the proof of Lemma \ref{lem:A2}. We then have
\begin{equation*}
    \begin{aligned}
        \sum_{l\geq 1}\lambda_{l}c_{l}^{2} &= \sum_{j=1}^{D_m}\frac{v_{j}^{2}}{\widetilde{\lambda}_j}\left(\widetilde{\lambda}_{j}-\frac{\tau^{2}}{p}\right) \\
        &\leq \sum_{j=1}^{D_m}v_{j}^{2} \\
        &= \|v\|_{2}^{2}\\
        &= 1.
    \end{aligned}
\end{equation*}
Therefore,
\begin{equation*}
    \begin{aligned}
        \mathbb{E}\left(\exp{\left(t\sum_{j=1}^{D_m}\overline{x}_{i,j}\frac{v_j}{\sqrt{\widetilde{\lambda}_j}}\right)}\right) \leq \exp{\left(\frac{t^{2}K_{1}^{2}}{2}\right)}.
    \end{aligned}
\end{equation*}
Using the fact that the $\eta_{i,h}$'s are independent of the $X_{i}$'s, we get that for all $i\in\{1,\ldots,n\}$
\begin{equation*}
    \begin{aligned}
        \mathbb{E}\left(\exp{\left(t\langle v,U_{i}^{(m)}\rangle_{2}\right)}\right) &= \mathbb{E}\left(\exp{\left(t\sum_{j=1}^{D_m}\overline{x}_{i,j}\frac{v_j}{\sqrt{\widetilde{\lambda}_j}}\right)}\right)\mathbb{E}\left(\exp{\left(t\sum_{j=1}^{D_m}\overline{\eta}_{i,j}\frac{v_j}{\sqrt{\widetilde{\lambda}_j}}\right)}\right) \\
        &\leq \exp{\left(\frac{t^{2}(K_{1}^{2}+1)}{2}\right)}.
    \end{aligned}
\end{equation*}
Therefore, for all $i\in\{1,\ldots,n\}$ $U_{i}^{(m)}$ is sub-gaussian with variance factor $K_{1}^{2}+1$.

\end{proof}

\begin{lemma}\label{lem:overline_lambda}
Let us assume that (\textbf{H3}) holds. Recall that for all $j\in\{1,\ldots,D_{N_{n,p}}\}$,
\[\overline{\lambda}_{j} = 
\begin{cases}
    2 \displaystyle\sum_{q \geq 1} \lambda_{2qp}, & \text{if } j = 1, \\[1.5ex]
    \displaystyle\sum_{q \geq 1} \bigl( \lambda_{j+2qp} + \lambda_{2qp-j+2} \bigr), & \text{if } j > 1 \text{ is odd}, \\[1.5ex]
    \displaystyle\sum_{q \geq 1} \bigl( \lambda_{j+2qp} + \lambda_{2qp-j} \bigr), & \text{if } j \text{ is even}.
\end{cases}\]
Then there exist a positive constant $C(a,c')$ such that for every $p\geq2$ and every $j\in\{1,\dots,p-1\}$,
\[\overline{\lambda}_j\leq C(a,c')p^{-2a}.\]
\end{lemma}

\begin{proof}

We first note the following inequalities valid for every $q\ge1$ and $1\le j<p$ :
\[2qp \leq 2qp + j \leq (2q+1)p,\]
\[(2q-1)p \leq 2qp - j \leq 2qp,\]
\[(2q-1)p+2 \leq 2qp - j + 2 \leq (2q+1)p.\]
Since $(\lambda_k)$ is non-increasing, if $u\leq v$ then $\lambda_u \geq \lambda_v$.
We repeatedly use this together with (\textbf{H3})
\textbf{Case 1. $j$ even.}
We have
\[\overline{\lambda}_j=\sum_{q\geq1}\big(\lambda_{2qp+j}+\lambda_{2qp-j}\big).\]
Since $2qp+j\geq2qp$ and $2qp-j\ge(2q-1)p$,
\[\lambda_{2qp+j}\le c'(2qp)^{-2a},\qquad\lambda_{2qp-j}\le c'((2q-1)p)^{-2a}.\]
Hence
\[\overline{\lambda}_j\leq c'p^{-2a}\sum_{q\geq1}\big((2q)^{-2a}+(2q-1)^{-2a}\big)=c'\zeta(2a)p^{-2a},\]
where $\zeta$ denotes the Riemann function. Thus
\[\overline{\lambda}_j\leq C(a,c')p^{-2a}.\]
\textbf{Case 2. $j$ odd and $j>1$.}
We have
\[\overline{\lambda}_j=\sum_{q\geq1}\big(\lambda_{2qp+j}+\lambda_{2qp-j+2}\big).\]
Since $2qp+j\ge2qp$ and $2qp-j+2\geq(2q-1)p$,
\[\overline{\lambda}_j\leq c'\zeta(2a)\,p^{-2a}.\]
we conclude that $\overline{\lambda}_j \leq C(a,c')p^{-2a}$.

\medskip\noindent\textbf{Case 3. $j=1$.}
We have
\[\overline{\lambda}_1=2\sum_{q\geq1}\lambda_{2qp}.\]
Using assumption (\textbf{H3}) we get,
\[\overline{\lambda}_1\leq2c'\sum_{q\geq 1}(2qp)^{-2a}.\]
But
\[\sum_{q\geq1}(2qp)^{-2a}=p^{-2a}2^{-2a}\zeta(2a),\]
therefore
\[\overline{\lambda}_1\leq C(a,c')p^{-2a}.\]
\medskip\noindent Combining the three cases yields the result.
\end{proof}

\begin{lemma}\label{lem:C_beta}
    Suppose that $D_{N_{n,p}} < p$, $\mathbb{E} \langle \beta, X_1 \rangle_{L^2}^4 \big)<\infty$ and that (\textbf{H3}) holds. Then,
    \begin{equation}
        \frac{\widetilde{C}_{\beta}}{n} \leq C(a,c',\tau)C_{\beta}\left(\frac{1}{n}+\frac{1}{np}\right) + \frac{2\|\beta\|_{L^{2}}^{2}\Big( \mathbb{E} \|\widetilde{X}_1 - X_1\|_{L^2}^4 \Big)^{1/2}}{n},
    \end{equation}
    where $\widetilde{C}_{\beta}$ is defined in (\ref{eq:widetilde_c_beta}) and $C_\beta$ in (\ref{eq:c_beta}).
\end{lemma}

\begin{proof}
First we can rewrite $\langle \beta, \widetilde{X}_1 \rangle_{L^2}$ as
\[\langle \beta, \widetilde{X}_1 \rangle_{L^2} =\langle \beta, X_1 \rangle_{L^2} +\langle \beta, \widetilde{X}_1 - X_1\rangle_{L^2}.\]By Minkowski's inequality in $L^4$ we have,
\[\Big( \mathbb{E} \langle \beta, \widetilde{X}_1 \rangle_{L^2}^4 \Big)^{1/2}\leq
\left(\big( \mathbb{E} \langle \beta, X_1 \rangle_{L^2}^4 \big)^{1/4}+\big( \mathbb{E} |\langle \beta, \widetilde{X}_1 - X_1 \rangle_{L^2}|^4 \big)^{1/4}\right)^2.\]
Cauchy--Schwarz inequality gives,
\[|\langle \beta, \widetilde{X}_1 - X_1 \rangle_{L^2}|\leq\|\beta\|_{L^2}\|\widetilde{X}_1- X_1\|_{L^2},\]
Therefore,
\[\big( \mathbb{E} |\langle \beta, \widetilde{X}_1 - X_1 \rangle_{L^2}|^4 \big)^{1/4}
\leq\|\beta\|_{L^2}\big( \mathbb{E} \|\widetilde{X}_1 - X_1\|_{L^2}^4 \big)^{1/4}.\]
Hence,
\[\Big( \mathbb{E} \langle \beta, \widetilde{X}_1 \rangle_{L^2}^4 \Big)^{1/2}\leq\left(\big( \mathbb{E} \langle \beta, X_1 \rangle_{L^2}^4 \big)^{1/4}+\|\beta\|_{L^2}\big( \mathbb{E} \|\widetilde{X}_1 - X_1\|_{L^2}^4 \big)^{1/4}\right)^2.\]
Using now $(u+v)^2 \le 2u^2 + 2v^2$, $u,v\in\mathbb{R}$, we obtain the simpler bound
\[\Big( \mathbb{E} \langle \beta, \widetilde{X}_1 \rangle_{L^2}^4 \Big)^{1/2}
\leq2 \Big( \mathbb{E} \langle \beta, X_1 \rangle_{L^2}^4 \Big)^{1/2}+2 \|\beta\|_{L^2}^2
\Big( \mathbb{E} \|\widetilde{X}_1 - X_1\|_{L^2}^4 \Big)^{1/2}.\]
Using the definition of $\|.\|_{\widetilde{\Gamma}}$ given in (\ref{eq:norme_tilde}), Lemma \ref{lem:eigenvalues_gamma_tilde} and Lemma \ref{lem:overline_lambda} we get
\begin{equation}\label{eq:beta_tilde_comp}
    \begin{aligned}  \|\beta\|_{\widetilde{\Gamma}}^{2}&= \langle \widetilde{\Gamma}\beta, \beta\rangle_{L^{2}} 
    = \langle \sum_{j\geq 1}\beta_{j}\widetilde{\Gamma}\phi_j, \sum_{k\geq 1}\beta_k\phi_k\rangle_{L^{2}} 
    = \sum_{j,k\geq 1}\beta_j\beta_k\langle\widetilde{\Gamma}\phi_j, \phi_k\rangle_{L^{2}}
    =\sum_{j=1}^{D_{N_{n,p}}}\widetilde{\lambda}_{j}\beta_{j}^{2}
    \end{aligned}
\end{equation}
\begin{equation*}
    \leq \max_{1\leq j\leq D_{N_{n,p}}}\widetilde{\lambda}_{j}\|\beta\|_{L^{2}}^{2} \leq \|\beta\|_{L^{2}}^{2}\left(\frac{\tau^{2}}{p}+\frac{C_1}{p^{2a}}+c'\right)\leq C(a,c',\tau)\|\beta\|_{L^{2}}^{2}\left(\frac{1}{p}+1\right).
\end{equation*}
Hence,
\begin{equation*}
    \begin{aligned}
        \frac{\widetilde{C}_{\beta}}{n} &\leq \frac{2\Big( \mathbb{E} \langle \beta, X_1 \rangle_{L^2}^4 \Big)^{1/2}}{n} + \frac{2\|\beta\|_{L^{2}}^{2}\Big( \mathbb{E} \|\widetilde{X}_1 - X_1\|_{L^2}^4 \Big)^{1/2}}{n}+\frac{C(a,c',\tau)\|\beta\|_{L^{2}}^{2}}{n}\left(\frac{1}{p}+1\right)+\frac{1}{n}.
    \end{aligned}
\end{equation*}
Therefore we have that 
\begin{equation*}
    \frac{\widetilde{C}_{\beta}}{n} \leq C(a,c',\tau)C_{\beta}\left(\frac{1}{n}+\frac{1}{np}\right) + \frac{2\|\beta\|_{L^{2}}^{2}\Big( \mathbb{E} \|\widetilde{X}_1 - X_1\|_{L^2}^4 \Big)^{1/2}}{n}.
\end{equation*}

\end{proof}

\subsection{Proof of Theorem \ref{thm:3.1}}

\begin{proof}
    
We have that 
\begin{equation}\label{eq:change_omega}
    \mathbb{E}(\|\widetilde{\beta}-\beta\|_{\widetilde{\Gamma}}^{2})= \mathbb{E}(\|\widetilde{\beta}-\beta\|_{\widetilde{\Gamma}}^{2}\mathds{1}_{\Omega^{(n)}}) + \mathbb{E}(\|\widetilde{\beta}-\beta\|_{\widetilde{\Gamma}}^{2}\mathds{1}_{\Omega^{(n)c}}),
\end{equation}
where $\Omega^{(n)}$ is such that 
\begin{equation}\label{eq:Omega_n}
    \Omega^{(n)}:=\bigcap_{m\in\mathcal{M}_{n}}\Omega_{m},
\end{equation}
with
\begin{equation}
    \Omega_m :=\left\{\sup_{\substack{f\in S_m\\ f\neq 0}}\left|
\frac{\|f\|_{\widetilde{\Gamma}_n}^{2}}{\|f\|_{\widetilde{\Gamma}}^{2}}-1\right|\leq \widetilde{\omega}\right\},
\end{equation}
where $0<\widetilde{\omega}<1$. The second term of the sum in Equation (\ref{eq:change_omega}) is controlled by Lemma \ref{lem:c_beta}. Let us focus now on the first term of the sum. Let $\beta^{(m)}$ be the projection of $\beta$ onto $S_m$ with respect to the semi-norm $\langle\cdot,\cdot\rangle_{\Gamma}$, for all $m\in\mathcal{M}_{n,p}$. Then,
\begin{equation*}
    \begin{aligned}
        \|\widetilde{\beta}-\beta\|_{\widetilde{\Gamma}}^{2}\mathds{1}_{\Omega^{(n)}} &= \|\widetilde{\beta}-\beta^{(m)}+\beta^{(m)} - \beta\|_{\widetilde{\Gamma}}^{2}\mathds{1}_{\Omega^{(n)}} \\
        &\leq 2\|\widetilde{\beta}-\beta^{(m)}\|_{\widetilde{\Gamma}}^{2}\mathds{1}_{\Omega^{(n)}} + 2\|\beta^{(m)} - \beta\|_{\widetilde{\Gamma}}^{2}\mathds{1}_{\Omega^{(n)}} \\
        &\leq \frac{2}{1-\widetilde{\omega}}\|\widetilde{\beta}-\beta^{(m)}\|_{\widetilde{\Gamma}_{n}}^{2} + 2\|\beta^{(m)}-\beta\|_{\widetilde{\Gamma}}^{2} \\
        &\leq \frac{4}{1-\widetilde{\omega}} \|\widetilde{\beta}-\beta\|_{\widetilde{\Gamma}_{n}}^{2} + \frac{4}{1-\widetilde{\omega}} \|\beta-\beta^{(m)}\|_{\widetilde{\Gamma}_{n}}^{2} + 2\|\beta^{(m)}-\beta\|_{\widetilde{\Gamma}}^{2},
    \end{aligned}
\end{equation*}
where we used the definition of $\Omega^{(n)}$ to go from the second to third line. Taking the expectation, using the results of Proposition \ref{thm:3.1} we get that
\begin{equation*}
\begin{aligned}
    \mathbb{E}(\|\widetilde{\beta}-\beta\|_{\widetilde{\Gamma}}^{2}) 
    &\leq C_{1} \biggl[ \min_{m \in \mathcal{M}_{n,p}} \Bigl( \mathbb{E}(\|\beta^{(m)}-\beta\|_{\widetilde{\Gamma}}^{2}) + \text{pen}(m) \Bigr) \\
    &\quad + \|\beta\|_{L^{2}}^{2} \mathbb{E}(\|\widetilde{X}_{1}-X_{1}\|_{L^{2}}^{2}) \\
    &\quad + \frac{\|\beta\|_{L^{2}}^{2}}{n} \left( \mathbb{E} \|\widetilde{X}_1 - X_1\|_{L^2}^4 \right)^{1/2} + C_{\beta}\left(\frac{1}{n}+\frac{1}{np}\right) \biggr],
\end{aligned}
\end{equation*}
where $C_{\beta}=\mathbb{E}(\langle\beta,X_1\rangle^{4})^{1/2}+\|\beta\|_{L^{2}}^{2}+1$ and $C_1>0$ depends on $\theta, K, l, \tau_{l}, \delta, \sigma^{2}$.

\end{proof}

\begin{lemma}\label{lem:c_beta}
    Under the assumptions that $D_{N_{n,p}} < \min{\left(\frac{n}{\ln^{2}n},p\right)}$ and that (\textbf{H1}), (\textbf{H2}) and (\textbf{H3}) hold, there exist a constant $C_{0}'$ depending on $\rho$, $K$, $\sigma$, $a$ and $c'$ such that 
    \begin{equation*}
        \mathbb{E}\left(\|\widetilde{\beta}-\beta\|_{\widetilde{\Gamma}}^{2}\mathds{1}_{\Omega_{n}^{c}}\right) \leq C_{0}'\left(\frac{1}{n}+\frac{1}{np}\right)\left(\mathbb{E}(\langle\beta,X_1\rangle^{4})^{1/2}+\|\beta\|_{L^{2}}^{2}+1\right).
    \end{equation*}
\end{lemma}
\begin{proof}
The proof of this Lemma was developped by Brunel and Roche \cite{Brunel2015penalized} in their Lemma 5. Here we just adapt it to our setting.
\medskip
Using the triangular inequality and the definition of $\widetilde{\beta}$ we have,
\begin{equation}\label{eq:beta_temp0}
        \begin{aligned}
            \mathbb{E}\left(\|\widetilde{\beta}-\beta\|_{\widetilde{\Gamma}}^{2}\mathds{1}_{\Omega^{(n)c}}\right) &\leq 2\mathbb{E}\left((\|\widetilde{\beta}\|_{\widetilde{\Gamma}}^{2} + \|\beta\|_{\widetilde{\Gamma}}^{2})\mathds{1}_{\Omega^{(n)c}}\right) \\
            &= 2\mathbb{E}\left(\|\widehat{\beta}_{\widehat{m}}\|_{\widetilde{\Gamma}}^{2}\mathds{1}_{\Omega^{(n)c}\cap\overline{G}}\right) + 2 \|\beta\|_{\widetilde{\Gamma}}^{2}\mathbb{P}(\Omega^{(n)c}).
        \end{aligned}
\end{equation}
Lemma \ref{lem:A4} and the definition of $\overline{G}$ allow us to derive the following inclusions :
\begin{equation}\label{eq:inclusion}
    \overline{G}\subset \{\widehat{\lambda}_{N_{n,p}}^{(p)}>s_{n}\}\subset \left\{ \widehat{\mu}_{N_{n,p}} > \frac{s_{n}}{\rho(\widetilde{\Gamma})} \right\},
\end{equation}
where $\rho$ denotes the spectral radius of the operator. By using the same reasoning as detailed in the proof of Lemma \ref{lem:help_3.2} bellow, we have
\begin{equation}\label{eq:rapport1}
    \inf_{f\in S_{N_{n,p}}} \frac{\|f\|_{\widetilde{\Gamma}_{n}}^{2}}{\|f\|_{\widetilde{\Gamma}}^{2}} = \inf_{b\neq 0, \|b\|=1}b^{T}\Psi_{N_{n,p}}b.
\end{equation}
If we assume that $\widehat{\mu}_{N_{n,p}}>0$, $\Psi_{N_{n,p}}$ is a symmetric matrix which is positive and there exists an orthogonal matrix $U$ such that $U^{T}\Psi_{N_{n,p}}U$ is a diagonal matrix whose diagonal entries are the eigenvalues of $\Psi_{N_{n,p}}$. Therefore,
\begin{equation}\label{eq:rapport2}
    \widehat{\mu}_{N_{n,p}}=\inf_{b\neq 0, \|b\|=1}b^{T}U^{T}\Psi_{N_{n,p}}Ub=\inf_{c\neq 0, \|c\|=1}c^{T}\Psi_{N_{n,p}}c.
\end{equation}
Combining the results from Equations (\ref{eq:inclusion}), (\ref{eq:rapport1}) and (\ref{eq:rapport2}), we have for all $f\in S_{N_{n,p}}$ and $f\neq 0$,
\[\|f\|_{\widetilde{\Gamma}}^{2}<\frac{\rho(\widetilde{\Gamma})\|f\|_{\widetilde{\Gamma}_{n}}^{2}}{s_{n}}.\]
Taking $f=\widehat{\beta}_{\widehat{m}}$ we get
\begin{equation}\label{eq:beta_temp1}
\mathbb{E}\left(\|\widehat{\beta}_{\widehat{m}}\|_{\widetilde{\Gamma}}^{2}\mathds{1}_{\Omega^{(n)c}\cap\overline{G}}\right) \leq \frac{\rho(\widetilde{\Gamma})}{s_{n}
}\mathbb{E}\left(\|\widehat{\beta}_{\widehat{m}}\|_{\widetilde{\Gamma}_{n}}^{2}\mathds{1}_{\Omega^{(n)c}\cap\overline{G}}\right).
\end{equation}
As $\widehat{\beta}_{\widehat{m}}$ is a mean-square-type estimator, the vector $\left(\langle \widehat{\beta}_{\widehat{m}},\widetilde{X}_1\rangle_{L^{2}},\ldots,\langle \widehat{\beta}_{\widehat{m}},\widetilde{X}_n\rangle_{L^{2}}\right)^{T}$ can be seen as the orthogonal projection with respect to the euclidean scalar product on $\mathbb{R}^{n}$ of the vector $\left(Y_1,\ldots,Y_n\right)^{T}$ on the subspace $\left\{ (\langle f,\widetilde{X}_1\rangle_{L^{2}},\ldots,\langle f,\widetilde{X}_n\rangle_{L^{2}})^{T},f\in S_{\widehat{m}} \right\}$. Therefore,
\[\sum_{i=1}^{n}\langle \widehat{\beta}_{\widehat{m}},\widetilde{X}_{i}\rangle_{L^{2}}^{2}\leq \sum_{i=1}^{n}Y_{i}^{2},\]and
\[n\|\widehat{\beta}_{\widehat{m}}\|_{\widetilde{\Gamma}_{n}}^{2}\leq \sum_{i=1}^{n}Y_{i}^{2}.\]Replacing $Y_i$ by its definition, $Y_i=\langle\beta, X_i\rangle_{L^{2}}+\epsilon_i$ we get
\begin{equation}\label{eq:beta_temp2}
    \|\widehat{\beta}_{\widehat{m}}\|_{\widetilde{\Gamma}_{n}}^{2} \leq 2\|\beta\|_{\Gamma_n}^{2} + \frac{2}{n}\sum_{i=1}^{n}\epsilon_{i}^{2}.
\end{equation}
Using Equations (\ref{eq:beta_temp1}) and (\ref{eq:beta_temp2}) we get
\begin{equation}\label{eq:beta_temp3}
    \mathbb{E}\left(\|\widehat{\beta}_{\widehat{m}}\|_{\widetilde{\Gamma}}^{2}\mathds{1}_{\Omega^{(n)c}\cap\overline{G}}\right) \leq \frac{2\rho(\widetilde{\Gamma})}{s_{n}}\mathbb{E}\left(\bigl(\|\beta\|_{\Gamma_n}^{2}+\frac{1}{n}\sum_{i=1}^{n}\epsilon_{i}^{2}\bigr)\mathds{1}_{\Omega^{(n)c}}\right).
\end{equation}
As the $\epsilon_i$'s are independent of the $\widetilde{X}_i$'s and that the set $\Omega_{n}^{c}$ depends only on the $\widetilde{X}_i$'s we have
\begin{equation}\label{eq:beta_temp4bis}
    \mathbb{E}\left(\mathds{1}_{\Omega^{(n)c}}\frac{1}{n}\sum_{i=1}^{n}\epsilon_{i}^{2}\right)=\sigma^{2}\mathbb{P}(\Omega^{(n)c}).
\end{equation}
Cauchy--Schwarz inequality gives,
\[\mathbb{E}\left(\|\beta\|_{\Gamma_{n}}^{2}\mathds{1}_{\Omega^{(n)c}}\right)\leq \mathbb{E}\left(\|\beta\|_{\Gamma_{n}}^{4}\right)^{1/2}\mathbb{P}(\Omega^{(n)c})^{1/2}.\]
By definition of the norm $\|\cdot\|_{\Gamma_n}$ we have,
\begin{equation*}
    \begin{aligned}
        \mathbb{E}\left(\|\beta\|_{\Gamma_{n}}^{4}\right) &= \mathbb{E}\left(\left[\frac{1}{n}\sum_{i=1}^{n}\langle\beta,X_{i}\rangle_{L^{2}}^{2}\right]^{2}\right) \\
        &= \mathbb{E}\left(\frac{1}{n^{2}}\sum_{i=1}^{n}\langle\beta,X_{i}\rangle_{L^{2}}^{4}\right) + \mathbb{E}\left(\frac{2}{n^{2}}\sum_{1\leq i<i'\leq n}\langle\beta,X_i\rangle_{L^{2}}^{2}\langle\beta,X_{i'}\rangle_{L^{2}}^{2}\right) \\
        &= \frac{1}{n}\mathbb{E}\left(\langle\beta,X_1\rangle_{L^{2}}^{4}\right) + \frac{n-1}{n}\|\beta\|_{\Gamma}^{4}.
    \end{aligned}
\end{equation*}
Hence,
\begin{equation*}
    \mathbb{E}\left(\|\beta\|_{\Gamma_{n}}^{2}\mathds{1}_{\Omega^{(n)c}}\right)\leq \left(\frac{1}{n}\mathbb{E}\left(\langle\beta,X_1\rangle_{L^{2}}^{4}\right) + \frac{n-1}{n}\|\beta\|_{\Gamma}^{4}\right)^{1/2}\mathbb{P}(\Omega^{(n)c})^{1/2}.
\end{equation*}
Using the fact that for all $x,y\geq 0$ $\sqrt{x+y}\leq \sqrt{x}+\sqrt{y}$ we have 
\begin{equation}\label{eq:beta_temp4}
    \mathbb{E}\left(\|\beta\|_{\Gamma_{n}}^{2}\mathds{1}_{\Omega^{(n)c}}\right)\leq \left(\frac{1}{\sqrt{n}}\mathbb{E}\left(\langle\beta,X_1\rangle_{L^{2}}^{4}\right)^{1/2} + \|\beta\|_{\Gamma}^{2}\right)\mathbb{P}(\Omega^{(n)c})^{1/2}.
\end{equation}
Combining Equations (\ref{eq:beta_temp3}), (\ref{eq:beta_temp4bis}) and (\ref{eq:beta_temp4}),
\begin{equation*}
    \mathbb{E}\left(\|\widehat{\beta}_{\widehat{m}}\|_{\widetilde{\Gamma}}^{2}\mathds{1}_{\Omega^{(n)c}\cap\overline{G}}\right) \leq \frac{2\mathbb{P}(\Omega^{(n)c})^{1/2}}{s_{n}}\left(\rho(\widetilde{\Gamma})\left(\frac{1}{\sqrt{n}}\mathbb{E}\left(\langle\beta,X_1\rangle_{L^{2}}^{4}\right)^{1/2} + \|\beta\|_{\Gamma}^{2}+\sigma^{2}\right)\right).
\end{equation*}
As $\mathbb{P}(\Omega^{(n)c})\leq\mathbb{P}(\Omega^{(n)c})^{1/2} $ and $s_{n}\leq 2$ we get, combining the previous result and Equation (\ref{eq:beta_temp0}) :
\begin{equation}\label{eq:beta_temp5}
    \mathbb{E}\left(\|\widetilde{\beta}-\beta\|_{\widetilde{\Gamma}}^{2}\mathds{1}_{\Omega^{(n)c}}\right) \leq \frac{4\mathbb{P}(\Omega^{(n)c})^{1/2}}{s_{n}}\left(\rho(\widetilde{\Gamma})\left(\frac{1}{\sqrt{n}}\mathbb{E}\left(\langle\beta,X_1\rangle_{L^{2}}^{4}\right)^{1/2} + \|\beta\|_{\Gamma}^{2}+\sigma^{2}\right)+\|\beta\|_{\widetilde{\Gamma}}^{2}\right).
\end{equation}
Using Lemma \ref{lem:help_3.2} and the assumption that $D_{N_{n,p}}<n/\ln^{2}n$ we have 
\begin{equation}\label{eq:beta_temp6}
\begin{aligned}
    \frac{\mathbb{P}(\Omega^{(n)c})^{1/2}}{s_{n}} &\leq \frac{1/2}{1-1/\sqrt{\ln n}}(\ln n)^{-1}n^{\alpha+1/2}\exp{\left(-\frac{\widetilde{\omega}^{2}n}{8(K_{1}^{2}+1)^{2}}\right)} \\
    &\leq C\exp{\left(\left(\alpha+\frac{1}{2}\right)\ln n-C'n\right)} \\
    &\leq \frac{C''}{n},
\end{aligned}
\end{equation}
with $C,C',C''$ depending on $K_1$. Now, using Lemma \ref{lem:eigenvalues_gamma_tilde} and Lemma \ref{lem:overline_lambda} we have that
\[\rho(\widetilde{\Gamma})=\max_{1\leq j\leq D_{N_{n,p}}}\widetilde{\lambda}_{j}=\frac{\tau^{2}}{p} + \frac{C_{1}}{p^{2a}}+c'\leq C_{2}\left(\frac{1}{p}+1\right),\]
where $C_1,C_2>0$ depends on $c'$, $a$ and $\tau$ (for $C_2)$. We recall that here $\rho$ denotes the spectral radius. Let's now define the following quantity : $$A_n:=\frac{1}{\sqrt{n}}\mathbb{E}\left(\langle\beta,X_1\rangle_{L^{2}}^{4}\right)^{1/2} + \|\beta\|_{\Gamma}^{2}+\sigma^{2}.$$
Then using Equation (\ref{eq:beta_temp5}) and (\ref{eq:beta_temp6}) we get,
\[\mathbb{E}\left(\|\widetilde{\beta}-\beta\|_{\widetilde{\Gamma}}^{2}\mathds{1}_{\Omega^{(n)c}}\right) \leq C_{3}\left[\frac{1}{n}A_{n} + \frac{1}{np}A_{n} + \frac{1}{n}\|\beta\|_{\widetilde{\Gamma}}^{2}\right],\]
where $C_{3}$ depends on $K_1, a$ and $c'$. Using the definition of the $\widetilde{\lambda}_j$'s and the fact that the $\lambda_j$'s decrease in a polynomial way (see (\textbf{H3})), we get using the same computations as used in \eqref{eq:beta_tilde_comp}
\begin{equation*}
    \begin{aligned}  \|\beta\|_{\widetilde{\Gamma}}^{2}&=\sum_{j=1}^{D_{N_{n,p}}}\widetilde{\lambda}_{j}\beta_{j}^{2} \\
        &\leq \max_{1\leq j\leq D_{N_{n,p}}}\widetilde{\lambda}_{j}\|\beta\|_{L^{2}}^{2} \\
        &\leq \|\beta\|_{L^{2}}^{2}\left(\frac{\tau^{2}}{p}+\frac{C_1}{p^{2a}}+c'\right)\\
        &\leq C_{4}\|\beta\|_{L^{2}}^{2}\left(\frac{1}{p}+1\right).
    \end{aligned}
\end{equation*}
Hence,
\[\mathbb{E}\left(\|\widetilde{\beta}-\beta\|_{\widetilde{\Gamma}}^{2}\mathds{1}_{\Omega^{(n)c}}\right) \leq C_{5}\left[\frac{1}{n}A_{n} + \frac{1}{np}A_{n} +\frac{1}{n}\|\beta\|_{L^{2}}^{2}+\frac{1}{np}\|\beta\|_{L^{2}}^{2}\right],\]
where $C_5>0$ depends on $\widetilde{\omega}, K_1, a$ and $c'$. Therefore we obtain
\[\mathbb{E}\left(\|\widetilde{\beta}-\beta\|_{\widetilde{\Gamma}}^{2}\mathds{1}_{\Omega^{(n)c}}\right) \leq C_{6}\left(\frac{1}{n}+\frac{1}{np}\right)\left(\mathbb{E}\left(\langle\beta,X_1\rangle_{L^{2}}^{4}\right)^{1/2} + \|\beta\|_{\Gamma}^{2}+\|\beta\|_{L^{2}}^{2}+1\right),\]
where $C_6>0$ is a constant which depends on $K_1 , a$ and $c'$. Finally, as 
\begin{equation}\label{eq:beta_gamma_comp}
    \begin{aligned}
        \|\beta\|_{\Gamma}^{2} &= \langle\Gamma\beta, \beta\rangle_{L^{2}}=\langle\sum_{j\geq 1}\beta_j\Gamma\phi_j,\sum_{k\geq 1}\beta_k\phi_k\rangle_{L^{2}}= \sum_{j,k\geq 1}\beta_j\beta_k\langle\Gamma\phi_j,\phi_k\rangle_{L^{2}}=\sum_{j\geq 1}\lambda_j\beta_{j}^{2}
    \end{aligned}
\end{equation}
\begin{equation*}
    \leq \max_{j\geq 1}\lambda_j \|\beta\|_{L^{2}}^{2}\leq c'\|\beta\|_{L^{2}}^{2},
\end{equation*}
we get that 
\[\mathbb{E}\left(\|\widetilde{\beta}-\beta\|_{\widetilde{\Gamma}}^{2}\mathds{1}_{\Omega^{(n)c}}\right) \leq C_{0}'\left(\frac{1}{n}+\frac{1}{np}\right)\left(\mathbb{E}(\langle\beta,X_1\rangle^{4})^{1/2}+\|\beta\|_{L^{2}}^{2}+1\right),\]
where $C_0'>0$ depends on $K_1$, $\sigma$, $a$ and $c'$.
\end{proof}

\begin{lemma} \label{lem:help_3.2}
    Let us assume that (\textbf{H1}), (\textbf{H2}) hold and that $D_{N_{n,p}} < p$. Then,
\[\mathbb{P}\left(\Omega^{(n)c}\right)\leq D_{N_{n,p}}\exp\left(-\frac{n\widetilde{\omega}^{2}}{4(K_{1}^{2}+1)^{2}}
\right),\]
where $\Omega^{(n)c}$ is defined in (\ref{eq:change_omega}).
\end{lemma}

\begin{proof}
    Let $f\in S_{m}$. We can write $f$ as $f=\sum_{j=1}^{D_m}f_{j}\phi_{j}$. Therefore, by using the definition (\ref{eq:norme_emp_tilde}), we have
    \begin{equation*}
        \begin{aligned}
            \|f\|_{\widetilde{\Gamma}_{n}}^{2} &= \frac{1}{n}\sum_{i=1}^{n}\langle f, \widetilde{X}_{i}\rangle_{L^{2}}^{2} \\
            &= \frac{1}{n}\sum_{i=1}^{n}\langle \sum_{j=1}^{D_m}f_{j} \phi_{j}, \sum_{k=1}^{D_{N_{n,p}}}\widetilde{x}_{i,k}\phi_{k}\rangle_{L^{2}}^{2} \\
            &= \frac{1}{n}\sum_{i=1}^{n}\Bigl(\sum_{j=1}^{D_m}f_{j}\widetilde{x}_{i,j}\Bigr)^{2} \\
            &= \sum_{j,k=1}^{D_m}f_{j}f_{k}\Bigl(\frac{1}{n}\sum_{i=1}^{n}\widetilde{x}_{i,j}\widetilde{x}_{i,k}\Bigr)\\
            &= h^{T}\Phi_{m}^{(p)}h,
        \end{aligned}
    \end{equation*}
    where $h=(f_1, \ldots, f_{D_m})^{T}$, and $\Phi_{m}^{(p)}$ was defined in (\ref{eq:Phi_matrix}). We used the fact that the $\phi_j$'s are orthonormal to get the third line. Moreover, using the definition (\ref{eq:norme_tilde}) we get
    \begin{equation*}
        \begin{aligned}
            \|f\|_{\widetilde{\Gamma}}^{2} &= \mathbb{E}\Bigl(\langle \widetilde{X}, f\rangle_{L^{2}}^{2}\Bigr) \\
            &= \mathbb{E}\Bigl(\langle \sum_{k=1}^{D_{N_{n,p}}}\widetilde{x}_{k}\phi_{k}, \sum_{j=1}^{D_{m}}f_{j}\phi_{j}\rangle_{L^{2}}^{2}\Bigr) \\
            &= \sum_{j,k=1}^{D_m}f_{j}f_{k}\mathbb{E}(\widetilde{x}_{k}\widetilde{x}_{j}) \\
            &= \sum_{j=1}^{D_m}f_{j}^{2}\widetilde{\lambda}_{j}\\
            &= h^{T}\Lambda_{m}^{T}\Lambda_{m}h,
        \end{aligned}
    \end{equation*}
    where $\Lambda_{m}$ was defined by (\ref{eq:Lambda}). We get the last line by using Lemma \ref{lem:A2}. Using the matrix $\Psi_{m}$ defined in (\ref{eq:Psi}), we get the following equation,
    \[\|f\|_{\widetilde{\Gamma}_{n}}^{2}=h^{T}\Phi_{m}h=h^{T}\Lambda_{m}^{T}\Psi_{m}\Lambda_{m}h=b^{T}\Psi_{m}b,\]
    with $b=\Lambda_{m}h$, and
    \[\|f\|_{\widetilde{\Gamma}}^{2} = h^{T}\Lambda_{m}^{T}\Lambda_{m}h=b^{T}b.\]
    Therefore, 
    \begin{equation}\label{eq:36}
        \frac{\|f\|_{\widetilde{\Gamma}_{n}}^{2}}{\|f\|_{\widetilde{\Gamma}}^{2}}=\frac{b^{T}\Psi_{m}b}{b^{T}b},
    \end{equation}
    and 
    \begin{equation*}
            \sup_{\substack{f\in S_m\\ f\neq 0}}
\left|
\frac{\|f\|_{\widetilde{\Gamma}_n}^{2}}{\|f\|_{\widetilde{\Gamma}}^{2}}
-1
\right| = \sup_{b\neq 0}
\left|
\frac{b^{T}(\Psi_{m}-I)b}{b^{T}b} 
\right| = \|\Psi_{m}-I\|_{op},
    \end{equation*}
    as $\Psi_{m}$ is symmetric. Therefore we can rewrite $\Omega_{m}$ as \[\Omega_{m}=\left\{\|\Psi_{m}-I\|_{op}\leq \widetilde{\omega}\right\}.\]
    As discussed in the proof of Lemma \ref{lem:A5} we can apply after some computations Theorem 4.6.1 of Vershynin \cite{vershynin2026high} to get 
    \[\mathbb{P}\left(\Omega_{m}^{c}\right)=\mathbb{P}(\|\Psi_{m}-I\|_{op}>\widetilde{\omega})\leq 2\exp{\left(-\frac{n\widetilde{w}^{2}}{4(K_{1}^{2}+1)^{2}}\right)},\]and using the fact that $N_{n,p}\leq D_{N_{n,p}}/2$ we get that
    \[\mathbb{P}\left(\Omega^{(n)c}\right)=\sum_{m\in\mathcal{M}_n}\mathbb{P}(\Omega_{m}^{c})\leq D_{N_{n,p}}\exp\left(-\frac{n\widetilde{\omega}^{2}}{4(K_{1}^{2}+1)^{2}}\right).\]
\end{proof}

\subsection{Proof of Theorem \ref{thm:3.2}}
\begin{proof}
    Lemma \ref{lem:eigenvalues_gamma_tilde} gives us that the eigenvalues of $\widetilde{\Gamma}$ in $S_{N_{n,p}}$ are the $\widetilde{\lambda}_{j}$'s which are defined in (\ref{eq:tilde_lambda}). Now for all $j\in\{1,\ldots,D_{N_{n,p}}\}$ we have that $\lambda_{j}\leq \widetilde{\lambda}_{j}$, therefore for all $f\in S_{N_{n,p}}$,
    \[\|f\|_{\Gamma}\leq \|f\|_{\widetilde{\Gamma}}.\]
For all $m\in\mathcal{M}_{n,p}$ and considering $\beta^{(m)}$ the projection of $\beta$ on $S_m$ we get that,
\begin{equation*}
    \begin{aligned}
        \|\widetilde{\beta}-\beta\|_{\Gamma}^{2} &= \|\widetilde{\beta}-\beta^{(m)}+ \beta^{(m)}- \beta\|_{\Gamma}^{2} \\
        &\leq 2\|\widetilde{\beta}-\beta^{(m)}\|_{\Gamma}^{2} + 2\| \beta^{(m)}- \beta\|_{\Gamma}^{2} \\
        &\leq 2\|\widetilde{\beta}-\beta^{(m)}\|_{\widetilde{\Gamma}}^{2} + 2\| \beta^{(m)}- \beta\|_{\Gamma}^{2} \\
        &\leq 4\|\widetilde{\beta}-\beta\|_{\widetilde{\Gamma}}^{2} + 4\|\beta - \beta^{(m)}\|_{\widetilde{\Gamma}}^{2} + 2\| \beta^{(m)}- \beta\|_{\Gamma}^{2}.
    \end{aligned}
\end{equation*}
Taking the expectation in the equation just above and using Theorem \ref{thm:3.1}, we get that 
\begin{equation*}
\begin{split}
\mathbb{E}(\|\widetilde{\beta}-\beta\|_{\Gamma}^{2}) 
&\leq C_{2} \biggl[ \min_{m \in \mathcal{M}_{n,p}} \Bigl( \|\beta^{(m)}-\beta\|_{\Gamma}^{2} + \|\beta^{(m)}-\beta\|_{\widetilde{\Gamma}}^{2} + pen(m) \Bigr) \\
&\quad + \|\beta\|_{L^{2}}^{2} \mathbb{E}(\|\widetilde{X}-X\|_{L^{2}}^{2}) \\
&\quad + \frac{\|\beta\|_{L^{2}}^{2}}{n} \left( \mathbb{E} \|\widetilde{X} - X\|_{L^2}^4 \right)^{1/2} + C_{\beta}\left(\frac{1}{n}+\frac{1}{np}\right)\biggr],
\end{split}
\end{equation*}
To derive Equation (\ref{eq:second_oracle_3.2}) we combine this result with Lemma A.3 of Tsybakov \cite{tsybakov2009introduction}.

\end{proof}

\begin{lemma}\label{lem:second_help_oracle3.2}
    Let us suppose that (\textbf{H3}) hold, $D_{N_{n,p}} < p$ and that $\beta\in W^{\text{per}}(k,L)$. Then for all $m\in\mathcal{M}_{n,p}$, \[\|\beta^{(m)}-\beta\|_{\widetilde{\Gamma}}^{2}\leq C_{2}''\Bigl(\|\beta^{(m)}-\beta\|_{\Gamma}^{2}+\frac{1}{p}\Bigr),\]where $C_2''$ is a positive constant which depends on $L, k, c', a, \tau$.

\end{lemma}
\begin{proof}
We have, denoting as $\beta_j$ the coefficients of $\beta$ in the Fourier basis,
\begin{equation*}
    \begin{aligned}
        \|\beta - \beta^{(m)}\|_{\widetilde{\Gamma}}^{2} &= \langle \widetilde{\Gamma}(\beta-\beta^{(m)}),\beta-\beta^{(m)}\rangle_{L^{2}} \\
        &= \sum_{j=D_m+1}^{D_{N_{n,p}}}\widetilde{\lambda}_{j}\beta_{j}^{2},
    \end{aligned}
\end{equation*}
since Lemma \ref{lem:eigenvalues_gamma_tilde} gives us that the eigenvalues of $\widetilde{\Gamma}$ associated with the elements of the Fourier basis are the $\widetilde{\lambda}_{j}$ for $j\in\{1,\ldots,D_{N_{n,p}}\}$, and are 0 for $j>D_{N_{n,p}}$. Using now the definition of the $\widetilde{\lambda}_{j}$'s defined in (\ref{eq:tilde_lambda}) and our assumption that $\beta\in W^{\text{per}}(k,L)$ for a certain positive integer $k$ and a positive real number $L$ we get that 
\begin{equation}
    \begin{aligned}
        \mathbb{E}(\|\beta - \beta^{(m)}\|_{\widetilde{\Gamma}}^{2}) &= \sum_{j=D_m+1}^{D_{N_{n,p}}}\lambda_{j}\beta_{j}^{2} + \sum_{j= D_m+1}^{D_{N_{n,p}}}\overline{\lambda}_{j}\beta_{j}^{2} + \sum_{j= D_m+1}^{D_{N_{n,p}}}\frac{\tau^{2}}{p}\beta_{j}^{2} \\
        &\leq \sum_{j> D_m}\lambda_{j}\beta_{j}^{2} + \sup_{D_m < l\leq D_{N_{n,p}}}\overline{\lambda}_{l}\sum_{j> D_m}\beta_{j}^{2} + \sum_{j> D_m}\frac{\tau^{2}}{p}\beta_{j}^{2} \\
        &\leq C_{3}'\bigl(\|\beta-\beta^{(m)}\|_{\Gamma}^{2} + D_{m}^{-2k}p^{-2a} + \frac{\tau^{2}}{p}D_{m}^{-2k}\bigr),
    \end{aligned}
\end{equation}
where $C_{3}'>0$ depends on $a, c', L, k$. Let us give more details on how we derived the last inequality. Let us focus first on the term $\sup_{D_m < l\leq D_{N_{n,p}}}\overline{\lambda}_{l}\sum_{j> D_m}\beta_{j}^{2}$. Lemma \ref{lem:overline_lambda} allows us to deduce that $\overline{\lambda}_{j}=\mathcal{O}(p^{-2a})$ for all $j=1,...,D_{N_{n}}$. Hence by Lemma A.3 of Tsybakov \cite{tsybakov2009introduction} we have
\[\sup_{D_m < l\leq D_{N_{n,p}}}\overline{\lambda}_{l}\sum_{j> D_m}\beta_{j}^{2}\leq C'p^{-2a}D_{m}^{-2k},\]
where $C'$ is a positive constant which depends on $c', a, L, k$. Using Lemma A.3 of Tsybakov \cite{tsybakov2009introduction} again we get that $\frac{\tau^{2}}{p}\sum_{j>D_{m}}\beta_{j}^{2}\leq C''\frac{\tau^{2}}{p}D_{m}^{-2k}$ where $C''$ is a positive constant which depends on $L, k$. Hence, for all $m\in\mathcal{M}_{n,p}$,
\begin{equation*}
    \mathbb{E}(\|\beta - \beta^{(m)}\|_{\widetilde{\Gamma}}^{2}) \leq C_{3}'\bigl(\|\beta-\beta^{(m)}\|_{\Gamma}^{2} + \sup_{l\in\mathcal{M}_{n,p}} D_{l}^{-2k}p^{-2a} + \frac{\tau^{2}}{p}D_{l}^{-2k}\bigr).
\end{equation*}
Since for all $l\in\mathcal{M}_{n,p}$ $D_l\geq 1$ we have $\sup_{l\in\mathcal{M}_{n,p}} D_{l}^{-2k}p^{-2a} + \frac{\tau^{2}}{p}D_{l}^{-2k}=\left(p^{-2a}+\frac{\tau^{2}}{p}\right)$. Now using the fact that $p^{-2a}\leq p^{-1}$ as $a> 1/2$, we find that 
\begin{equation*}
    \mathbb{E}(\|\beta - \beta^{(m)}\|_{\widetilde{\Gamma}}^{2}) \leq C_{2}''\bigl(\|\beta-\beta^{(m)}\|_{\Gamma}^{2} + p^{-1}\bigr).
\end{equation*}
\end{proof}

\section{Proof of Theorem \ref{thm:4.1}}
Using the result of Theorem \ref{thm:3.2}, we get that 
\begin{equation*}
    \begin{aligned}
    \mathbb{E}(\|\widetilde{\beta}-\beta\|_{\Gamma}^{2}) 
    &\leq C_{3} \biggl[ \min_{m \in \mathcal{M}_{n,p}} \Bigl( \|\beta^{(m)}-\beta\|_{\Gamma}^{2} + pen(m) \Bigr) \\
    &\quad + \mathbb{E}(\|\widetilde{X}-X\|_{L^{2}}^{2}) \\
    &\quad + \frac{1}{n} \left( \mathbb{E} \|\widetilde{X} - X\|_{L^2}^4 \right)^{1/2} + \frac{1}{p} +\frac{1}{n}+ \frac{1}{np}\biggr].
\end{aligned}
\end{equation*}
With the assumption that the $\lambda_j$'s decay in a polynomial way (see (\textbf{H3})) and that $\beta$ is in $W^{\text{per}}(k,L)$, we can use the result of the Theorem 2 of Brunel and Roche \cite{Brunel2015penalized} which gives that
\[\|\beta-\beta^{(m)}\|_{\Gamma}^{2}=\sum_{j>D_{m}}\beta_{j}^{2}\lambda_{j}\leq CD_{m}^{-2a-2k}.\]Combining this reasoning with Lemmas \ref{lem:B1} and \ref{lem:theo_41} leads to the following equation,
\begin{equation}
    \mathbb{E}(\|\widetilde{\beta}-\beta\|_{\Gamma}^{2})\leq C
\min_{\substack{
1\leq D_{N_{n,p}} < p \wedge \sqrt{\frac{n}{\ln^3 n}}, \\
1 \leq D_m \le D_{N_{n,p}}
}}
\bigl( f(D_m) + g(D_{N_{n,p}}) \bigr),
\end{equation}
where for all $x\geq1$,
\[f(x)=x^{-2k-2a}+x\frac{\sigma^{2}}{n}\qquad \text{and}\qquad g(x)=\frac{1}{n}+\frac{1}{p}+\frac{1}{np}+\left(x^{-(2a-1)}+\frac{x}{p}\right)\left(\frac{1}{n}+1\right).\]
Let us focus on minimizing the function $g$ first. As g is strictly convex, coercive and $C^{1}$, $g$ admits a unique minimum which is given by solving $g'=0$. For all $x\geq 1$, we have
\[g'(x)=\left[(1-2a)x^{-2a}+\frac{1}{p}\right]\cdot\left[\frac{1}{n}+1\right].\]Therefore, $g'(x)=0 \iff x=\bigl((2a-1)p\bigr)^{1/(2a)}$. Taking into account the fact that $D_{N_{n,p}}$ must satisfy the following inequality $D_{N_{n,p}}<\min({p,\frac{n}{\ln^{2}n}})$, we get that 
\begin{equation*}
    D_{N_{n,p}}^{*} \approx
\begin{cases} 
    p^{\frac{1}{2a}} & \text{if } p \lesssim \bigl(\frac{n}{\ln(n)^{2}}\bigr)^{2a} \\
    \frac{n}{\ln(n)^{2}} & \text{otherwise .}
\end{cases}
\end{equation*}
\textit{Remark :} We derive the condition $p \lesssim \bigl(\frac{n}{\ln(n)^{2}}\bigr)^{2a}$ since if $\min(p,\frac{n}{\ln^{2}n})=p$ then $p^{\frac{1}{2a}}\leq p$ as $a> 1/2$, and $D_{N_{n,p}}^{*}\approx p^{1/(2a)}$ in this situation. Otherwise, if $\min(p,\frac{n}{\ln^{2}n})=\frac{n}{\ln^{2}n}$ then $$p^{1/(2a)}\lesssim \frac{n}{\ln^{2}n}\iff p\lesssim \left(\frac{n}{\ln^{2}n}\right)^{2a}.$$If $p\gtrsim \left(\frac{n}{\ln^{2}n}\right)^{2a}$ then the best value that $D_{N_{n,p}}$ can reach is $\frac{n}{\ln(n)^{2}}$.

\medskip\noindent
Now let's focus on $f$. We can easily verify that $f$ is $C^{1}$, strictly convex and coercive therefore the function admits a unique minimum. For all $x\geq 1$ we have
\[f'(x)=-(2k+2a)x^{-(2k+2a+1)}+\frac{\sigma^{2}}{n}.\]
Therefore, $f'(x)=0 \iff x=\bigl(\frac{n(2k+2a)}{\sigma^{2}}\bigr)^{1/(2k+2a+1)}$. Thus,
\begin{equation*}
    D_{m}^{*} \approx
        n^{1/(2a+2k+1)} .
\end{equation*}
Now let's check when $D_{m}^{*}\leq D_{N_{n,p}}^{*}$. We distinguish two cases :
\begin{itemize}
    \item If $D_{N_{n,p}}^{*}\approx p^{1/(2a)}$ (or the case where $p \lesssim \bigl(\frac{n}{\ln(n)^{2}}\bigr)^{2a}$) then,
    \begin{equation*}
            D_{m}^{*}\leq D_{N_{n,p}}^{*} \iff n^{\frac{1}{2a+2k+1}}\leq p^{\frac{1}{2a}} \iff n^{\frac{2a}{2a+2k+1}}\leq p.
    \end{equation*}
    \item If $D_{N_{n,p}}^{*}\approx \frac{n}{\ln(n)^{2}}$ (or the case where $p \gtrsim \bigl(\frac{n}{\ln(n)^{2}}\bigr)^{2a}$) then,
    \begin{equation*}
        D_{m}^{*}\leq D_{N_{n,p}}^{*} \iff n^{\frac{1}{2a+2k+1}}\leq \frac{n}{\ln^{2}(n)} .
    \end{equation*}
    As $a>0$ and $k\geq 1$, the last inequality and the condition $p \gtrsim \bigl(\frac{n}{\ln(n)^{2}}\bigr)^{2a}$ are always compatible for $n$ large enough.
\end{itemize}
From these two cases, we can deduce the following ones :
\begin{itemize}
    \item If $n^{\frac{2a}{2a+2k+1}}\lesssim p \lesssim \bigl(\frac{n}{\ln(n)^{2}}\bigr)^{2a}$, then
    \[\mathbb{E}(\|\widetilde{\beta}-\beta\|_{\Gamma}^{2})\leq C_{1}\left[n^{\frac{-2k-2a}{2k+2a+1}}+p^{\frac{-(2a-1)}{2a}}+p^{\frac{-(2a-1)}{2a}}n^{-1}+\frac{1}{n}+\frac{1}{p}+\frac{1}{np}\right],\]
 where $C_{1}>0$. Since $k\geq1$, $a>0$, $(-2k-2a)/(2k+2a+1)$ lies between -1 and 0. Therefore $n^{\frac{-2k-2a}{2k+2a+1}}$ converges slower than $n^{-1}$. Moreover the fact that $-1\leq -(2a-1)/(2a)$ allows us to deduce that $p^{-1}$ is dominated by $p^{-(2a-1)/(2a)}$. We also have,
 \[\frac{n^{-1}p^{-1}}{n^{-1}p^{-\frac{(2a-1)}{2a}}}=p^{-\frac{1}{2a}}\xrightarrow[p \to \infty]{} 0,\]and
 \[\frac{n^{-1}p^{-\frac{(2a-1)}{2a}}}{n^{-\frac{2k+2a}{2k+2a+1}}}\lesssim n^{-\frac{2a}{2a+2k+1}}\xrightarrow[n \to \infty]{} 0,\]where we used the fact that $p^{-1}\lesssim n^{-\frac{2a}{2a+2k+1}}$ and thus that $p^{-\frac{(2a-1)}{2a}}\lesssim n^{-\frac{2a-1}{2a+2k+1}}$ to derive the last inequality. This allows us to conclude that $n^{-1}p^{-\frac{(2a-1)}{2a}}$ and $n^{-1}p^{-1}$ are dominated by $n^{-\frac{2a+2k}{2a+2k+1}}$. Hence the dominant rate is $\mathcal{O}(n^{\frac{-2k-2a}{2k+2a+1}}+p^{\frac{-(2a-1)}{2a}})$, and 
 \[\mathbb{E}\Bigl(\|\widetilde{\beta}-\beta\|_{\Gamma}^{2}\Bigr)=\mathcal{O}\Bigl(n^{-\frac{2a+2k}{2a+2k+1}}+p^{-\frac{2a-1}{2a}}\Bigr).\]
    \item If $p\gtrsim \bigl(\frac{n}{\ln(n)^{2}}\bigr)^{2a}$, then
\[\mathbb{E}\left(\|\widetilde{\beta}-\beta\|_{\Gamma}^{2}\right)
\leq \begin{aligned}[t]
&C_2\Bigl[n^{\frac{-2k-2a}{2k+2a+1}}
+ n^{-(2a-1)}\ln^{-2(2a-1)} (n)
+ n^{-2a}\ln^{-2(2a-1)} (n)
\\
&\quad
+\frac{1}{p}\frac{n}{\ln^{2} (n)}
+ \frac{1}{np}\frac{n}{\ln^{2} (n)}
+ \frac{1}{n}
+ \frac{1}{p}
+ \frac{1}{np}\Bigr],
\end{aligned}
\]where $C_{2}>0$. 
We now compare the different terms. First, let's denote
\[T_1:=n^{-\frac{2k+2a}{2k+2a+1}},\qquad T_2:=n^{-(2a-1)}\ln^{-2(2a-1)}(n),\qquad T_3:=n^{-2a}\ln^{-2(2a-1)}(n),\]
\[T_4:=\frac{n}{p\ln^2 n},\qquad
T_5:=\frac{1}{p\ln^2 n},\qquad
T_6:=\frac1n,\qquad
T_7:=\frac1p,\qquad
T_8:=\frac1{np}.
\]Then the bound may be rewritten as
\[\mathbb{E}\left(\|\widetilde{\beta}-\beta\|_{\Gamma}^{2}\right)\leq C_2(T_1+T_2+T_3+T_4+T_5+T_6+T_7+T_8).\]
We first observe that
\[T_3=n^{-2a}\ln^{-2(2a-1)}(n)=\frac{1}{n}\,n^{-(2a-1)}\ln^{-2(2a-1)}(n)=\frac{1}{n}T_2.\]
Hence $T_3=o(T_2)$ and thus $T_3$ is negligible compared with $T_2$. In a similar way we get that
\[T_5=\frac{1}{p\ln^2 n}=\frac1n\,\frac{n}{p\ln^2 n}=\frac1n\,T_4,\]
therefore $T_5=o(T_4)$. Moreover,
\[T_8=\frac{1}{np}=\frac1n\,\frac1p=\frac1n\,T_7,\]
thus $T_8=o(T_7)$. Next, let's compare $T_7$ and $T_4$:
\[\frac{T_4}{T_7}=\frac{n/(p\ln^2 n)}{1/p}=\frac{n}{\ln^2 n}\xrightarrow[n\to\infty]{}\infty.\]
Therefore $T_7=o(T_4)$, and consequently $T_8=o(T_4)$ as well. Finally, let us compare $T_6$ and $T_1$. Since $\frac{2k+2a}{2k+2a+1}<1,$ we have that $T_1=n^{-\frac{2k+2a}{2k+2a+1}}$ decays more slowly than $n^{-1}$. Hence $T_6=o(T_1)$. Therefore the bound simplifies to
\[\mathbb{E}\left(\|\widetilde{\beta}-\beta\|_{\Gamma}^{2}\right)=\mathcal{O}\left(n^{-\frac{2k+2a}{2k+2a+1}}+n^{-(2a-1)}\ln^{-2(2a-1)}(n)+\frac{n}{p\ln^2 n}\right).\]As we are in the situation where $p\gtrsim \left(\frac{n}{\ln^2 n}\right)^{2a}$ we have that
\[T_4=\frac{n}{p\ln^2 n}\lesssim\frac{n}{\left(n/\ln^2 n\right)^{2a}\ln^2 n}=\frac{n}{n^{2a}\ln^{-4a}n\cdot \ln^2 n}=n^{1-2a}\ln^{4a-2}(n).\]
Since \(1-2a=-(2a-1)\), we obtain
\[T_4=\mathcal{O}\!\left(n^{-(2a-1)}\ln^{4a-2}(n)\right).\]
Therefore,
\[\mathbb{E}\left(\|\widetilde{\beta}-\beta\|_{\Gamma}^{2}\right)=\mathcal{O}\left(n^{-\frac{2k+2a}{2k+2a+1}}+n^{-(2a-1)}\ln^{-2(2a-1)}(n)+n^{-(2a-1)}\ln^{4a-2}(n)\right).\]
We compare now the last two terms. Since for all $n\ge 2$, $\ln^{4a-2}(n)\geq \ln^{-2(2a-1)}(n)$ the term $n^{-(2a-1)}\ln^{4a-2}(n)$ is larger than $n^{-(2a-1)}\ln^{-2(2a-1)}(n).$ Thus $T_2$ is dominated by $T_4$ under the present lower bound on $p$. Therefore,
\[\mathbb{E}\left(\|\widetilde{\beta}-\beta\|_{\Gamma}^{2}\right)=\mathcal{O}\!\left(n^{-\frac{2k+2a}{2k+2a+1}}+n^{-(2a-1)}\ln^{4a-2}(n)\right).\]
We now compare
\[n^{-\frac{2k+2a}{2k+2a+1}}\qquad\text{and}\qquad n^{-(2a-1)}\ln^{4a-2}(n).\]
The dominant term is determined first by the powers of $n$. We therefore study whether $\frac{2k+2a}{2k+2a+1}<2a-1.$ We have
\begin{equation*}
\begin{aligned}
\frac{2k+2a}{2k+2a+1}<2a-1
&\iff 2k+2a<(2a-1)(2k+2a+1)\\
&\iff 2k+2a<4ak+4a^2-2k-1\\
&\iff 0<4a^2+4ak-2a-4k-1\\
&\iff 0<4a(a-1)+4k(a-1)+(2a-1)\\
&\iff 0<(a-1)(4a+4k)+(2a-1).
\end{aligned}
\end{equation*}
Since $a\geq 1$ and $k\geq 0$, it follows that
\[(a-1)(4a+4k)\geq 0\qquad\text{and}\qquad2a-1>0,\]
and thus $(a-1)(4a+4k)+(2a-1)>0.$ Hence,
\[\frac{2k+2a}{2k+2a+1}<2a-1.\]
Therefore $n^{-\frac{2k+2a}{2k+2a+1}}$ decays more slowly and is the dominant term. We conclude that
\[\mathbb{E}\left(\|\widetilde{\beta}-\beta\|_{\Gamma}^{2}\right)=\mathcal{O}\left(n^{-\frac{2k+2a}{2k+2a+1}}\right).\]
\end{itemize}
Let us focus now on the case where $p \lesssim n^{\frac{2a}{2a+2k+1}}$. In this situation we can't have $D_{m}^{*}\leq D_{N_{n,p}}^{*}$ and $D_{m}^{*}\approx n^{1/(2a+2k+1)}$. Here the best value possible that minimizes $f$ is attained by taking $D_{m}^{*}=D_{N_{n,p}}^{*}\approx p^{1/(2a)}$. We then get 
\begin{equation*}
    \mathbb{E}\Bigl(\|\widetilde{\beta}-\beta\|_{\Gamma}^{2}\Bigr) \leq C_{3}\left(p^{-\frac{2a+2k}{2a}}+n^{-1}p^{\frac{1}{2a}}+p^{-\frac{2a-1}{2a}} +p^{-\frac{2a-1}{2a}}n^{-1} +n^{-1} + p^{-1} +n^{-1}p^{-1}\right).
\end{equation*}
Following the previous discussion, we have that $p^{-1}$ is dominated by $p^{-\frac{2a-1}{2a}}$. As $a>0$ and $k\geq 1$ we have $-(2a+2k)/(2a)\leq -(2a-1)/(2a)$ and $p^{-\frac{2a+2k}{2a}}$ is dominated by $p^{-\frac{2a-1}{2a}}$.

\medskip\noindent
Since we are in the framework of $p \lesssim n^{\frac{2a}{2a+2k+1}}$, we have $n^{-1}\lesssim p^{-\frac{2a+2k+1}{2a}}$ and obviously,
$-(2a+2k+1)/(2a)\leq -(2a-1)/(2a)$. Hence $p^{-(2a-1)/(2a)}$ dominates $n^{-1}$. 

\medskip\noindent
We also have, $n^{-1}p^{\frac{1}{2a}}\lesssim p^{-\frac{2a+2k}{2a}}$ and using a previous argument, we get that $p^{-\frac{2a-1}{2a}}$ dominates $n^{-1}p^{\frac{1}{2a}}$. Moreover,
\[\frac{p^{-\frac{(2a-1)}{2a}}n^{-1}}{p^{-\frac{(2a-1)}{2a}}}\lesssim p^{-\frac{2a+2k+1}{2a}}\xrightarrow[p \to \infty]{} 0\]and
\[\frac{n^{-1}p^{-1}}{p^{-\frac{(2a-1)}{2a}}}\lesssim p^{-\frac{2a+2k}{2a}}\xrightarrow[p \to \infty]{} 0.\]Therefore, $p^{-\frac{(2a-1)}{2a}}n^{-1}$ and $n^{-1}p^{-1}$ are dominated by $p^{-\frac{(2a-1)}{2a}}$. Hence,
 \[\mathbb{E}\Bigl(\|\widetilde{\beta}-\beta\|_{\Gamma}^{2}\Bigr)=\mathcal{O}\Bigl(p^{-\frac{2a-1}{2a}}\Bigr).\]

\begin{lemma}
\label{lem:B1}
    Assume that \textbf{(H3)} holds. Then \[\mathbb{E}\Bigl(\|\widetilde{X}-X\|_{L^{2}}^{2}\Bigr) \leq C'\left(\frac{D_{N_{n,p}}}{p}+D_{N_{n,p}}^{-(2a-1)}\right),\]
where $C'$ is a positive constant which depends on $\tau, c', a$. 
\end{lemma}

\begin{proof}
We can express $X$ in the Fourier basis : $X=\sum_{1\leq l}x_{l}\phi_{l}$, where $x_{l}=\langle X,\phi_{l}\rangle_{L^{2}}$. Using the definition of $X$ and $\widetilde{X}$ we get that 

\begin{equation*}
    \begin{split}
\mathbb{E}(\| \widetilde{X}-X\|_{L^{2}}^{2}) &= \mathbb{E}\Bigl(\int_{0}^{1}(\widetilde{X}(t)-X(t))^{2}dt\Bigr) \\
 & = \mathbb{E}\Bigl(\int_{0}^{1}|\sum_{j=1}^{D_{N_{n,p}}}\widetilde{x}_{j}\phi_{j}(t)-\sum_{1\leq j}x_{j}\phi_{j}(t)|^{2}dt\Bigr) \\
 &= \mathbb{E}\Bigl(\int_{0}^{1}|\sum_{j=1}^{D_{N_{n,p}}}(\widetilde{x}_{j}-x_{j})\phi_{j}(t)-\sum_{D_{N_{n,p}}< j}x_{j}\phi_{j}(t)|^{2}dt\Bigr).
\end{split}
\end{equation*}
As the $(\phi_{j})_{1\leq j}$ are orthonormal, we then get 
\[\mathbb{E}(\| \widetilde{X}-X\|_{L^{2}}^{2})=\mathbb{E}\Bigl(\sum_{j=1}^{D_{N_{n,p}}}(\widetilde{x}_{j}-x_{j})^{2}\Bigr)+\mathbb{E}\Bigl(\sum_{D_{N_{n,p}}<j}x_{j}^{2}\Bigr).\]
As $\widetilde{x}_{j}=\overline{x}_{j}+\overline{\eta}_{j}$ for all $j\in\{1,\ldots,D_{N_{n,p}}\}$, where $\overline{x}_{j}$ and $\overline{\eta}_{j}$ are defined in (\ref{eq:eta_bar_x_bar}), we get that $(\widetilde{x}_{j}-x_{j})^{2}=(\overline{x}_{j}-x_{j})^{2}+\overline{\eta}_{j}^{2}+2\overline{\eta}_{j}(\overline{x}_{j}-x_{j})$. Now, using the fact that $\mathbb{E}(\overline{\eta}_{j})=0$, $\mathbb{E}(\overline{\eta}_{j}^{2})= \frac{\tau^{2}}{p}$ by Lemma \ref{lem:A2} and that $(\eta_{j})_{j=1,...,D_{N_{n}}}$ are independent of $X$, we get the following equation :
\begin{equation}\label{eq:start_last_lem}
    \mathbb{E}\Bigl(\|\widetilde{X}-X\|_{L^{2}}^{2}\Bigr) \leq \frac{D_{N_{n,p}}\tau^{2}}{p}+\mathbb{E}\Bigl(\sum_{j=1}^{D_{N_{n,p}}}(\overline{x}_{j}-x_{j})^{2}\Bigr)+\mathbb{E}\Bigl(\sum_{D_{N_{n,p}}<j}x_{j}^{2}\Bigr).
\end{equation}
Let us focus on the second term of the right part of the inequality. Using Fubini's theorem we have,
\begin{equation*}
    \begin{aligned}
        \mathbb{E}\Bigl(\sum_{D_{N_{n,p}}<j}x_{j}^{2}\Bigr) &= \sum_{D_{N_{n,p}}<j}\mathbb{E}\left(x_{j}^{2}\right).
    \end{aligned}
\end{equation*}
Now using Fubini's theorem and the definition of the kernel $K$ associated with the norm Gamma mentioned in (\ref{eq:kernel_Gamma}) we get,
\begin{equation*}
\begin{aligned}
    \mathbb{E}\left(x_{j}^{2}\right) &= \mathbb{E}\left(\langle X, \phi_j\rangle_{L^{2}}^{2}\right) \\
    &= \int_{0}^{1}\int_{0}^{1}\phi_j(t)\phi_{j}(s)\mathbb{E}\left(X(s)X(t)\right)ds dt \\
    &= \int_{0}^{1}\phi_{j}(t)\int_{0}^{1}K(s,t)\phi_j(s)ds dt \\
    &= \langle \Gamma \phi_j, \phi_j\rangle_{L^{2}} \\
    &= \lambda_j.
\end{aligned}
\end{equation*}
Hence using (\textbf{H3}),
\begin{equation}\label{eq:middle_last_lem}
    \begin{aligned}
        \mathbb{E}\Bigl(\sum_{D_{N_{n,p}}<j}x_{j}^{2}\Bigr) &= \sum_{D_{N_{n,p}}<j}\mathbb{E}\left(x_{j}^{2}\right) \\
        &=\sum_{D_{N_{n,p}}<j}\lambda_j  \\
        &\leq c'\sum_{D_{N_{n,p}}<j}j^{-2a} \\
        &\leq \frac{c'}{2a-1}D_{N_{n,p}}^{-(2a-1)} .
    \end{aligned}
\end{equation}
Let us focus now on $\mathbb{E}\Bigl(\sum_{j=1}^{D_{N_{n,p}}}(\overline{x}_{j}-x_{j})^{2}\Bigr)$. In Lemma \ref{lem:A2} we defined for all $j\in\{1,\ldots,D_{N_{n,p}}\}$ 
\[e_{j}=\overline{x}_{j}-x_{j},\]
and proved that \[\mathbb{E}\left(e_{j}^{2}\right)=\overline{\lambda}_{j}.\]
Using now Lemma \ref{lem:overline_lambda} we get that
\[\forall j\in\{1,\ldots,D_{N_{n,p}}\} \qquad \overline{\lambda}_{j}=\mathcal{O}(p^{-2a}) .\]
Hence,
\begin{equation}\label{eq:last_last_lem}
    \mathbb{E}\left(\sum_{j=1}^{D_{N_{n,p}}}(\overline{x}_{j}-x_{j})^{2}\right) \leq CD_{N_{n,p}}p^{-2a}.
\end{equation}
Combining Equations (\ref{eq:start_last_lem}), (\ref{eq:middle_last_lem}) and (\ref{eq:last_last_lem}) and using the fact that $p^{-2a}\leq p^{-1}$ as $a>1/2$, we get
\[\mathbb{E}\Bigl(\|\widetilde{X}-X\|_{L^{2}}^{2}\Bigr) \leq C'\left(\frac{D_{N_{n,p}}}{p}+D_{N_{n,p}}^{-(2a-1)}\right).\]
\end{proof}

\begin{lemma}\label{lem:theo_41}
    Assume that \textbf{(H1), (H2), (H3)} hold. Then,
    \begin{equation*}
        \mathbb{E}^{1/2}\left(\|X-\widetilde{X}\|_{L^{2}}^{4}\right) \leq C(a,c',\tau)\left(D_{N_{n,p}}^{-(2a-1)}+\frac{D_{N_{n,p}}}{p}\right).
    \end{equation*}
\end{lemma}
\begin{proof}
    By definition of $X_1$ and $\widetilde{X}_1$ we have
    \[\|X-\widetilde{X}\|_{L^{2}}^{2}=\sum_{j\leq D_{N_{n,p}}}(\widetilde{x}_{j}-x_{j})^{2}+\sum_{j>D_{N_{n,p}}}x_{j}^{2}.\]Using the inequality $(x+y)^{2}\leq 2x^{2}+2y^{2}$ first, and then the inequality $\sqrt{x+y}\leq \sqrt{x}+\sqrt{y}$ we get that 
    \begin{equation*}
        \begin{aligned}
            \mathbb{E}^{1/2}\left(\|X-\widetilde{X}\|_{L^{2}}^{4}\right)&\leq \sqrt{2}\left[\mathbb{E}\Bigl(\bigl(\sum_{j\leq D_{N_{n,p}}}(\widetilde{x}_{j}-x_{j})^{2}\bigr)^{2}\Bigr)+\mathbb{E}\Bigl(\bigl(\sum_{j>D_{N_{n,p}}}x_{j}^{2}\bigr)^{2}\Bigr)\right]^{1/2} \\
            &\leq \sqrt{2}\left[\mathbb{E}^{1/2}\Bigl(\bigl(\sum_{j\leq D_{N_{n,p}}}(\widetilde{x}_{j}-x_{j})^{2}\bigr)^{2}\Bigr) + \mathbb{E}^{1/2}\Bigl(\bigl(\sum_{j>D_{N_{n,p}}}x_{j}^{2}\bigr)^{2}\Bigr)\right] .
        \end{aligned}
    \end{equation*}
    Let us first study the term $\mathbb{E}\Bigl(\bigl(\sum_{j\leq D_{N_{n,p}}}(\widetilde{x}_{j}-x_{j})^{2}\bigr)^{2}\Bigr)^{1/2}$. Let us recall that we can write $\widetilde{x}_{j}$ as $\widetilde{x}_{j}=x_{j}+e_{j}+\overline{\eta}_{j}$ as defined in Equations (\ref{eq:eta_bar_x_bar}) and (\ref{eq:e_ij}). Therefore, applying Cauchy--Schwarz inequality to the vector $(1,1,\ldots,1)\in \mathbb{R}^{D_{N_{n,p}}}$ and $\bigl((e_{j}+\overline{\eta}_{j})^{2}\bigr)_{j=1,\ldots,D_{N_{n,p}}}$ we get
    \begin{equation*}
        \begin{aligned}
            \left(\sum_{j=1}^{D_{N_{n,p}}}(\widetilde{x}_{j}-x_{j})^{2}\right)^{2} &= \left(\sum_{j=1}^{D_{N_{n,p}}}(e_{j}+\overline{\eta}_{j})^{2}\right)^{2} \\
            &\leq D_{N_{n,p}}\sum_{j=1}^{D_{N_{n,p}}}(e_{j}+\overline{\eta}_{j})^{4} \\
            &\leq 8D_{N_{n,p}}\left(\sum_{j=1}^{D_{N_{n,p}}}e_{j}^{4}+\sum_{j=1}^{D_{N_{n,p}}}\overline{\eta}_{j}^{4}\right) .
        \end{aligned}
    \end{equation*}

\noindent Let us study the term $\mathbb{E}\left(\sum_{j=1}^{D_{N_{n,p}}}e_{j}^{4}\right)$. Hypothesis \textbf{(H1)} implies that the $e_{j}$'s are sub-Gaussian with variance factor $\overline{\lambda}_{j}$. Theorem 2.1 of Boucheron et al. \cite{boucheron2013concentration} gives us that for all $q'\geq 1$ and for all $1\leq j\leq D_{N_{n,p}}$,
\[\mathbb{E}\left(e_{j}^{2q'}\right)\leq q'!(4\overline{\lambda}_j)^{q'},\]and taking $q'=2$ leads to 
\[\mathbb{E}\left(e_{j}^{4}\right)\leq 32\overline{\lambda}_{j}^{2}.\]Lemma \ref{lem:overline_lambda} gives that for all $1\leq j\leq D_{N_{n,p}}$, $\overline{\lambda}_j\leq C(a,c')p^{-2a}$, hence 
\[\mathbb{E}\left(e_{j}^{4}\right)\leq C(a,c')p^{-4a}.\] Now, let us focus on the term $\mathbb{E}\left(\overline{\eta}_{j}^{4}\right)$. In a similar way, hypothesis \textbf{(H2)} gives us that $\overline{\eta}_{j}$ is sub-Gaussian with variance factor $\tau^2/p$. Theorem 2.1 of Boucheron et al. \cite{boucheron2013concentration} gives that for all $q'\geq 1$ and for all $1\leq j\leq D_{N_{n,p}}$,
\begin{equation*}
    \mathbb{E}\left(\overline{\eta}_{j}^{2q'}\right)\leq q'!\left(4\frac{\tau^{2}}{p}\right)^{q'}.
\end{equation*}
Taking $q'=2$ we get that $$\mathbb{E}\left(\overline{\eta}_{j}^{4}\right)\leq 32\frac{\tau^{4}}{p^{2}}.$$This allows us to get that 
\begin{equation}\label{eq:A_squared}
\begin{aligned}
    \mathbb{E}^{1/2}\Bigl(\bigl(\sum_{j\leq D_{N_{n,p}}}(\widetilde{x}_{j}-x_{j})^{2}\bigr)^{2}\Bigr) &\leq D_{N_{n,p}}\left(\frac{C(c',a)}{p^{4a}}+\frac{C(\tau)}{p^{2}}\right)^{1/2} \\
    &\leq D_{N_{n,p}}\frac{C(c',a,\tau)}{p},
\end{aligned}
\end{equation}
where we used the fact that $\sqrt{x+y}\leq \sqrt{x}+\sqrt{y}$ and that $a> 1/2$ to get the last inequality. Lastly, let's focus on $\mathbb{E}^{1/2}\Bigl(\bigl(\sum_{j>D_{N_{n,p}}}x_{j}^{2}\bigr)^{2}\Bigr)$.
By using Tonelli's theorem, we have
\begin{equation*}
    \begin{aligned}
        \mathbb{E}\Bigl(\bigl(\sum_{j>D_{N_{n,p}}}x_{j}^{2}\bigr)^{2}\Bigr) &= \mathbb{E}\left(\sum_{j>D_{N_{n,p}}}\sum_{k>D_{N_{n,p}}}x_{j}^{2}x_{k}^{2}\right)\\
        &= \sum_{j>D_{N_{n,p}}}\sum_{k>D_{N_{n,p}}}\mathbb{E}\left(x_{j}^{2}x_{k}^{2}\right) .
    \end{aligned}
\end{equation*}
Applying Cauchy--Schwarz inequality to $x_{j}^{2}$ and $x_{k}^{2}$ we find
\begin{equation*}
    \begin{aligned}
        \mathbb{E}\Bigl(\bigl(\sum_{j>D_{N_{n,p}}}x_{j}^{2}\bigr)^{2}\Bigr) &\leq \sum_{j>D_{N_{n,p}}}\sum_{k>D_{N_{n,p}}}\mathbb{E}(x_{j}^{4})^{1/2}\mathbb{E}(x_{k}^{4})^{1/2} \\
        &= \left(\sum_{j>D_{N_{n,p}}}\mathbb{E}(x_{j}^{4})^{1/2}\right)^{2} .
    \end{aligned}
\end{equation*}
Hypothesis \textbf{(H1)} and the fact that the $\lambda_j$'s decrease in a polynomial way, give us that 
\begin{equation*}
    \begin{aligned}
        \mathbb{E}\Bigl(\bigl(\sum_{j>D_{N_{n,p}}}x_{j}^{2}\bigr)^{2}\Bigr) &\leq 32\left(\sum_{j>D_{N_{n,p}}}\lambda_j\right)^{2} \\
        &\leq 32c'^{2}\left(\sum_{j>D_{N_{n,p}}}j^{-2a}\right)^{2} \\
        &\leq C(a,c')D_{N_{n,p}}^{-2(2a-1)},
    \end{aligned}
\end{equation*}
and 
\begin{equation}\label{eq:B_squared}
    \mathbb{E}^{1/2}\Bigl(\bigl(\sum_{j>D_{N_{n,p}}}x_{j}^{2}\bigr)^{2}\Bigr)\leq C(a,c')D_{N_{n,p}}^{-(2a-1)}.
\end{equation}
Combining (\ref{eq:A_squared}) and (\ref{eq:B_squared}) we find that 
\[\mathbb{E}\left(\|X-\widetilde{X}\|_{L^{2}}^{4}\right) \leq C(a,c',\tau)\left(D_{N_{n,p}}^{-(2a-1)}+\frac{D_{N_{n,p}}}{p}\right).\]
\end{proof}
\section{Proof of Theorem \ref{thm:5.1}}
This section is divided into four parts. The first part presents the three conditions of Theorem 2.7 from Tsybakov \cite{tsybakov2009introduction}, while the remaining three sections provide the results used to prove them. The proof follows the methodology introduced in Section 2.6 of Tsybakov \cite{tsybakov2009introduction}.
The first step of the proof is to build test functions. We define $M\in\mathbb{N}^{*}$ and for all $l=1,\ldots,M$ we consider $w^{(l)}\in\{0,1\}^{d}$ . We set $w^{(0)}=(0,\ldots,0)\in \mathbb{R}^d$. Then we consider the following test functions 
    \begin{equation}\label{eq:beta_test}
        \beta_{w^{(l)}}(t)=\sum_{j=1}^{d}w_{j}^{(l)}u_{j}\phi_{j}(t)\quad \text{for all}\quad l=0,\ldots,M,
    \end{equation}
    where 
    \begin{equation}\label{eq:u_j}
        u_{j}^{2}=\frac{C}{n\lambda_{j}}\quad \text{with}\quad C=\min\left(L^{2}(2\pi)^{-2k}c'^{-1},\sigma^{2}\frac{\alpha' \log(2)}{4}\right).
    \end{equation}
    We set $0<\alpha'<1/8$ and $0<\kappa<1$ such that 
    \begin{equation}\label{eq:n_geqkappa}
        n\geq (1-\kappa)^{-(2a+2k+1)}.
    \end{equation}
    We also define
    \begin{equation}\label{eq:d}
        d=\lceil \kappa n^{\frac{1}{2a+2k+1}} \rceil
    \end{equation}
    with $n$ such that $d\geq 8$.
    
    \medskip\noindent
    \begin{remark}
    The parameter $M$ corresponds to the number of hypotheses in the packing set. Its introduction is motivated by the Varshamov--Gilbert lemma (Lemma 2.9 in Tsybakov \cite{tsybakov2009introduction}), which ensures the existence of a large subset of binary vectors with pairwise Hamming distance bounded below. This large family of well-separated alternatives is the key ingredient for applying Tsybakov's lower-bound theorem.
    \end{remark}
    
    \medskip\noindent
    We then consider the following setting :

    \begin{itemize}
        \item A class of functions containing the true slope function $\beta$ that we want to estimate : $W^{\text{per}}(k,L)$ which is given by Equation \eqref{eq:w_per}.
        \item A sample $(Z_{i},Y_{i})_{i=1}^{n}$ which is such that for all $h=0,\ldots,p-1$, for all $i=1,\ldots,n$,
        \begin{equation}\label{eq:Z_i}
            Z_i(t_h)=X_i(t_h)+\eta_{i,h},
        \end{equation}
        where the $\eta_{i,h}$'s are i.i.d, Gaussian, centered and of variance $\tau^{2}$. For each $l=0,\ldots,M$ we consider
        \begin{equation}\label{eq:Y_i}
            Y_{i}=\langle X_i, \beta_{w^{(l)}}\rangle_{L^{2}}+\epsilon_{i},
        \end{equation}
        where $\beta_{w^{(l)}}\in W^{\text{per}}(k,L)$ and the $\epsilon_{i}$'s are i.i.d, are Gaussian, centered and of variance $\sigma^{2}$. Both the $\epsilon_i$'s and $\eta_{i,h}$'s are independent of everything else. We denote the joint distribution of this sample by $P_{l}=P_{w^{(l)}}$.
        \item A semi distance $\|\cdot\|_{\Gamma}$ which is defined in \eqref{eq:semi_norm_gamma}.
    \end{itemize}
        
     In order to apply Theorem 2.7 of Tsybakov \cite{tsybakov2009introduction} we need to check the three following conditions:
    \begin{enumerate}
        \item $\forall l=0,\ldots,M$, $\beta_{w^{(l)}}\in W^{\text{per}}(k,L).$
        \item $P_j\ll P_{0}$ for all $j=1,\ldots,M$ and
        \[\frac{1}{M}\sum_{j=1}^{M}K(P_{j}\|P_{0})\leq \alpha' \log(M),\]with $0<\alpha'<1/8$ and $P_{j}=P_{w^{(j)}}$ for $j=0,1,\ldots,M$. Here $K(\cdot\| \cdot)$ denotes the Kullback--Leibler divergence.
        \item $\forall$ $0\leq l<l'\leq M$ \[\|\beta_{w^{(l)}}-\beta_{w^{(l')}}\|_{\Gamma}^{2}\geq 2C_{4}n^{-\frac{2k+2a}{2k+2a+1}},\]
        where $C_4>0$ is a positive constant which will be defined bellow.
    \end{enumerate}

    Proposition \ref{prop:condition1},  \ref{prop:condition2} and \ref{prop:condition3} allow us to apply this theorem and therefore to prove the result.

\subsection{Condition 1. of the proof}
\begin{prop}\label{prop:condition1}
Consider $M\in\mathbb{N}\backslash\{0\}$ and for all $l=1,\ldots,M$, define $w^{(l)}\in\{0,1\}^{d}$, as well as $w^{(0)}=(0,0,\ldots,0)\in\mathbb{R}^{d}$  where $d$ is defined by (\ref{eq:d}). Consider for all $l=0,1,\ldots,M$, the test functions $\beta_{w^{(l)}}$ given by (\ref{eq:beta_test}). Then, for all $l=0,1,\ldots,M$,
\[\beta_{w^{(l)}}\in W^{\text{per}}(k,L).\]
\end{prop}
\begin{proof}
    As for all $t\in[0,1]$ $\beta_{w^{(0)}}(t)=0$ we have that $\beta_{w^{(0)}}$ is trivially in $W^{\text{per}}(k,L)$. Let us focus on the case where $l=1,2,\ldots,M$. We need to prove that for all $s=1,\ldots,k$, $\beta_{w^{(l)}}^{(s)}$ is 1-periodic and that $\|\beta_{w^{(l)}}^{(s)}\|_{L^{2}}^{2}\leq L^{2}$. For all $s=1,\ldots,k$ we have, by Lemma \ref{lem:derivative_computation}
    \[\begin{aligned}
    \beta_{w^{(l)}}^{(s)}(t) &= \sum_{j=1}^{\lfloor d/2 \rfloor} (2\pi j)^{s}
\Big(\big[a_{s}w_{2j}^{(l)}u_{2j}+ c_{s}\mathbf{1}_{\{2j+1 \leq d\}}w_{2j+1}^{(l)}u_{2j+1}
  \big]\phi_{2j}(t) \\
  &\quad+\big[b_{s}w_{2j}^{(l)}u_{2j}+ a_{s}\mathbf{1}_{\{2j+1 \leq d\}}w_{2j+1}^{(l)}u_{2j+1}\big]\phi_{2j+1}(t)\Big).
\end{aligned} \]
As each real Fourier basis function is $1$-periodic we have that $\beta_{w^{(l)}}^{(s)}$ is $1$-periodic for all $s=1,...,M$. With the same argument we have, for $s=0$ that the function is also $1$-periodic. To conclude the proof we then need to prove that $\|\beta_{w^{(l)}}^{(k)}\|_{L^{2}}^{2}\leq L^2$. By Lemma \ref{lem:derivative_computation} we have 
\[\|\beta_{w^{(l)}}^{(k)}\|_{L^{2}}^2 \leq (2\pi)^{2k} \sum_{j=1}^{d}j^{2k}(w_{j}u_{j})^2.\]
Using the definition of $u_j$ given by (\ref{eq:u_j}) and the fact that $\lambda_{j}$ decrease in a polynomial way we have
\[\|\beta_{w^{(l)}}^{(k)}\|_{L^{2}}^{2}\leq \frac{L^{2}}{n}\sum_{j=1}^{d}j^{2k+2a}.\]
As $a>1/2$ and $k$ is a positive integer, $2k+2a\geq 1$ and we get
\[\frac{L^{2}}{n}\sum_{j=1}^{d}j^{2k+2a}\leq \frac{L^{2}}{n}d^{2k+2a+1}.\]Using now the definition of $d$ and equation (\ref{eq:n_geqkappa}) we find for $n$ large enough
\[\|\beta_{w^{(l)}}^{(k)}\|_{L^{2}}^{2}\leq L^{2}.\]Therefore, for all $l=0,1,\ldots,M$,
\[\beta_{w^{(l)}}\in W^{\text{per}}(k,L).\]

\end{proof}

\begin{lemma}\label{lem:derivative_computation}
    Consider $M\in\mathbb{N}\backslash\{0\}$ and for all $l=1,\ldots,M$, define $w^{(l)}\in\{0,1\}^{d}$  where $d$ is defined by (\ref{eq:d}). Consider for all $l=1,\ldots,M$, the test functions $\beta_{w^{(l)}}$ given by (\ref{eq:beta_test}). Then, for all $l=1,\ldots,M$ and for all integer $s\geq 1$, $\beta_{w^{(l)}}$ is $s$-times differentiable and 
    \[\begin{aligned}
    \beta_{w^{(l)}}^{(s)}(t) &= \sum_{j=1}^{\lfloor d/2 \rfloor} (2\pi j)^{s}
\Big(\big[a_{s}w_{2j}^{(l)}u_{2j}+ c_{s}\mathbf{1}_{\{2j+1 \leq d\}}w_{2j+1}^{(l)}u_{2j+1}\big]\phi_{2j}(t) \\
&\quad+\big[b_{s}w_{2j}^{(l)}u_{2j}+ a_{s}\mathbf{1}_{\{2j+1 \leq d\}}w_{2j+1}^{(l)}u_{2j+1}\big]\phi_{2j+1}(t)\Big),\end{aligned} \]with $a_{s}=\cos{(\frac{s\pi}{2})}$, $b_{s}=-\sin{(\frac{s\pi}{2})}$ and $c_{s}=\sin{(\frac{s\pi}{2})}$. Moreover, 
\[\|\beta_{w^{(l)}}^{(s)}\|_{L^{2}}^{2}\leq (2\pi)^{2k}\sum_{j=1}^{d} j^{2k}(w_{j}u_{j})^2.\]

\end{lemma}

\begin{proof}
    Let $s\in\{1,\ldots,k\}$. Then, $\beta_{w^{(l)}}$ is $s$-times differentiable as a sum of $s$-times differentiable functions and 
\[\beta_{w^{(l)}}^{(s)}(t)=\sum_{j=1}^{d}w_{j}^{(l)}u_{j}\phi_{j}^{(s)}(t)=\sum_{j=1}^{\lfloor d/2\rfloor}w_{2j}^{(l)} u_{2j}\,\phi_{2j}^{(s)}(t)+\sum_{j=1}^{\left\lfloor (d-1)/2 \right\rfloor}w_{2j+1}^{(l)}u_{2j+1}\,\phi_{2j+1}^{(s)}(t).\]
Now,
\[
\begin{aligned}
\sqrt{2}\frac{d^{s}\cos(2\pi jt)}{dt^{s}} &= \sqrt{2}(2\pi j)^{s}\cos\left(2\pi jt + \frac{s\pi}{2}\right), \\[1ex]
\sqrt{2}\frac{d^{s}\sin(2\pi jt)}{dt^{s}} &= \sqrt{2}(2\pi j)^{s}\sin\left(2\pi jt + \frac{s\pi}{2}\right).
\end{aligned}
\]
The trigonometric formulas lead to,
\[\begin{aligned}
\sqrt{2}\cos\left(2\pi jt + \frac{s\pi}{2}\right) 
&= \sqrt{2}\cos\left(\frac{s\pi}{2}\right)\cos(2\pi jt) - \sqrt{2}\sin\left(\frac{s\pi}{2}\right)\sin(2\pi jt) \\
&= a_{s}\phi_{2j}(t) + b_{s}\phi_{2j+1}(t),
\end{aligned}\]
and
\[\begin{aligned}
\sqrt{2}\sin\left(2\pi jt + \frac{s\pi}{2}\right) 
&= \sqrt{2}\sin\left(\frac{s\pi}{2}\right)\cos(2\pi jt) + \sqrt{2}\cos\left(\frac{s\pi}{2}\right)\sin(2\pi jt) \\
&= c_{s}\phi_{2j}(t) + a_{s}\phi_{2j+1}(t),
\end{aligned}\]
with $a_{s}=\cos{(\frac{s\pi}{2})}$, $b_{s}=-\sin{(\frac{s\pi}{2})}$ and $c_{s}=\sin{(\frac{s\pi}{2})}$. More particularly these coefficients have the following values depending on $s$:

\begin{center}
\begin{tabular}{|c|c|c|c|}
\hline
\text{s mod 4} & \textbf{$a_{s}$} & \textbf{$b_{s}$} & \textbf{$c_{s}$} \\
\hline
0 & 1 & 0 & 0 \\
\hline
1 & 0 & -1 & 1 \\
\hline
2 & -1 & 0 & 0 \\
\hline
3 & 0 & 1 & -1 \\
\hline

\end{tabular}
\end{center}
Using the previous result, we get that 
\[\begin{aligned}
\beta_{w^{(l)}}^{(s)}(t)
&= \sum_{j=1}^{\lfloor d/2 \rfloor} (2\pi j)^s
   \big[a_{s}w_{2j}^{(l)}u_{2j}\phi_{2j}(t)
      + b_{s}w_{2j}^{(l)}u_{2j}\phi_{2j+1}(t)\big] \\
&\quad + \sum_{j=1}^{\lfloor (d-1)/2 \rfloor} (2\pi j)^s
   \big[c_{s}w_{2j+1}^{(l)}u_{2j+1}\phi_{2j}(t)
      + a_{s}w_{2j+1}^{(l)}u_{2j+1}\phi_{2j+1}(t)\big] \\
      &= \sum_{j=1}^{\lfloor d/2 \rfloor} (2\pi j)^{s}
\Big(
  \big[
    a_{s}w_{2j}^{(l)}u_{2j}
    + c_{s}\mathbf{1}_{\{2j+1 \leq d\}}w_{2j+1}^{(l)}u_{2j+1}
  \big]\phi_{2j}(t) \\
  &\quad+
  \big[
    b_{s}w_{2j}^{(l)}u_{2j}
    + a_{s}\mathbf{1}_{\{2j+1 \leq d\}}w_{2j+1}^{(l)}u_{2j+1}
  \big]\phi_{2j+1}(t)
\Big),
\end{aligned}\]which proves the first point of the Lemma. Now let us focus on the second one. Let's consider
\[\begin{aligned}
A_j &:= (2\pi j)^k \left( a_{k}w_{2j}^{(l)}u_{2j} + c_{k}\mathbf{1}_{\{2j+1 \le d\}}w_{2j+1}^{(l)}u_{2j+1} \right), \\
B_j &:= (2\pi j)^k \left( b_{k}w_{2j}^{(l)}u_{2j} + a_{k}\mathbf{1}_{\{2j+1 \le d\}}w_{2j+1}^{(l)}u_{2j+1} \right),
\end{aligned}\]
as well as,
\[x_j := w_{2j}^{(l)}u_{2j}, \qquad y_j := \mathbf{1}_{\{2j+1 \leq d\}}w_{2j+1}^{(l)}u_{2j+1}.\]
This gives the following expressions for $A_j$ and $B_j$ :
\[A_j = (2\pi j)^k (a_k x_j + c_k y_j), \qquad B_j = (2\pi j)^k (b_k x_j + a_k y_j).\]
From this we can deduce that
\[
\begin{aligned}
A_j^2 + B_j^2
&= (2\pi j)^{2k} \big[(a_k x_j + c_k y_j)^2 + (b_k x_j + a_k y_j)^2\big] \\
&= (2\pi j)^{2k} \big[(a_k^2 + b_k^2)x_{j}^2 + (c_k^2 + a_k^2)y_{j}^2 + 2x_jy_j(a_k c_k + a_k b_k)\big].
\end{aligned}
\]
Now we can notice that
\[a_k^2 + b_k^2 = \cos^2\left(\tfrac{k\pi}{2}\right) +\sin^2\left(\tfrac{k\pi}{2}\right) = 1,\qquad c_k^2 + a_k^2 = \sin^2\left(\tfrac{k\pi}{2}\right) + \cos^2\left(\tfrac{k\pi}{2}\right) = 1,\]
and
\[a_k c_k + a_k b_k = a_k(c_k + b_k) = 0.\]
Therefore the cross term vanishes and both quadratic coefficients are equal to $1$, which gives 
\begin{equation}\label{eq:A_j+B_j}
    A_j^2 + B_j^2 = (2\pi j)^{2k}(x_{j}^2 + y_{j}^2) = (2\pi j)^{2k}
\left((w_{2j}^{(l)}u_{2j})^2+ \mathbf{1}_{\{2j+1 \leq d\}}\,(w_{2j+1}^{(l)}u_{2j+1})^2\right).
\end{equation}
By orthonormality of the real Fourier basis 
$\{\phi_j\}_{j\ge1}$ in $L^2(0,1)$, we have $\langle \phi_i, \phi_j \rangle_{L^{2}}=\delta_{ij}.$ Hence, when $\beta_{w^{(l)}}^{(k)}$ is written as a linear combination
\[\beta_{w^{(l)}}^{(k)}(t)= \sum_{j=1}^{\lfloor d/2 \rfloor}\big(A_j\phi_{2j}(t) + B_j\phi_{2j+1}(t)\big),\]
its squared $L^2$ norm is simply the sum of the squared coefficients :
\[\|\beta_{w^{(l)}}^{(k)}\|_{L^2}^2= \sum_{j=1}^{\lfloor d/2 \rfloor}
\big(A_j^2 + B_j^2\big),\]
since all cross terms $\langle \phi_i,\phi_j\rangle_{L^{2}}$ with $i\neq j$
vanish. Substituting equation (\ref{eq:A_j+B_j}) we obtain
\begin{equation}\label{eq:almostfinal}
\begin{aligned}
    \|\beta_{w^{(l)}}^{(k)}\|_{L^2}^2
&= (2\pi)^{2k}\sum_{j=1}^{\lfloor d/2 \rfloor}j^{2k}\left(\big(w_{2j}^{(l)}u_{2j}\big)^2+ \mathbf{1}_{\{2j+1\le d\}}\big(w_{2j+1}^{(l)}u_{2j+1}\big)^2\right)  \\
&= (2\pi)^{2k}\left(\sum_{j=1}^{\lfloor d/2\rfloor} j^{2k}(w_{2j}u_{2j})^2+\sum_{j=1}^{\lfloor (d-1)/2\rfloor} j^{2k}(w_{2j+1}u_{2j+1})^2\right),
\end{aligned}
\end{equation}
where the indicator function has been replaced by the upper limit $\lfloor (d-1)/2\rfloor$. Let us treat the first sum. We set $r=2j.$ Then $j = r/2$, and when $j = 1$ we have $r = 2$, while when $j = \lfloor d/2\rfloor$, $r = 2\lfloor d/2\rfloor \leq d$. Therefore
\[\sum_{j=1}^{\lfloor d/2\rfloor} j^{2k}(w_{2j}u_{2j})^2= \sum_{\substack{r = 2,4,\dots\\ r \leq d}}\left(\frac{r}{2}\right)^{2k}(w_{r}u_{r})^2.\]
Since $r/2 \leq r$, we get the bound
\[\sum_{j=1}^{\lfloor d/2\rfloor} j^{2k}(w_{2j}u_{2j})^2\leq \sum_{\substack{r = 2,4,\dots\\ r \leq d}}r^{2k}(w_{r}u_{r})^2.\]
Let us now focus on the second sum of equation (\ref{eq:almostfinal}). We set $r = 2j + 1.$ Then $j = (r - 1)/2$, and as $j$ runs from $1$ to $\lfloor (d - 1)/2\rfloor$,  
$r$ runs over all odd integers $r = 3,5,\dots,2\lfloor (d - 1)/2\rfloor + 1 \leq d$.
Hence,
\[\sum_{j=1}^{\lfloor (d-1)/2\rfloor} j^{2k}(w_{2j+1}u_{2j+1})^2= \sum_{\substack{r = 1,3,\dots\\ r \le d}}\left(\frac{r-1}{2}\right)^{2k}(w_{r}u_{r})^2.\]
Since $(r-1)/2 \leq r$, we obtain
\[\sum_{j=1}^{\lfloor (d-1)/2\rfloor} j^{2k}(w_{2j+1}u_{2j+1})^2 \leq \sum_{\substack{r = 1,3,\dots\\ r \leq d}}r^{2k}(w_{r}u_{r})^2.\]
Adding the even and odd contributions gives
\[\|\beta_{w^{(l)}}^{(k)}\|_{L^2}^2\leq (2\pi)^{2k}\left[\sum_{\substack{r = 2,4,\dots\\ r \leq d}} r^{2k}(w_{r}u_{r})^2+\sum_{\substack{r = 1,3,\dots\\ r \leq d}} r^{2k}(w_{r}u_{r})^2\right].\]
The union of even and odd indices from $1$ to $d$ is simply 
$\{1,2,\dots,d\}$, therefore we can merge the sums:
\[\|\beta_{w^{(l)}}^{(k)}\|_{L^2}^2\leq (2\pi)^{2k}\sum_{r=1}^{d} r^{2k}(w_{r}u_{r})^2.\]
\end{proof}

\subsection{Condition 2. of the proof}
\begin{prop}\label{prop:condition2}
Let $M \in \mathbb{N}\setminus\{0\}$ and, for each $l=1,\ldots,M$, let 
$w^{(l)} \in \{0,1\}^{d}$ and $w^{(0)}=(0,\ldots,0)$, where $d$ is defined in
\eqref{eq:d}. For $j\in\{0,l\}$, let $\beta_{w^{(j)}}$ be defined by
\eqref{eq:beta_test}. Let $P_j$ denote the joint law of the observed sample
$\{(Z_i,Y_i)\}_{i=1}^n$ under the model associated with $\beta_{w^{(j)}}$.
Then, for every $l=1,\ldots,M$,
\[K(P_0\|P_l)\leq \alpha'\log M,\]
for some $0 < \alpha' < \frac{1}{8}$.
\end{prop}

\begin{proof}
We first work at the level of a single observation $(X,Z,Y)$, and later use the fact that $(X_i,Z_i,Y_i)_{i=1,\ldots,n}$ are i.i.d to recover the result for the full sample. Using Lemma \ref{lem:condition_lemma2}, we have
\begin{equation}\label{eq:almost_final_full_clean}
    K\big(P_{0}^{1}\|P_{l}^{1}\big)\leq\mathbb{E}_{X\sim P^{(0),1}_{\mathrm{full}}(X)}
\left[K\big(P^{(0),1}_{\mathrm{full}}(Y\mid X)\|P^{(l),1}_{\mathrm{full}}(Y\mid X)\big)\right],
\end{equation}
where $P^{(l),1}_{\mathrm{full}}$ is the probability of the triplet $(X,Z,Y)$ under model $l$. We now compute the right-hand side. Under $P^{(l),1}_{\mathrm{full}}$, conditionally on $X$, we have
\[Y\mid X \sim \mathcal N\big(\langle X,\beta_{w^{(l)}}\rangle_{L^2},\sigma^2\big).\]
Define
\begin{equation}\label{eq:nu_l}
    \nu^{(l)} := \langle X,\beta_{w^{(l)}}\rangle_{L^2} = \sum_{j=1}^{d} w_{j}^{(l)} u_j \langle X,\phi_j\rangle_{L^2}.
\end{equation}
and 
\[\nu^{(0)}:=0.\]
Since the conditional distributions are Gaussian with the same variance, we obtain
\[\frac{dP^{(0),1}_{\mathrm{full}}(Y\mid X)}{dP^{(l),1}_{\mathrm{full}}(Y\mid X)}=\exp\left(\frac{1}{2\sigma^{2}}\left[(Y-\nu^{(l)})^{2}-(Y-\nu^{(0)})^{2}\right]\right),\]
and after developing the square we get
\[\frac{dP^{(0),1}_{\mathrm{full}}(Y\mid X)}{dP^{(l),1}_{\mathrm{full}}(Y\mid X)}=\exp\left(
\frac{1}{2\sigma^{2}}\left[\left(\nu^{(l)}\right)^{2}-2Y\nu^{(l)}\right]\right)=\exp\left(\frac{\left(\nu^{(l)}\right)^{2}}{2\sigma^{2}}-\frac{Y\nu^{(l)}}{\sigma^{2}}\right).\]
By taking expectation under $P^{(0),1}_{\mathrm{full}}(Y\mid X)$, we obtain
\[\mathbb{E}_{P^{(0),1}_{\mathrm{full}}}\left[\left.\log\frac{dP^{(0),1}_{\mathrm{full}}(Y\mid X)}{dP^{(l),1}_{\mathrm{full}}(Y\mid X)}\right|X\right]= \frac{1}{2\sigma^{2}} \Big[ \left(\nu^{(l)}\right)^{2} - 2\nu^{(l)}\mathbb E(Y\mid X) \Big].\]
Now, under $P^{(0),1}_{\mathrm{full}}$, we have $\mathbb E(Y\mid X)=\nu^{(0)}=0$, hence
\[\mathbb{E}_{P^{(0),1}_{\mathrm{full}}}\left[\left.\log\frac{dP^{(0),1}_{\mathrm{full}}(Y\mid X)}{dP^{(l),1}_{\mathrm{full}}(Y\mid X)}\right|X\right]= \frac{1}{2\sigma^{2}} \left(\nu^{(l)}\right)^{2}.\]
Using the definition of $\nu^{(l)}$ given in (\ref{eq:nu_l}) we have that
\[\left(\nu^{(l)}\right)^{2}= \sum_{j,k=1}^d w^{(l)}_j w^{(l)}_k u_j u_k\,\langle X,\phi_j\rangle_{L^2}\,\langle X,\phi_k\rangle_{L^2}.\]
Taking expectation with respect to $X \sim P^{(0),1}_{\mathrm{full}}(X)$, we obtain
\[\mathbb{E}\left[ \left(\nu^{(l)}\right)^2 \right]= \sum_{j,k=1}^d w^{(l)}_j w^{(l)}_k u_j u_k\mathbb{E}\big[\langle X,\phi_j\rangle_{L^2}\langle X,\phi_k\rangle_{L^2}\big].\]
Using now equation (\ref{eq:x_uncorrelated}) we get that
\[\mathbb{E}\left[ \left(\nu^{(l)}\right)^2 \right]= \sum_{j=1}^{d} \lambda_j u_j^2 \big(w_{j}^{(l)}\big)^2.\]
Combining with \eqref{eq:almost_final_full_clean}, we get for one observation
\[K\big(P_{0}^{1} \| P_{l}^{1}\big)\leq \frac{1}{2\sigma^{2}}\sum_{j=1}^{d}\lambda_{j}u_{j}^{2}\big(w_{j}^{(l)}\big)^{2},\]
Since the observations are independent, we have
\[K\big(P_{0}\|P_{l}\big)= nK\big(P_{0}^{1} \| P_{l}^{1}\big),\]
and therefore
\[K(P_{0}\|P_{l})\leq\frac{n}{2\sigma^{2}}\sum_{j=1}^{d}\lambda_{j}u_{j}^{2}\big(w_{j}^{(l)}\big)^{2}.\]
Using the definition $u_j^2=\tfrac{C}{n\lambda_j}$ and the bound
$C\le\sigma^{2}\tfrac{\alpha'\log(2)}{4}$, we obtain
\[\frac{n}{2\sigma^{2}}\sum_{j=1}^{d}\lambda_j u_j^2 \big(w_{j}^{(l)}\big)^2\leq \frac{\alpha'\log(2)d}{8}.\]
Finally, by the Varshamov--Gilbert bound (Lemma 2.9 in \cite{tsybakov2009introduction}),
$M\geq 2^{d/8}$, hence $d\leq \frac{8\log(M)}{\log(2)}$, which concludes the proof.
\end{proof}

\begin{lemma}\label{lem:condition_lemma2}
Let $l\in\{1,\ldots,M\}$. Let $P_0^1$ and $P_l^1$ denote the joint laws of
$(Z,Y)$ under the models associated with $\beta_{w^{(0)}}$ and
$\beta_{w^{(l)}}$, respectively. Let
$P_{\mathrm{full}}^{(0),1}$ and $P_{\mathrm{full}}^{(l),1}$ denote the
corresponding joint laws of $(X,Z,Y)$.
Then
\[K(P_0^1\|P_l^1) \leq \mathbb E_{X\sim P_{\mathrm{full}}^{(0),1}(X)} \left[ K\left( P_{\mathrm{full}}^{(0),1}(Y|X)\|P_{\mathrm{full}}^{(l),1}(Y|X) \right)\right].\]

\end{lemma}

\begin{proof}
We work with the full joint model for each $j\in\{0,l\}$, denoted $P^{(j),1}_{\mathrm{full}}$, governing the triplet $(X,Z,Y)$. The observed law $P_{j}^{1}$ is the $(Z,Y)$-marginal of $P^{(j),1}_{\mathrm{full}}$.

\medskip\noindent
Using Lemma \ref{lem:chainrule} stated below and its associated notation, we have
\begin{equation}\label{eq:kl_decomposition}
K\big(P_{0}^{1}\|P_{l}^{1}\big)= K\big(P_{0}^{1,Z}\|P_{l}^{1,Z}\big)+ \mathbb{E}_{Z\sim P_{0}^{1,Z}}\left[K\big(P_{0}^{1}(\cdot\mid Z)\|P_{l}^{1}(\cdot\mid Z)\big)\right].
\end{equation}
Since the marginal law of $Z$ does not depend on $j$ for $j\in\{0,l\}$ (the distribution of $(X,\eta)$ being identical across models), we have $P_{0}^{1,Z}=P_{l}^{1,Z}$ and thus
\[K\big(P_{0}^{1,Z}\|P_{l}^{1,Z}\big)=0.\]
Now fix $z$. For $j\in\{0,l\}$ we can define under the full model $P^{(j),1}_{\mathrm{full}}$, 
\[\mu_z(dx) := P^{(j),1}_{\mathrm{full}}(X \in dx \mid Z = z)\]
be a regular conditional distribution of $X$ given $Z=z$. Note that $\mu_z$
does not depend on $j$, since the joint law of $(X,Z)$ is the same for all $j$. Let us denote by $\mu(dx)$ the law of $X$. We have using Bayes's formula that the density of $\mu_z(dx)$ with respect to $\mu(dx)$ is
\begin{equation}\label{eq:mu_z}
    \mu_z(dx)=\frac{f_{Z\mid X=x}(z)}{\int f_{Z\mid X=u}(z)\mu(du)}\mu(dx),
\end{equation}
where $f_{Z\mid X=x}$ is defined in equation (\ref{eq:f_ZmidX}). Using the definition of the conditional density (\ref{eq:conditional_densities}) we have for $j\in\{0,l\}$ that
\begin{equation*}
    p_{j}(y\mid z)=\frac{p_{j}(z,y)}{p_{j}^{Z}(z)} = \frac{\int f_{Z\mid X=x}(z)p_{\mathrm{full}}^{(j)}(y\mid x)\mu(dx)}{\int f_{Z\mid X=x}(z)\mu(dx)}.
\end{equation*}
Now using the definition of $\mu_z$ given in equation (\ref{eq:mu_z}) we find for $j\in\{0,l\}$,
\begin{equation*}
    \begin{aligned}
        \int p_{\mathrm{full}}^{(j)}(y\mid x)\mu_{z}(dx) &= \int p_{\mathrm{full}}^{(j)}(y\mid x)\frac{f_{Z\mid X=x}(z)}{\int f_{Z\mid X=u}(z)\mu(du)}\mu(dx) \\
        &= \frac{\int f_{Z\mid X=x}(z)p_{\mathrm{full}}^{(j)}(y\mid x)\mu(dx)}{\int f_{Z\mid X=x}(z)\mu(dx)} \\
        &= p_{j}(y\mid z).
    \end{aligned}
\end{equation*}
Let us still consider $z$ fixed, and let us compute $K\big(P_{0}^{1}(\cdot\mid Z=z)\|P_{l}^{1}(\cdot\mid Z=z)\big)$.
Using the definition of the Kullback--Leibler divergence and the previous
representation for $p_{j}(y\mid z)$, we have
\begin{align*}
K\big(P_{0}^{1}(\cdot\mid Z=z)\|P_{l}^{1}(\cdot\mid Z=z)\big) &= \int p_0(y\mid z)\log\frac{p_0(y\mid z)}{p_l(y\mid z)}dy \\
&= \int\left(\int p^{(0)}_{\mathrm{full}}(y\mid x)\mu_z(dx)\right)
       \log\frac{\int p^{(0)}_{\mathrm{full}}(y\mid x)\mu_z(dx)}
                    {\int p^{(l)}_{\mathrm{full}}(y\mid x)\mu_z(dx)}dy.
\end{align*}
Now let us fix $y$ and set
\[a(x)=p^{(0)}_{\mathrm{full}}(y\mid x),\qquad b(x)=p^{(l)}_{\mathrm{full}}(y\mid x),\qquad
s=\int b(x)\mu_z(dx).\]
Let us  define the probability measure $\pi(dx)=\frac{b(x)\mu_z(dx)}{s}$ and the ratio $r(x)=a(x)/b(x)$. For $\varphi(t)=t\log t$ which is a convex function on $(0,\infty)$ we have by Jensen's inequality,
\begin{align*}
\int a(x)\log\left(\frac{a(x)}{b(x)}\right)\mu_z(dx) 
&= \int b(x)\varphi(r(x))\mu_z(dx) \\
&= s\int \varphi(r(x))\pi(dx) \\
&\ge s\varphi\left(\int r(x)\pi(dx)\right) \\
&= s\left(\frac{\int a(x)\mu_z(dx)}{s}\right)\log\left(\frac{\int a(x)\mu_z(dx)}{s}\right).
\end{align*}
As $s=\int b(x)\mu_z(dx)$, we obtain for each $y$,
\[\Big(\int p^{(0)}_{\mathrm{full}}(y\mid x)\mu_z(dx)\Big)\log\frac{\int p^{(0)}_{\mathrm{full}}(y\mid x)\mu_z(dx)}{\int p^{(l)}_{\mathrm{full}}(y\mid x)\mu_z(dx)}\leq\int p^{(0)}_{\mathrm{full}}(y\mid x)\log\frac{p^{(0)}_{\mathrm{full}}(y\mid x)}{p^{(l)}_{\mathrm{full}}(y\mid x)}\mu_z(dx).\]
Integrating over $y$ and applying Fubini’s theorem gives
\[K\big(P_{0}^{1}(Y\mid Z=z)\|P_{l}^{1}(Y\mid Z=z)\big)\leq\int\left(\int p^{(0)}_{\mathrm{full}}(y\mid x)\log\frac{p^{(0)}_{\mathrm{full}}(y\mid x)}{p^{(l)}_{\mathrm{full}}(y\mid x)}dy\right)\mu_z(dx).\]
The inner integral is the Kullback--Leibler divergence
$K\big(P^{(0),1}_{\mathrm{full}}(Y\mid X=x)\|P^{(l),1}_{\mathrm{full}}(Y\mid X=x)\big)$.
Hence
\[K\big(P_{0}^{1}(Y\mid Z=z)\|P_{l}^{1}(Y\mid Z=z)\big)\leq\mathbb{E}_{X\sim P^{(0),1}_{\mathrm{full}}(X\mid Z=z)}\left[K\big(P^{(0),1}_{\mathrm{full}}(Y\mid X)\|P^{(l),1}_{\mathrm{full}}(Y\mid X)\big)\right].\]
Taking expectations over $Z\sim P_{0}^{1,Z}$ and using \eqref{eq:kl_decomposition} we get,
\begin{equation}
K\big(P_{0}^{1}\|P_{l}^{1}\big)
\leq\mathbb{E}_{X\sim P^{(0),1}_{\mathrm{full}}(X)}\left[K\big(P^{(0),1}_{\mathrm{full}}(Y\mid X)\|P^{(l),1}_{\mathrm{full}}(Y\mid X)\big)\right].
\end{equation}

\end{proof}

\begin{lemma}\label{lem:chainrule}
Let $P_{0}^{1}$ and $P_{l}^{1}$, $l=1,\ldots,M$, be two probability measures on
$\mathbb{R}^{p}\times\mathbb{R}$ corresponding to the joint distribution of
$(Z,Y)$. In particular, for all model $j\in\{0,l\}$,
\begin{equation}\label{eq:Y}
    Y=\langle X,\beta_{w^{(j)}}\rangle_{L^{2}}+\epsilon,
\end{equation}
where $\beta_{w^{(j)}}$ given by (\ref{eq:beta_test}) and $\epsilon\sim\mathcal{N}(0,\sigma^{2})$. Assume that $P_{0}^{1} \ll P_{l}^{1}$ and let us denote by $P_{0}^{1,Z}$ and $P_{l}^{1,Z}$ the marginal laws of $Z$. Then,
\begin{equation*}
K(P_{0}^{1}\|P_{l}^{1})=K(P_{0}^{1,Z}\| P_{l}^{1,Z})+\mathbb{E}_{Z\sim P_{0}^{1,Z}}\left[K\bigl(P_{0}^{1}(\cdot\mid Z)\|P_{l}^{1}(\cdot\mid Z)\bigr) \right].
\end{equation*}
\end{lemma}

\begin{proof}
We first show that $P_{0}^{1}$ and $P_{l}^{1}$ admit densities with respect to the Lebesgue measure on $\mathbb{R}^{p+1}$. By definition of $Z$ and using the fact that the $\eta_{i,h}$'s are Gaussian, centered and of variance $\tau^{2}$, we have that
\begin{equation*}
    Z\mid X=x \sim \mathcal{N}(x^{\text{grid}},\tau^{2}I_{p}),
\end{equation*}
where $x^{\text{grid}}=\left(x(h/p)\right)_{h=0,\ldots,p-1}$. Therefore $Z\mid X=x$ admits a density $f_{Z\mid X=x}$ which is such that 
\begin{equation}\label{eq:f_ZmidX}
    f_{Z\mid X=x}(z)=\frac{1}{(2\pi\tau^{2})^{p/2}}\exp\left(-\frac{\|z-x^{\text{grid}}\|_{2}^{2}}{2\tau^{2}}\right).
\end{equation}
Moreover, under $P_{j}^{1}, j \in \{0, l\}$, the response variable Y is described by equation (\ref{eq:Y}). Therefore for all $j \in \{0, l\}$ we have conditional on $X = x$,
\begin{equation*}
    Y \mid X = x \sim \mathcal{N}(\langle x, \beta_{w^{(j)}} \rangle_{L^2}, \sigma^2).
\end{equation*}
Hence it admits a density $p_{\mathrm{full}}^{(j)}(\cdot\mid x)$ which is such that
\begin{equation*}
    p_{\mathrm{full}}^{(j)}(y\mid x)=\frac{1}{(2\pi\sigma^{2})^{1/2}}\exp\left(-\frac{(y-\langle x,\beta_{w^{(j)}}\rangle_{L^{2}})^{2}}{2\sigma^{2}}\right).
\end{equation*}
As the $\epsilon_i$'s and $\eta_{i,h}$'s are assumed to be independent we have that for all $j\in\{0,l\}$, $l=1,\ldots,M$, $(Z,Y)\mid X=x$ admits a density which is given by 
\begin{equation*}
    f_{(Z,Y)\mid X=x}^{(j)}(z,y)=f_{Z\mid X=x}(z)p_{\mathrm{full}}^{(j)}(y\mid x).
\end{equation*}
Let us now integrate out with respect to the law of $X$ that we denote as $\mu(dx)$. We obtain that the density of $P_{j}^{1}$ is given by
\begin{equation}\label{eq:p_j(z,y)}
\begin{aligned}
    p_{j}(z,y) &=\int f_{(Z,Y)\mid X=x}^{(j)}(z,y)\mu(dx) \\
    &= \int f_{Z\mid X=x}(z)p_{\mathrm{full}}^{(j)}(y\mid x)\mu(dx).
\end{aligned}
\end{equation}
Finally, we get that $P_{j}^{1,Z}$ admits a density too which is given by
\begin{equation}\label{eq:p_jZ}
    \begin{aligned}
        p_{j}^{Z}(z) &= \int p_{j}(z,y)dy \\
        &= \int\int f_{Z\mid X=x}(z)p_{\mathrm{full}}^{(j)}(y\mid x)\mu(dx)dy \\
        &= \int f_{Z\mid X=x}(z)\mu(dx),
    \end{aligned}
\end{equation}
where we used Fubini's theorem and the definition of a density to derive the last line. Thus the conditional density of $Y$ given $Z$ under $P_{j}^{1}$ is given by
\begin{equation}\label{eq:conditional_densities}
\begin{aligned}
    p_{j}(y\mid z)&=\frac{p_{j}(z,y)}{p_{j}^{Z}(z)}\\
    &= \frac{\int f_{Z\mid X=x}(z)p_{\mathrm{full}}^{(j)}(y\mid x)\mu(dx)}{\int f_{Z\mid X=x}(z)\mu(dx)},
\end{aligned}
\end{equation}
By definition of the Kullback--Leibler divergence, we have 
\begin{equation*}
    \begin{aligned}
        K\left(P_{0}^{1}\| P_{l}^{1}\right) &= \int p_{0}(z,y)\log{\left(\frac{p_{0}(z,y)}{p_{l}(z,y)}\right)dzdy} \\
        &= \int p_{0}(z,y)\log{\left(\frac{p_{0}(y\mid z)p_{0}^{Z}(z)}{p_{l}(y\mid z)p_{l}^{Z}(z)}\right)dzdy} \\
        &= \int p_{0}(z,y)\log{\left(\frac{p_{0}^{Z}(z)}{p_{l}^{Z}(z)}\right)}dzdy + \int p_{0}(z,y)\log{\left(\frac{p_{0}(y\mid z)}{p_{l}(y\mid z)}\right)}dzdy,
    \end{aligned}
\end{equation*}
where we used the definition of the conditional density given by equation (\ref{eq:conditional_densities}) to get the second line and the property of the $\log$ to get the last one. Now let's deal with these two terms one by one. For the first one we have, using Fubini's theorem and the definition of the marginal density that
\begin{equation*}
    \begin{aligned}
        \int p_{0}(z,y)\log{\left(\frac{p_{0}^{Z}(z)}{p_{l}^{Z}(z)}\right)}dzdy &= \int \log{\left(\frac{p_{0}^{Z}(z)}{p_{l}^{Z}(z)}\right)} \left(\int p_{0}(z,y)dy\right)dz \\
        &= \int p_{0}^{Z}(z)\log{\left(\frac{p_{0}^{Z}(z)}{p_{l}^{Z}(z)}\right)} dz \\
        &= K\left(P_{0}^{1,Z}\|P_{l}^{1,Z}\right).
    \end{aligned}
\end{equation*}
For the second term, we get using the definition of conditional densities given in (\ref{eq:conditional_densities}) and Fubini's theorem that
\begin{equation*}
\begin{aligned}
    \int p_{0}(z,y)\log{\left(\frac{p_{0}(y\mid z)}{p_{l}(y\mid z)}\right)}dzdy &= \int p_{0}^{Z}(z)p_{0}(y\mid z)\log{\left(\frac{p_{0}(y\mid z)}{p_{l}(y\mid z)}\right)}dzdy \\
    &= \int p_{0}^{Z}(z)\left(\int p_{0}(y\mid z)\log{\left(\frac{p_{0}(y\mid z)}{p_{l}(y\mid z)}\right)} dy\right) dz \\
    &= \mathbb{E}_{Z\sim P_{0}^{1,Z}}\left(K(P_{0}^{1}(\cdot\mid Z)\|P_{l}^{1}(\cdot\mid Z))\right).
\end{aligned}
\end{equation*}
Combining these two results yields the desired property.
\end{proof}

\subsection{Condition 3. of the proof}

\begin{prop}\label{prop:condition3}
    Consider $d$ defined by Equation (\ref{eq:d}), the test functions given by Equation (\ref{eq:beta_test}), the semi-norm described by Equation (\ref{eq:semi_norm_gamma}). Then,
    \[\|\beta_{w^{(l)}}-\beta_{w^{(l')}}\|_{\Gamma}^{2}\geq 2C_{4}n^{-\frac{2a+2k}{2a+2k+1}},\]
    where $C_{4}$ is a positive constant which depends on $k,L,c',\sigma$.
\end{prop}

\begin{proof}
    Using the definition of $\Gamma$ given in (\ref{eq:Gamma}) we find that 
    \begin{equation}\label{eq:start_condition3}
    \begin{aligned}
        \|\beta_{w^{(l)}}-\beta_{w^{(l')}}\|_{\Gamma}^{2} &= \langle\Gamma(\beta_{w^{(l)}}-\beta_{w^{(l')}}), \beta_{w^{(l)}}-\beta_{w^{(l')}}\rangle_{L^{2}} \\
        &=\langle\sum_{j= 1}^{d}u_{j}(w_{j}^{(l)}-w_{j}^{(l')})\Gamma\phi_j,\sum_{k=1}^{d}u_{k}(w_{k}^{(l)}-w_{k}^{(l')})\phi_k\rangle_{L^{2}}  \\
        &= \sum_{j,k=1}^{d}u_{j}u_{k}(w_{j}^{(l)}-w_{j}^{(l')})(w_{k}^{(l)}-w_{k}^{(l')})\langle\Gamma\phi_j,\phi_k\rangle_{L^{2}}  \\
        &=\sum_{j=1}^{d}\lambda_{j}(w_{j}^{(l)}-w_{j}^{(l')})^{2}u_{j}^{2},
    \end{aligned}
\end{equation}
as the $\lambda_j$'s are eigenvalues of $\Gamma$ associated to the eigenvectors $\phi_j$'s.
Now using the definition of $u_{j}$ given by (\ref{eq:u_j}) we get
\begin{equation*}
    \begin{aligned}
        \|\beta_{w^{(l)}} - \beta_{w^{(l')}}\|_{\Gamma}^{2} 
        &\geq \frac{C}{n}\sum_{j=1}^{d}(w_{j}^{(l)}-w_{j}^{(l')})^{2}.
    \end{aligned}
\end{equation*}
Now, as
\[\sum_{j=1}^{d}(w_{j}^{(l)}-w_{j}^{(l')})^{2}=\sum_{j=1}^{d}\mathds{1}_{\{w_{j}^{(l)}\neq w_{j}^{(l')}\}},\]we can use the Varshamov--Gilbert bound (Lemma 2.9 of Tsybakov \cite{tsybakov2009introduction}) and get that
\begin{equation*}
    \begin{aligned}
        \|\beta_{w^{(l)}} - \beta_{w^{(l')}}\|_{\Gamma}^{2}\geq \frac{Cd}{8n}.
    \end{aligned}
\end{equation*}
Using the definition of $d$ we find that 
\[\|\beta_{w^{(l)}} - \beta_{w^{(l')}}\|_{\Gamma}^{2}\geq 2C_{4}n^{-\frac{2a+2k}{2a+2k+1}},\]
where $C_{4}=\frac{C\kappa}{16}$ which is the desired result.
\end{proof}

\section{Orthogonal projection of \texorpdfstring{$\beta$}{beta} onto
\texorpdfstring{$S_m$}{S m} with respect to
\texorpdfstring{$\langle\cdot,\cdot\rangle_{\Gamma}$}{the Gamma inner product}
and
\texorpdfstring{$\langle\cdot,\cdot\rangle_{\widetilde{\Gamma}}$}{the Gamma tilde inner product}}\label{sec:remark_ortho}
As $\beta$ is in $L^{2}([0,1])$ and the Fourier basis forms a basis of this space, we can write
\[\beta=\sum_{j\geq 1}\beta_j\phi_j,\]
where $\beta_j=\langle\beta,\phi_j\rangle_{L^{2}}.$ We first determine $\beta^{(m)}_{\Gamma}$ the orthogonal projection of $\beta$ onto $S_m$ with respect to $\langle\cdot,\cdot\rangle_{\Gamma}$. Using the definition of $S_m$ defined by equation \eqref{eq:S_m} we have that $(\phi_1,\phi_2,\ldots,\phi_{D_m})$ form a basis of $S_m$. By definition of the orthogonal projection $\beta^{(m)}_{\Gamma}$ is in $S_m$, there exist real numbers $\mu_1,\mu_2,\ldots,\mu_{D_m}$, such that
\[\beta^{(m)}_{\Gamma}=\sum_{j=1}^{D_m}\mu_j\phi_j,\]
and for the same argument we have that $\beta-\beta^{(m)}_{\Gamma}\in S_{m}^{\perp}$. As $(\phi_1,\phi_2,\ldots,\phi_{D_m})$ is a basis of $S_m$ this implies that for all $k=1,\ldots,D_m$,
\[\langle \beta-\beta^{(m)}_{\Gamma},\phi_k\rangle_{\Gamma}=0.\]
Using the fact that for all $j,k=1,\ldots,D_m$
\[\langle\phi_j,\phi_k\rangle_{\Gamma}=\langle\Gamma\phi_j,\phi_k\rangle_{L^{2}}=\langle \lambda_j\phi_j,\phi_k\rangle_{L^{2}}=\lambda_j\delta_{jk},\]
with $\lambda_j>0$ for all $j=1,\ldots,D_m$, and Fubini's theorem to justify the following computation, we obtain
\begin{equation*}
    \begin{aligned}
        \langle \beta-\beta^{(m)}_{\Gamma},\phi_k\rangle_{\Gamma} &= \langle \sum_{j>D_m}\beta_j\Gamma\phi_j,\phi_k\rangle_{L^{2}}+\langle \sum_{j=1}^{D_m}(\beta_j-\mu_j)\Gamma\phi_j,\phi_k\rangle_{L^{2}} &=
        \lambda_k(\beta_k-\mu_k).
    \end{aligned}
\end{equation*}
Therefore, $\langle \beta-\beta^{(m)}_{\Gamma},\phi_k\rangle_{\Gamma}=0 \iff \lambda_k(\beta_k-\mu_k)=0$ for all $k=1,\ldots,D_m$. And as $\lambda_k>0$ we find that $\langle \beta-\beta^{(m)}_{\Gamma},\phi_k\rangle_{\Gamma}=0 \iff \beta_k=\mu_k$. Hence, $\beta^{(m)}_{\Gamma}=\sum_{j=1}^{D_m}\beta_j\phi_j$. We proceed in a similar way to determine $\beta^{(m)}_{\widetilde{\Gamma}}$ the orthogonal projection of $\beta$ onto $S_m$ with respect to $\langle\cdot,\cdot\rangle_{\widetilde{\Gamma}}$. By the definition of the orthogonal projection, we must have $\beta^{(m)}_{\widetilde{\Gamma}}\in S_m$, therefore there exist $\widetilde{\mu}_1,\ldots,\widetilde{\mu}_{D_m}$ such that
\[\beta^{(m)}_{\widetilde{\Gamma}}=\sum_{j=1}^{D_m}\widetilde{\mu}_j\phi_j.\]
We also have that for all $k=1,\ldots,D_m$,
\[\beta-\beta^{(m)}_{\widetilde{\Gamma}}=0.\]
Using Lemma \ref{lem:eigenvalues_gamma_tilde} and Fubini's theorem we obtain,
\begin{equation*}
    \begin{aligned}
        \langle \beta-\beta^{(m)}_{\widetilde{\Gamma}},\phi_k\rangle_{\widetilde{\Gamma}} &= \langle \sum_{j>D_m}\beta_j\widetilde{\Gamma}\phi_j,\phi_k\rangle_{L^{2}}+\langle \sum_{j=1}^{D_m}(\beta_j-\widetilde{\mu}_j)\widetilde{\Gamma}\phi_j,\phi_k\rangle_{L^{2}} &=
        \widetilde{\lambda}_k(\beta_k-\widetilde{\mu}_k).
    \end{aligned}
\end{equation*}
Hence, $\langle \beta-\beta^{(m)}_{\widetilde{\Gamma}},\phi_k\rangle_{\widetilde{\Gamma}}=0 \iff \lambda_k(\beta_k-\widetilde{\mu}_k)=0$ for all $k=1,\ldots,D_m$. And as $\widetilde{\lambda}_k>0$ we find that $\langle \beta-\beta^{(m)}_{\widetilde{\Gamma}},\phi_k\rangle_{\widetilde{\Gamma}}=0 \iff \beta_k=\widetilde{\mu}_k$. Therefore, $\beta^{(m)}_{\widetilde{\Gamma}}=\sum_{j=1}^{D_m}\beta_j\phi_j$ and $\beta^{(m)}_{\widetilde{\Gamma}}=\beta^{(m)}_{\Gamma}$.
\end{appendix}

\section*{Acknowledgments}
I am deeply grateful to my supervisors, Gaëlle Chagny, Vincent Rivoirard and Angelina Roche, for their insightful comments and constructive feedback.

\medskip\noindent
The author also acknowledge the support of the French Agence Nationale de la Recherche (ANR) under reference ANR-24-CE40-2439 (FUNMathStat project).

\printbibliography

\end{document}